\documentclass[a4paper]{statsoc}

\usepackage{amsmath,verbatim,amssymb,epsfig,enumitem,url,amsbsy,enumitem,color,natbib,bm,float, graphicx}
\usepackage{subcaption}
\usepackage{geometry}
\usepackage{algorithm2e}
\usepackage{appendix}
\RestyleAlgo{ruled}

\renewcommand\refname{ \noindent \bf References \vspace{0.0in}}

  \renewenvironment{thebibliography}[1]{
    \begin{oldthebibliography}{#1}
      \setlength{\parskip}{0ex}
      \setlength{\itemsep}{0.ex}}
  {\end{oldthebibliography}}

\usepackage{etoolbox}
\usepackage{ulem}

\makeatletter
\patchcmd{\@makecaption}
  {\parbox}
  {\advance\@tempdima-\fontdimen2} 
  {}{}
\makeatother  


\makeatother

\usepackage{xr-hyper}
\usepackage{hyperref}
\makeatletter
\newcommand*{\addFileDependency}[1]{
\typeout{(#1)}
\@addtofilelist{#1}
\IfFileExists{#1}{}{\typeout{No file #1.}}
}
\makeatother

\newcommand*{\myexternaldocument}[1]{%
\externaldocument{#1}%
\addFileDependency{#1.tex}%
\addFileDependency{#1.aux}%
}
\myexternaldocument{CasualDeepsets-supp}

\newtheorem{theorem}{Theorem}
\newtheorem{lemma}{Lemma}
\newtheorem{corollary}{Corollary}
\newtheorem{assumption}{Assumption}%
\newtheorem{definition}{Definition}
\newcommand{\Acal}{\mathcal A}

\newcommand{\Fcal}{\mathcal F}
\newcommand{\Gcal}{\mathcal G}

\newcommand{\Ical}{\mathcal I}

\newcommand{\Kcal}{\mathcal K}

\newcommand{\Ncal}{\mathcal N}

\newcommand{\Pcal}{\mathcal P}

\newcommand{\Rcal}{\mathcal R}
\newcommand{\Scal}{\mathcal S}
\newcommand{\Tcal}{\mathcal T}

\newcommand{\Wcal}{\mathcal W}

\newcommand{\E}{\mathbb{E}}
\newcommand{\N}{\mathbb{N}}

\newcommand{\R}{\mathbb{R}}
\newcommand{\I}{\mathbb{I}}

\newcommand{\m}{\mathbf{m}}
\newcommand{\p}{\mathbf{p}}
\newcommand{\af}{\mathbf{a}}

\newcommand{\mb}{\mathbb}

\newcommand{\calN}{\mathcal{N}}
\newcommand{\Mean}{{\mathbb{E}}}

\newcommand{\prob}{{\mathbb{P}}}

\newcommand{\bbs}{{\mathbb{S}}}
\DeclareMathOperator*{\argmin}{arg\,min}

\newcommand{\s}{\mathbf{s}}

\title[Causal Deepsets for Off-policy Evaluation under Spatial or Spatio-temporal Interferences]{Causal Deepsets for Off-policy Evaluation under Spatial or \\Spatio-temporal Interferences}

\author[Runpeng Dai$^{a,b*}$, Jianing Wang$^{a**}$, Fan Zhou$^{a**}$]{Runpeng Dai$^{a,b*}$\thanks{The first three authors have contributed equally to this paper.}, Jianing Wang$^{a**}$, Fan Zhou$^{a**}$, Shikai Luo$^{c}$,\\
 Zhiwei Qin$^{d}$, Chengchun Shi$^{e}$\thanks{The final two authors listed are joint senior contributors to this work.  }, and Hongtu Zhu$^{b}$\thanks{Corresponding author. Email: {htzhu@email.unc.edu}}\\}
\address{
        $^a$Shanghai University of Finance and Economics, Shanghai, China\\
        $^b$University of North Carolina at Chapel Hill, North Carolina, USA \\
        $^c$Bytedance, Beijing, China \\      $^d$Independent Researcher, California, USA\\
        $^e$London School of Economics and Political Science, London, UK\\
}

\begin{document}

\begin{abstract}
Off-policy evaluation (OPE) is widely applied in sectors such as pharmaceuticals and e-commerce to evaluate the efficacy of novel products or policies from offline datasets. This paper introduces a causal deepset framework that relaxes several key structural assumptions, primarily the mean-field assumption, prevalent in  existing OPE methodologies   that handle spatio-temporal interference. These traditional assumptions frequently prove inadequate in real-world settings, thereby restricting the capability of current OPE methods to effectively address complex interference effects. In response, we advocate for the implementation of the permutation invariance (PI) assumption. This innovative approach enables the data-driven, adaptive learning of the mean-field function, offering a more flexible estimation method beyond conventional averaging. Furthermore, we present novel algorithms that incorporate the PI assumption into OPE and thoroughly examine their theoretical foundations. Our numerical analyses demonstrate that this novel approach yields significantly more precise estimations than existing baseline algorithms, thereby substantially improving the practical applicability and effectiveness of OPE methodologies. A Python implementation of our proposed method is available at \url{https://github.com/BIG-S2/Causal-Deepsets}.

\end{abstract}

\section{Introduction}
            
    \subsection{Background}

Many causal inference problems involve spatial or spatio-temporal data, which consist of observations recorded at specific locations and/or times. This data type is increasingly prevalent in environmental studies, epidemiology, social science, two-sided marketplaces, and various other fields 
due to advanced modern data collection methods. Unlike traditional scenarios where observations are typically independent across different subjects, spatial data often exhibit spillover effects. This means the treatment applied to one subject can affect not only their own response but also the responses of their neighboring subjects, challenging the stable unit treatment value assumption \citep[SUTVA,][]{angrist1996identification}.
For example, in infectious disease control, an individual's risk of infection is influenced by the vaccination status of others nearby \citep{hudgens2008toward}. Social scientists have also observed that living in disadvantaged neighborhoods can intensify the difficulties faced by the urban poor, further illustrating the impact of spatial interference \citep{sobel2006randomized}.

This paper is motivated by the challenge of managing complex spatio-temporal interferences in major ride-sourcing platforms such as Uber, Lyft, and Didi Chuxing \citep{hahn2017ridesharing,qin2022reinforcement}. Key interferences in these platforms include  fluctuating demand and supply across various locations and times, the impact of surge pricing, changes in driver income due to shifts in driver locations, and the influence of external factors such as weather and events on demand \citep{WangYang2019,zhou2021graph,qin2022reinforcement}. An example of these challenges can be seen in our evaluation of new passenger-side subsidizing policies. When implemented in specific regions, these policies not only boost demand and draw drivers from nearby areas, affecting income in both the target and surrounding regions, but also have a delayed effect on the driver income distribution, leading to both spatial and temporal interferences. 
The presence of such complex interferences significantly complicates the application of traditional causal inference methods. To navigate these complexities effectively, the development of advanced statistical models and methods is essential for accurately evaluating the impact of these policies in the context of 
causal inference with spatio-temporal interference. 

\subsection{Related Works}
Our proposal intersects with three distinct bodies of literature: spatial interference in causal inference, off-policy evaluation (OPE), and the concept of permutation invariance within the realm of geometric deep learning. 

\subsubsection{Spatial Interference in Causal Inference}

The literature on addressing spatial interference in causal inference broadly falls into two main categories: design-based and model-based methods \citep[see][and the references therein]{reich2021review}. Design-based methods concentrate on leveraging the design of experiments and sampling processes. For instance, \cite{aronow2017estimating} introduced a spatial causal inference framework that estimates average potential outcomes using known treatment distributions and exposure mapping. \cite{li2019randomization} proposed randomization-based point estimators and confidence intervals to study peer effects with arbitrary numbers of peers and peer types.
\cite{wang2020design} proposed estimating the average marginalized response to quantify outcome interference by treatments at specific distances, assuming a Bernoulli treatment distribution. 
On the other hand, model-based methods are divided into four subcategories. The first subcategory, structural models, imposes interference structures to approximate interference effects. Notable examples include conditional autoregressive \citep[CAR,][]{banerjee2003hierarchical} and spatial autoregressive   \citep[SAR,][]{lee2007identification} models, which introduce spatial random effects to represent neighborhood influences. The second subcategory, partial interference methods, divides populations into non-overlapping blocks, assuming interference within these blocks 
\citep{sobel2006randomized, tchetgen2012causal, zigler2012estimating, perez2014assessing}.  
The third subcategory focuses on experimental units in geographical spaces or networks, employing local or network-based assumptions for causal inference \citep{verbitsky2012causal, wang2020design, tchetgen2021auto, puelz2022graph}. 
The final subcategory 
considers scenarios where interference manifests as congestion or pricing effects in two-sided markets \citep{munro2021treatment,  johari2022experimental}. 

\subsubsection{Off-policy Evaluation}
 
Off-policy evaluation (OPE) represents an actively evolving research domain within the field of reinforcement learning 
\citep[RL,][]{SuttonRL,agarwal2019reinforcement,OfflineRL_review}. The primary goal of OPE is to accurately estimate the mean reward that a prospective target policy might yield, utilizing observational data generated under a different, {\it behavior policy}. This area encompasses a variety of methods \citep[see][for reviews]{dudik2014doubly,uehara2022review}, including model-based approaches \citep{gottesman2019combining,yin2020asymptotically}, value-based techniques \citep{le2019batch,luckett2020estimating,hao2021bootstrapping,liao2021off,chen2022well,shi2022statistical}, importance sampling methods \citep{thomas2015high,luedtke2016statistical,liu2018breaking,hu2023off,thams2021statistical,wang2023projected}, and doubly robust strategies \citep{zhang2012robust,jiang2016doubly,thomas2016data,tang2019doubly,kallus2022efficiently,liao2022batch}. 
Despite the progress in OPE methodologies, a notable shortcoming is their general lack of accounting for spatial interference, an essential factor in many real-world scenarios, including ride-sharing platforms. Recognizing this gap, recent research endeavors have started integrating concepts from mean-field multi-agent RL \citep{yang2018mean}  to handle spatio-temporal interference effects \citep{shi2022,luo2022policy}. This approach marks a significant step towards modeling the complex interactions in environments where actions in one location or time can influence outcomes in another. However, a key limitation of these emerging methods is their dependency on specifying the mean-field function. This requirement often poses challenges in terms of flexibility and adaptability, particularly when applied to diverse and unpredictable real-world settings where the dynamics of interference are not easily quantifiable or predefined. Therefore, while these methods mark a critical advancement towards OPE with spatial data, there remains a need for further development to enhance their applicability and efficacy in a wider range of practical applications.

\subsubsection{Permutation Invariance}

Permutation invariant functions, which maintain their output regardless of the order of their inputs, are prevalent in domains such as cosmology \citep{ntampaka2016dynamical} and computer vision \citep{aittala2018burst}. This characteristic has led to the design of numerous geometric deep learning models specifically aimed at approximating these types of functions \citep{bronstein2017geometric}. A notable example is the structure called DeepSets, proposed by  \citet{zaheer2017deep}, which is capable of approximating any permutation invariant function. DeepSets has shown impressive performance in 
a number of tasks such as sum of digits, classification of point-clouds, and regression
with clustering side-information, 
highlighting its versatility. 
Expanding on this concept, the PINE model \citep{pine}  offers a generalization of DeepSets by studying partially permutation invariant functions. In these functions, inputs are categorized into distinct groups, and the function remains invariant to permutations occurring within these individual groups. This design allows for greater flexibility in various contexts.  
Applying these permutation invariant models to studying the interference effects in 
spatio-temporal systems appears to be a natural progression. 
However, there is a lack of existing literature specifically addressing the application of these models in spatio-temporal interference modeling. This gap suggests an opportunity for novel research contributions in this area, potentially leading to advancements in 
understanding 
complex spatio-temporal systems, such as ride-sharing platforms.

\begin{figure}[h]       \center{\includegraphics[width=0.6\linewidth]  {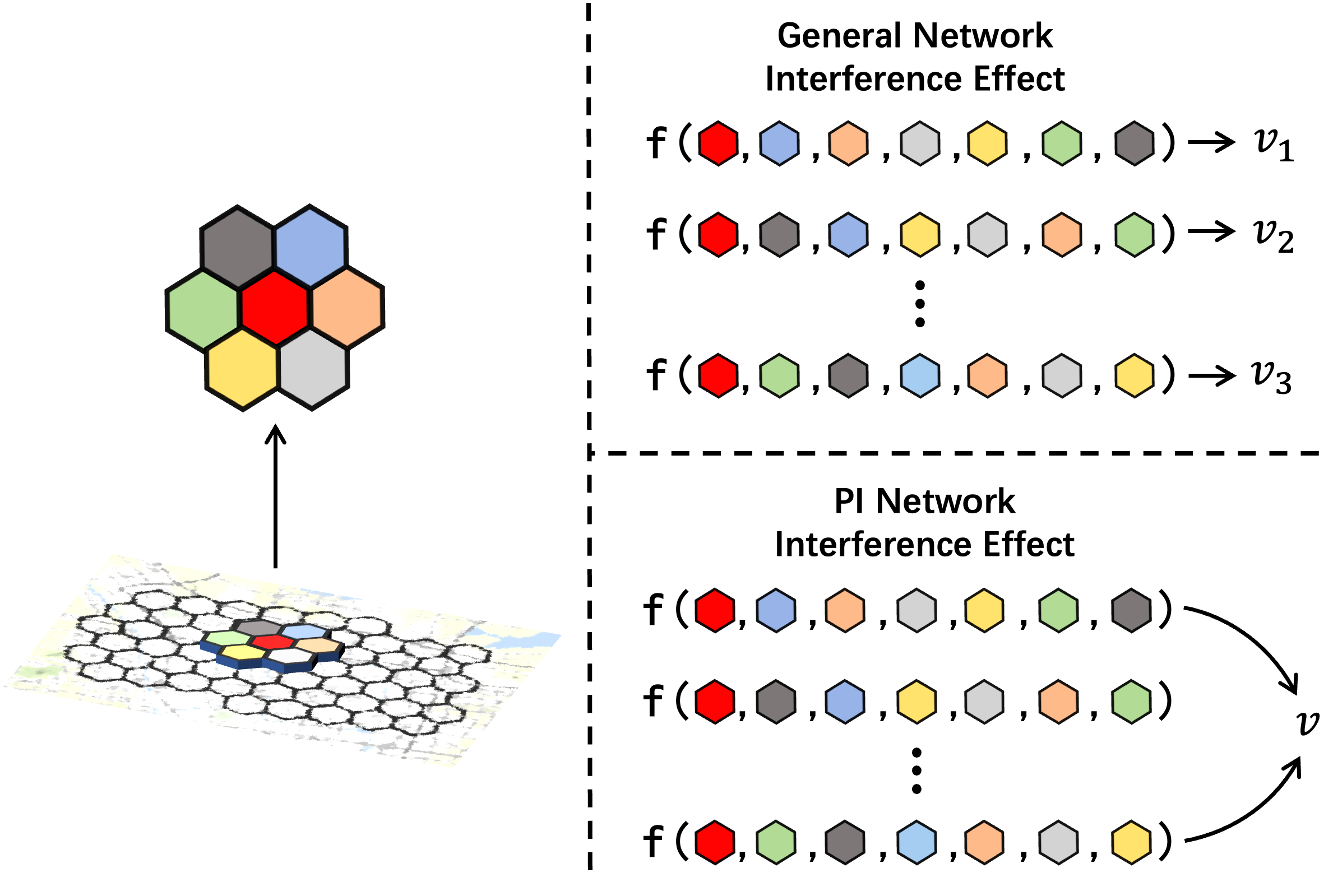}}   
	\caption{The illustration of permutation invariant (PI) mean-outcome function. The red hexagon represents confounder-treatment pair of the central region while the other six hexagons represent its neighboring regions. The upper-right subplot shows the mean-outcome function with only general network interference assumption \ref{ass: neighbor}. Here, the output of $f$ changes across different permutations of neighboring regions. On the other hand, the subplot on the right bottom shows that with the permutation invariant assumption \ref{ass:PI}, all outputs have the same value.}
	\label{fig:permutation invariant}  
\end{figure}

\subsection{Contributions}

In our paper, we introduce a novel causal deepsets model designed to address spatial interference, incorporating several innovative elements. Firstly, we propose a permutation invariant 
neural network architecture that effectively manages spatial interference. 
This architecture aggregates information from neighboring areas and constructs a general permutation invariant function through neural networks, thereby moving beyond traditional parametric models and facilitating adaptive learning of interference patterns. Additionally, we seamlessly combine this architecture with OPE algorithms to produce robust causal estimators. 
For spatio-temporal data, we further employ the multi-agent RL (MARL) framework which models each spatial region as an individual agent to capture spatial interference through agent interactions and utilizes Markov decision processes \citep[MDPs,][]{puterman2014markov}  to model temporal relationships. 
Our methodology's efficacy is evidenced through comprehensive simulations and real data analyses, where it demonstrates outstanding performance. 
Beyond these methodological advancements, we conduct an in-depth theoretical analysis of our proposed estimator. We establish its consistency and convergence rate under practical assumptions. We also prove the minimax optimality of our estimator in approximating permutation invariant functions, underscoring our approach's capability in accurately representing and managing complex interference structures. Considering that our estimator is built upon the assumption of permutation invariance, we have named it the Permutation Invariant Estimator, or PIE, for short.

\subsection{Paper Outline}
The paper is structured as follows. In Sections 2 and 3, we 
focus on the derivation of our estimators, 
along with the underlying assumptions. The statistical properties of our method are discussed in Section 4. In Section 5, we present the results of numerical studies and real data analysis, showcasing the superior performance of our approach.

\section{Spatial Interference in Nondynamic Settings}
\label{sec:nondynamic} 

This section aims to thoroughly explore the issue of spatial interference in nondynamic settings, often referred to as contextual bandits in OPE. 
Spatial interference implies that the potential outcome in one region is influenced by other regions. We denote each region by an index $i$ within the set $\{1,\ldots,R\}$. For the $i$th region, $Y_i$, $A_i$, and $X_i = (X_{i,1}, \ldots, X_{i,M})^\top\in \mathbb{R}^M$ represent the response, treatment, and $M$-dimensional confounding variables, respectively.

In scenarios without spatial interference, a region's outcome is a function of its own confounding variables and treatment:
\begin{equation}
Y_i = f_i(X_i, A_i) + \epsilon_i,
\label{eq: no spatial}
\end{equation}
where the error term $\epsilon_i$ satisfies $\mathbb{E}(\epsilon_i|\{A_j\}_{j=1}^R, \{X_j\}_{j=1}^R)=0$.
In practical applications, however, the assumption of no interference is often unrealistic, making model (\ref{eq: no spatial}) inapplicable. Without specific knowledge of the interference pattern, the outcome in a region can generally be expressed as a function of its own confounding variables and the treatments across all regions:
\begin{equation}
Y_i = f_i(X_i, A_i, m_i(X_{-i}, A_{-i})) + \epsilon_i,
\label{eq: general spatial}
\end{equation}
where $X_{-i}= (X_1, \ldots,X_{i-1},X_{i+1},\ldots, X_R)$ and $A_{-i} = (A_1, \ldots,A_{i-1}, A_{i+1},\ldots, A_R)$.  Here, $f_i$ includes two components: the first reflecting the $i$th region's contribution and the second capturing interference effects from other regions, denoted by the function $m_i(\cdot)$. This model, depicted in Figure \ref{fig: general spatial}, extends beyond traditional models that only consider neighboring region effects \citep[see e.g.,][]{reich2021review}. Specifically, the interference effect function $m_i(\cdot)$ in \eqref{eq: general spatial} accounts for influences from all regions. Moreover, this framework accommodates nonlinear interference effects, which are common in real-world scenarios.

Model \eqref{eq: general spatial} offers a broad framework for understanding spatial interference, but its complexity becomes daunting with the increase in the number of regions, often rendering it impractical to solve. To combat this, various strategies have been developed to reduce its complexity. One notable approach is spatial network inference  \citep[see e.g.,][]{forastiere2021identification}, which narrows the focus to interactions between regions and their immediate neighbors. While ``neighborhood" could be defined in various ways, we simplify it here to mean geographical proximity.

The study of spatial interference has a rich history \citep{blume1993statistical}. Many researchers have posited that interference typically occurs within a defined neighborhood and can be transmitted via environmental conduits like adjacent regions. Although this method simplifies the complexity of inter-regional interactions, it still implies dependencies between any two regions due to their potential indirect connections. There's often at least one indirect path connecting any two regions, creating a dependency link. 
For example, consider two non-adjacent units $A$ and $B$. In this case, $A$ would be conditionally independent of $B$ if we have complete information about $A$'s neighbors. However, indirect connections between $A$ and $B$ are still permissible. Our proposed approach is based on such an understanding of spatial interdependencies.


\begin{figure}[htbp]
  \centering
    \subfloat[No spatial interference (\ref{eq: no spatial})]
  {\includegraphics[width=0.23\textwidth]{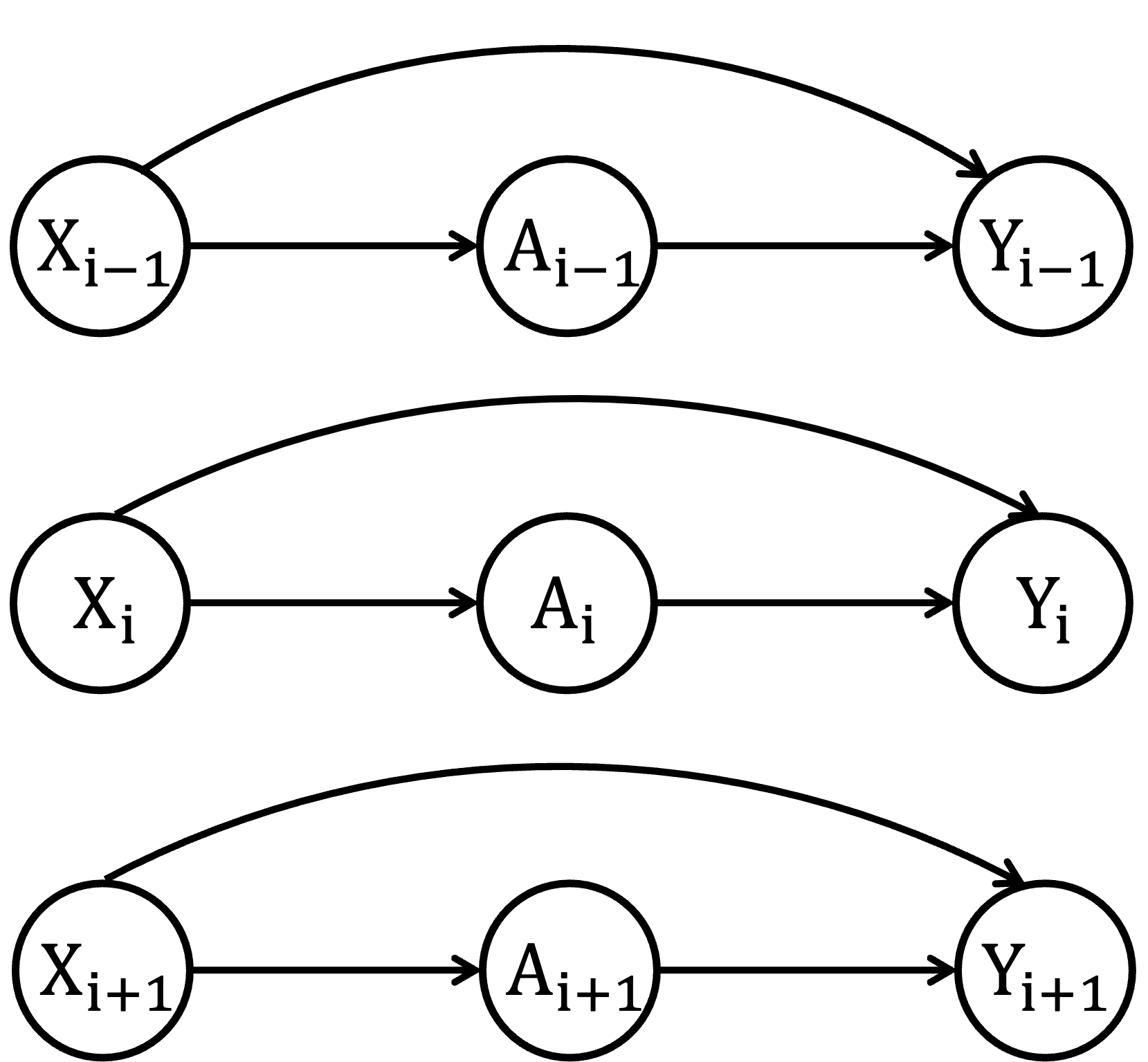}\label{fig: nospatial}}
    \quad
  \subfloat[Full spatial interference (\ref{eq: general spatial})]
  {\includegraphics[width=0.23\textwidth]{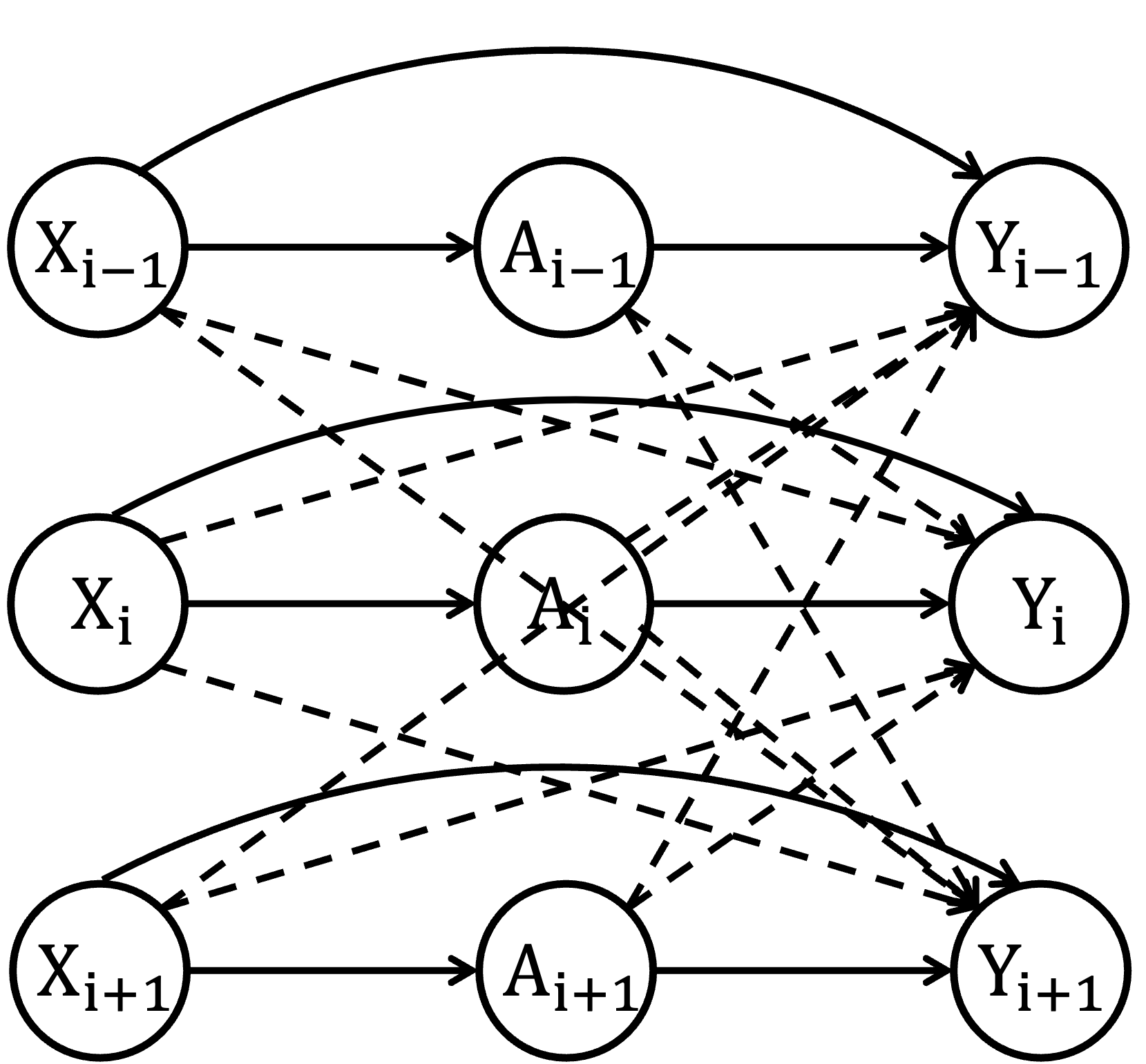}\label{fig: general spatial}}
    \quad
  \subfloat[Network interference \cite{forastiere2021identification}]
  {\includegraphics[width=0.23\textwidth]{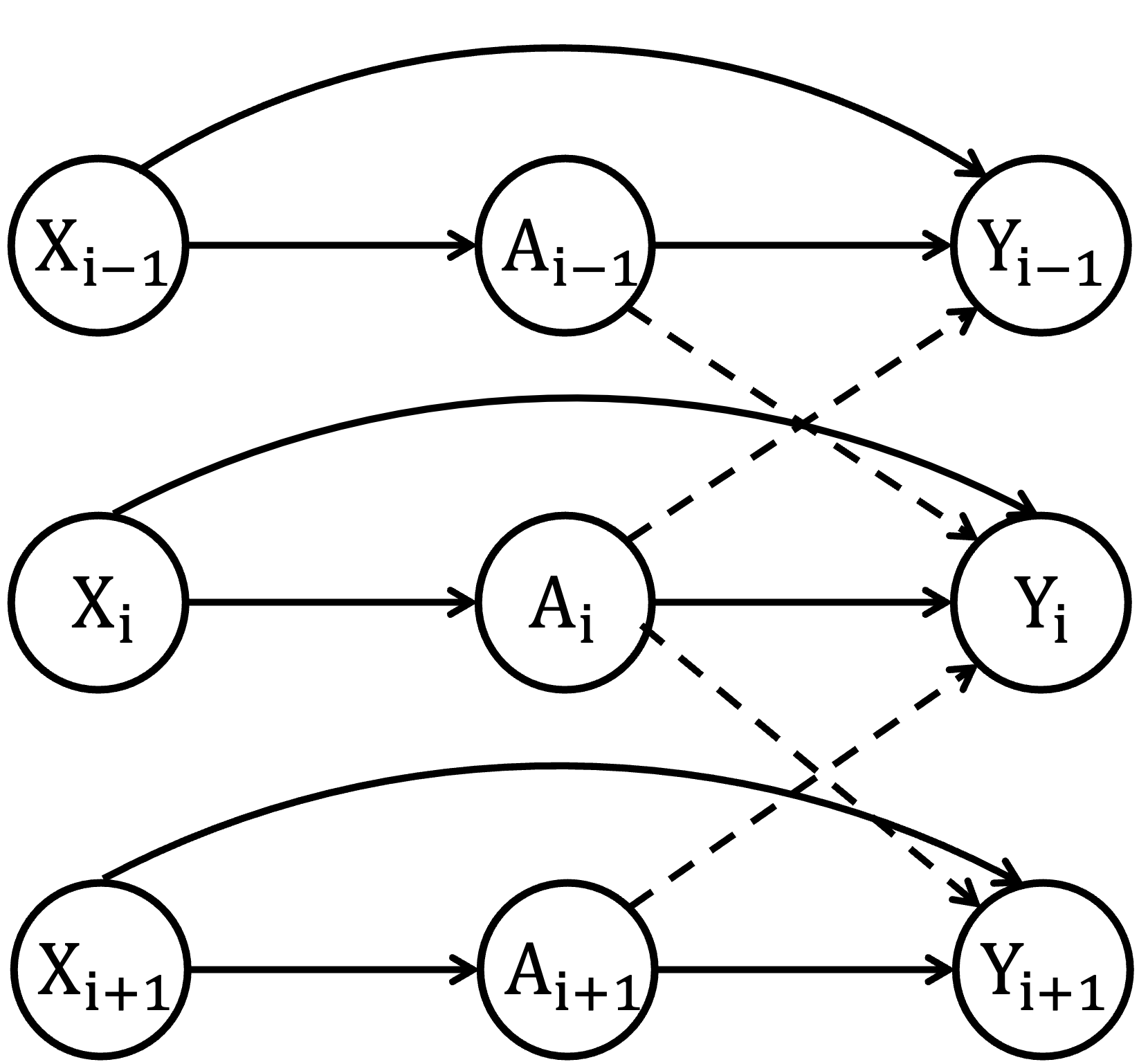}\label{fig: spatial network}}
    \quad
  \subfloat[General network interference (\ref{eq: network spatial})]
  {\includegraphics[width=0.23\textwidth]{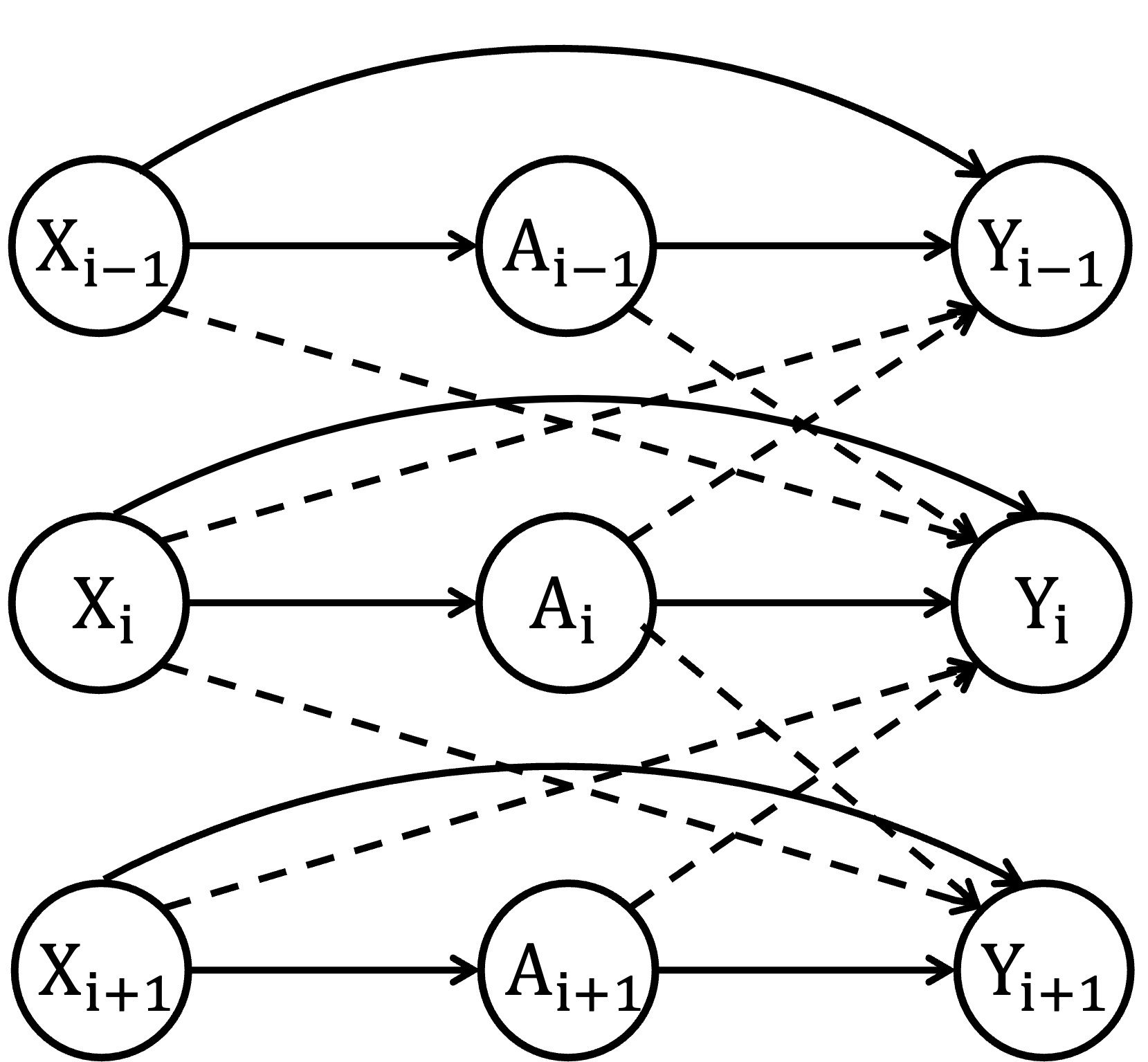}\label{fig: general spatial network}}
  \caption{Variable dependencies under different spatial interference structures. The dashed line represents the spatial interference effect. Each line represents a different region and horizontal adjacency depicts the proximity of regions.}
\end{figure}

\begin{assumption}[General Network Interference]
    The outcome within a particular region is conditionally independent of all other pairs of confounders and treatments, provided the pairs of confounders and treatments of its neighboring regions are given.
    \label{ass: neighbor}
\end{assumption}

We remark that Assumption \ref{ass: neighbor} slightly generalizes existing network inference by permitting the outcome of a region to be influenced by both the treatments and confounders of its neighboring regions; see Figures \ref{fig: spatial network} and \ref{fig: general spatial network} for illustrations. From a mathematical perspective, let $\Ncal(i)$ denote the index set of neighboring regions for region $i$. Let $X_{\Ncal(i)}$ and $A_{\Ncal(i)}$ represent the confounders and treatments of neighboring regions for region $i$, respectively. Then, according to Assumption \ref{ass: neighbor}, 
\begin{equation}
    Y_i= f_i(X_i,A_i, m_i(X_{\Ncal(i)}, A_{\Ncal(i)})) + \epsilon_i,
    \label{eq: network spatial}
\end{equation}
for some interference function $m_i$ that measures the interference effect.

Various methods have been proposed to model the function $m_i(\cdot, \cdot)$. A prominent method is the mean-field approach, in which $m_i(\cdot, \cdot)$ is defined by the average values of $X_{\Ncal(i)}$ and $A_{\Ncal(i)}$ \citep{yang2018mean,li2022random,shi2022}. This can be expressed as:
\begin{equation}
    \label{eq: meanfield}
    m_{\text{MF},i}( X_{\Ncal(i)}, A_{\Ncal(i)}) = \frac{1}{|\Ncal(i)|} \sum_{j\in\Ncal(i)}  \left[X_j ,   A_j\right].
\end{equation}  
The mean-field method greatly simplifies the interference structure by averaging the inputs from neighboring units. However, it also implies a strong parametric assumption, which could limit its flexibility. 
Given these considerations, it is reasonable to explore whether a more robust interference structure can be identified using minimal assumptions. This pursuit leads to the consideration of the permutation invariant assumption, offering a potentially more flexible and comprehensive approach to modeling spatial interference.

\subsection{PIE: A Permutation-Invariant Estimator with Interference Effect }
\label{sec: PIE}
We first introduce some definitions pertaining to permutation-invariant functions. 
\begin{definition}[Permutation operator]
\label{def: PO}
    Given an index set $\Ical = \{1,2,\ldots,N\}$ with $N$ integers, the permutation operator $\varPi$ is a one-to-one mapping $\varPi: \Ical \rightarrow \Ical$ 
    such that $\varPi_{\Ical}=\{\varPi(1),\varPi(2), \ldots, \varPi(N)\}$ is the resulting permutation of $\Ical$. 
    Let the symmetric group $\Scal_N$ denote the set containing all permutations. Apparently, 
    we have $|\Scal_N| = N!$.
\end{definition}
For instance, consider the set $\Ical=\{1,2,3\}$, which has six different permutations. Let $\varPi$ denote the permutation operator such that $\varPi(1)=3, \varPi(2)=1, \varPi(3)=2$. We have $\varPi_{\Ical}=\{3,1,2\}$. With the definition of permutation operator, we impose the permutation invariant assumption on mean outcome function.
\begin{assumption}[Permutation Invariance]
\label{ass:PI}
For any given region $i$, the mean-outcome function $f_i(X_i, A_i, m_i(X_{\Ncal(i)}, A_{\Ncal(i)}))$ exhibits Permutation Invariance with respect to its neighboring regions. This implies that for any permutation $\varPi: \Ncal(i) \rightarrow \Ncal(i)$, the following holds true:
\begin{equation*}
f_i(X_i, A_i, m_i(X_{\Ncal(i)}, A_{\Ncal(i)})) = f_i(X_i, A_i, m_i(X_{\varPi_{\Ncal(i)}}, A_{\varPi_{\Ncal(i)}})).
\end{equation*}
In essence, the outcome function remains unchanged regardless of the order in which the neighboring regions are considered, emphasizing a key characteristic of spatial invariance in the model. 
 \end{assumption}

This assumption is based on the observation that, from the perspective of the target region, all neighboring regions are indistinguishable. Therefore, the interference effect of neighbors on the center region depends solely on their features and is independent of the order. Assumption \ref{ass:PI} provides a problem-specific description of the permutation invariant property within the context of spatial causal inference. A more general definition will be formalized in Section \ref{sec:theory}.  
As an example, it is immediate to see $f_i$ is permutation invariant to neighboring regions with  $m_i(\cdot) = m_{\textrm{MF},i}(\cdot)$ in \eqref{eq: meanfield}. However, Assumption \ref{ass:PI} alleviates the parametric constraint imposed by the mean-field interference function, thereby enhancing its expressiveness. In the subsequent section, we will illustrate through theoretical and experimental analysis how this relaxation enhances expressivity.

Our objective is to develop an estimator that not only adheres to Assumptions \ref{ass: neighbor} and \ref{ass:PI} but also adeptly captures the complex nature of spatial interference. Research on permutation invariant interference structures is scant. The following theorem provides a closed-form expression for any such permutation invariant estimator (PIE) .

\begin{theorem}[Permutation Invariant Estimator (PIE)]\label{theo: PIE}
Assuming that Assumptions \ref{ass: neighbor} and \ref{ass:PI} hold, the mean outcome function $f_i$ can be accurately approximated by the following estimator, achieving any desired level of precision with the appropriate selection of the functions $\phi_i$ and $\psi_i$,
    \begin{align}
        \label{eq:sample}
        &\psi_i(X_i, A_i, m_{\text{PIE},i}(X_{\Ncal(i)}, A_{\Ncal(i)})) \nonumber \\
        with \quad &m_{\text{PIE},i}(X_{\mathcal{N}(i)},A_{\mathcal{N}(i)}) = \frac{1}{|\mathcal{N}(i)|}\sum_{j\in\Ncal(i)}\phi_i(X_j,A_j).
    \end{align}
\end{theorem}

 Theorem \ref{theo: PIE} suggests a versatile method for approximating PI mean outcome functions. The proposed PIE framework, with its focus on averaging over neighboring regions, inherently satisfies the permutation invariance criterion. The mean-field function in (\ref{eq: meanfield}) can be seen as a specific instance of $m_{\text{PIE}, i}$ when $\{\phi_i\}_i$ are identity functions.

It remains to specify $\{\phi_i\}_i$ and $\{\psi_i\}_i$. For practical application, we recommend to employ deep neural networks to parameterize these functions. Figure \ref{network} graphically illustrates the resulting architecture. The treatment and confounding variables from neighboring regions are concatenated into a single vector and input into a feedforward neural network $\phi_i$, capturing their intricate interrelations. Subsequently, to align with 
Assumption \ref{ass:PI}, these neighboring effects are averaged. The output is then combined with the central region's treatment-confounder vector and processed through another feedforward network $\psi_i$, resulting in the final mean outcome value. Section \ref{sec:theory} will demonstrate that this architecture also functions as a universal approximator, capable of precisely approximating any permutation-invariant interference effect function. Contrarily, traditional mean-field structures lack the versatility to effectively represent general permutation-invariant functions.

\begin{figure}      
\center{\includegraphics[width=0.5\linewidth]  {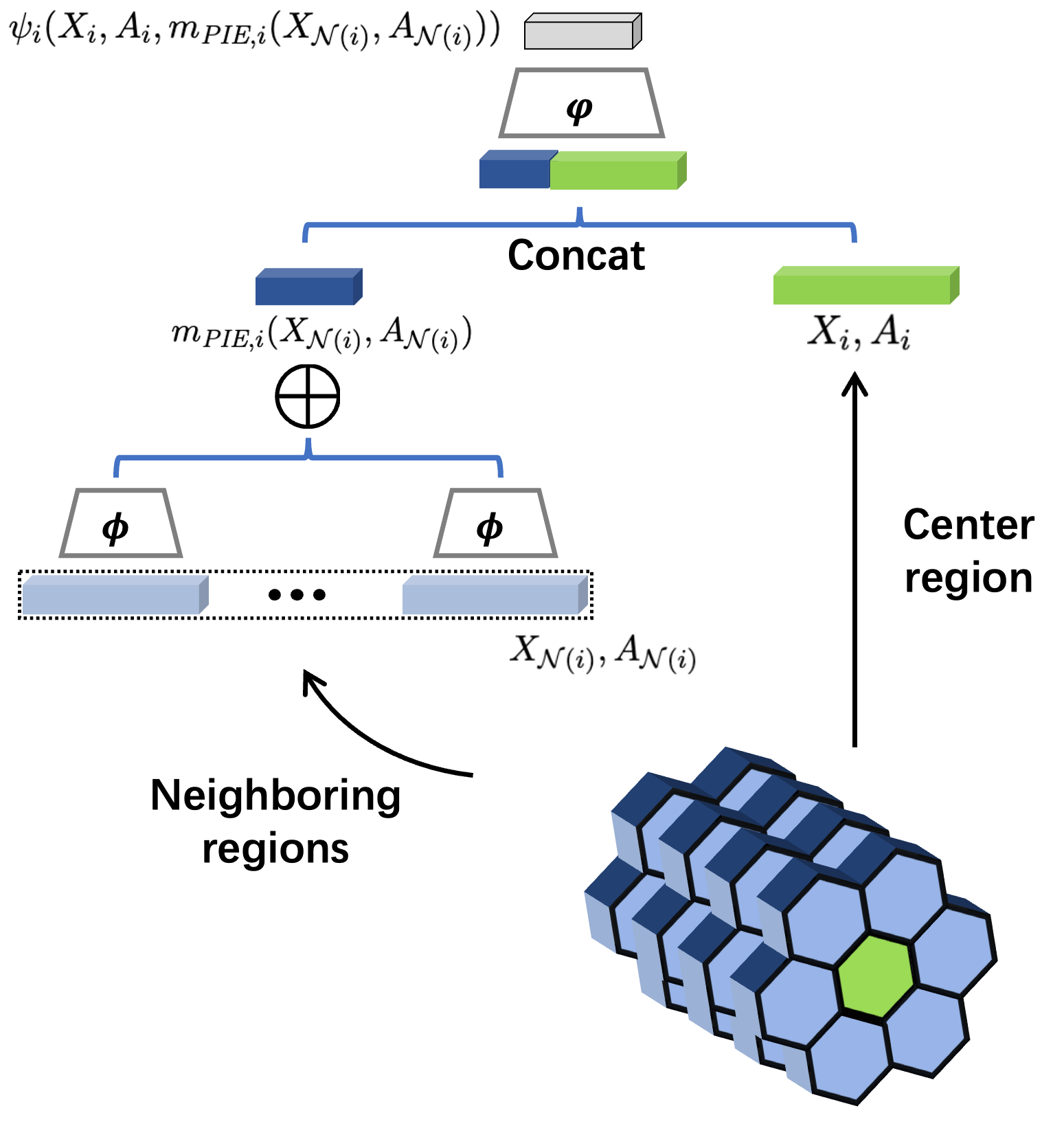}}   
	\caption{The proposed structure is depicted in a graphical visualization. In this representation, the hexagonal prism at the bottom-right corner symbolizes the confounder-treatment vector for a central region (colored green) and its six neighboring regions (colored blue). The vectors from these neighboring regions are simultaneously input into the same neural network, denoted by $\psi$. The aggregated output from this process forms the PIE interference effect function, $m_{\text{PIE}, i}$. Subsequently, the central region's vector is concatenated with $m_{\text{PIE},i}$. The final estimator is then obtained after this combined vector is processed through a feedforward neural network, labeled as $\phi$. } 
	\label{network}  
\end{figure}

To demonstrate the effectiveness of the proposed neural network structure, let us revisit our ride-sharing example. In this system, potential customers use the ride-sharing app to compare prices and decide whether to place an order. Each region's confounding vector $X_i$ includes variables like the total number of orders $O_i$ and the total number of available drivers $D_i$, reflecting the supply and demand dynamics of this marketplace. The system, following a specific dispatch policy, pairs orders with drivers, generating revenue $Y_t$ when an order is completed.

As previously mentioned, this ride-sharing system experiences significant spatial interference, as drivers often gravitate towards areas with more orders and fewer competing drivers. Additionally, the system may strategically assign move drivers to nearby regions with driver shortages, enhancing pickup rates and service coverage. This process, which involves assessing the supply-demand imbalance between neighboring regions, is discussed in detail in \citet{zhou2021graph}. The imbalance in each region, is typically a function of the total orders $O_i$ and drivers $D_i$ and often quantified as $M_{i,t} = |D_{i,t} - O_{i,t}|$ \citep{shi2022}. This results in an average mismatch  $|\Ncal(i)|^{-1}\sum_{j\in\Ncal(i)}|D_{j,t}-O_{j,t}|$ for the $i$th region.

 For an effective assessment of interference, the mean outcome function $f_i$ must accurately reflect this mismatch. The neural network-based PIE is able to approximate this mismatch with high precision. Particularly, when using the ReLU activation function, the proposed $m_{\text{PIE}, i}$ can perfectly replicate the average absolute difference without any approximation error. In contrast, a mean outcome function using the traditional mean-field interference effect can only model functions of $\sum_{j\in\Ncal(i)}D_{j,t}$ and $\sum_{j\in\Ncal(i)}O_{j,t}$. It is, however, not capable of effectively capturing the mismatch represented by $|\Ncal(i)|^{-1}\sum_{j\in\Ncal(i)}|D_{j,t}-O_{j,t}|$.

To conclude this section, here are some key highlights of the proposed estimator: 
\begin{itemize}
    \item 
{\bf Objective and Flexibility:} The primary goal of PIE is to accurately identify suitable mean outcome functions under the permutation invariance condition. The functions $\{\phi_i\}_i$ and $\{\psi_i\}_i$ are instrumental in capturing the complex interference structure, which may vary across different settings. Utilizing deep neural networks to parameterize these functions enables the application of advanced machine learning techniques in spatial causal inference, providing a robust framework for diverse environmental contexts.
\item {\bf  
Permutation Invariance and Extensibility:}  The effectiveness of our approach is anchored in the permutation invariance assumption. If this assumption does not hold, the PIE can be modified to align with necessary extensions. For example, instead of using a simple average in $\phi(X_j, A_j)$, a weighted average could be employed, with weights tailored to the specific characteristics of neighboring regions. This adaptability further enhances the applicability of PIE to a even wider range of scenarios. 
\end{itemize}

\subsection{Application to Policy Evaluation in Nondynamic Settings}
\label{sec: one-stage}
To demonstrate the versatility of the proposed architecture, we apply PIE to OPE in nondynamic settings. The observed data consists of the confounders-treatment-outcome triplets $\{(X_{i,j}, A_{i,j}, Y_{i,j}): 1\le i\le R, 1\le j\le S\}$ measured over time. In our ridesharing application, $\{(X_{i,j}, A_{i,j}, Y_{i,j}): 1\le i\le R\}$ corresponds to the data collected on the $j$th day. We assume these triplets are i.i.d. copies of $\{(X_i,A_i,Y_i): 1\le i\le R\}$ defined at the beginning of Section \ref{sec:nondynamic}. It is important to note that while there are strong temporal correlations within each day's data, it is reasonable to assume independence across different days, provided that the same treatment is implemented throughout each day. This assumption is partly based on the observation that order volumes typically decrease significantly between 1 am and 5 am. Consequently, we treat the observations of each day as new, independent realizations. In Section 3, we will study settings where different treatments are implemented across various time periods within each day.

We next introduce our estimand. In real-world applications, it is often necessary to understand the effects of implementing a specific policy denoted by $\pi=(\pi_1,\ldots,\pi_R)$. Each $\pi_i$ corresponds to a deterministic function of the confounding vector, indicating the treatment assignment for the $i$th agent. Specifically, on the $t$th day, the $i$th agent receives the treatment $\pi_i(X_{i,j})$. To simplify the presentation, we assume that all agents receive the same treatment assignment function, i.e., $\pi_1=\pi_2=\ldots=\pi_R$, which we denote as $\pi$. This assumption allows us to use a single symbol to represent the common mapping. However, it is worth noting that our approach can be easily extended to scenarios where the treatment assignment functions $\{\pi_i\}_i$ differ among agents, accommodating more complex settings.

We are interested in evaluating the average reward under $\pi$, given by 
\begin{eqnarray}\label{eqn:estimand}
   J(\pi)=\sum_{i=1}^R \mathbb{E}[\mathbb{E}(Y_i| \{A_p=\pi(X_p)\}_{p=1}^{R}, \{X_p\}_{p=1}^{R})],
\end{eqnarray}
Here, the initial expectation in \eqref{eqn:estimand} is with respect to the distributions of $\{X_p\}_{p=1}^{R}$, while the second expectation pertains to the conditional distribution of the outcome given all confounder-treatment pairs. 
Following the formulation of our estimand, we construct three OPE estimators utilizing the proposed PIE, corresponding to the value-based estimator, the importance sampling estimator and the doubly robust estimator, respectively. 

\subsubsection{Value-based Method}\label{subsection: de}
The value-based method involves a two-step procedure. Under Assumption \ref{ass: neighbor}, the conditional expectation of $Y_i$ is given by $f_i(X_i, A_i, m_i(X_{\mathcal{N}(i)}, A_{\mathcal{N}(i)}))$, so the initial step involves the estimation of the regression function $f_i$. In the second step, the obtained estimator $\widehat{f}_i$ is plugged in to value-based estimate of the policy value, i.e., 
\begin{eqnarray}
\label{one_vb}
    \widehat{J}_{\textrm{VB}}(\pi)=\frac{1}{S}\sum_{j=1}^S\sum_{i=1}^R \widehat{f}_i(X_{i,j},\pi(X_{i,j}), \widehat{m}_i(X_{\mathcal{N}(i),j}, \pi(X_{\mathcal{N}(i),j}))), 
\end{eqnarray} 
where we use the proposed PIE in Theorem \ref{theo: PIE} to parameterize $\widehat{f}_i$ and $\widehat{m}_i$, respectively. Recall that $X_{\mathcal{N}(i),j}$ denotes the vector obtained by concatenating $\{X_{p,j}\}_{p\in \mathcal{N}(i)}$ and the vector $\pi(X_{\mathcal{N}(i),j})$ is obtained by applying the policy $\pi$ componentwise to each element of $X_{\mathcal{N}(i),j}$. The parameters in $\widehat{f}_i$ and $\widehat{m}_i$ are estimated by minimizing the least square loss function,
\begin{eqnarray}
\label{eq: MSE-non-d}
(\widehat{f}_i,\widehat{m}_i)=\argmin_{(f_i,m_i)}\frac{1}{S}\sum_{j=1}^S [Y_{i,j}-f_i(X_{i,j},\pi(X_{i,j}), m_i(X_{\mathcal{N}(i),j},\pi(X_{\mathcal{N}(i),j)}))]^2.
\end{eqnarray}
The above formulation can be viewed as a regression problem, where the reward $Y_{i,j}$ is regressed on the tuple $(X_{i,j},\pi(X_{i,j}),X_{\mathcal{N}(i),j},\pi(X_{\mathcal{N}(i),j)})$ using the network structures proposed. It is worth mentioning that the value-based estimator may suffer from a significant bias when $\widehat{f}_i$ and $\widehat{m}_i$ are misspecified, rendering it less robust. This motivates us to consider the importance sampling method as an alternative approach. 

\subsubsection{Importance Sampling Method}\label{subsection:is}
The importance sampling method is motivated by the change-of-measure theorem which yields that
\begin{eqnarray*}
	\Mean^{\pi} Y_{i}=\Mean \Big[\underbrace{\frac{\mathbb{I}(A_{i}=\pi(X_{i})) \prod_{p\in \calN(i)}\mathbb{I}(A_p=\pi(X_p))}{\prob(  \cap_{p\in \calN(i)}  \{A_{p}=\pi(X_{p})\}\cap \{A_{i}=\pi(X_{i})\} |X_{i},X_{\calN(i)})}}_{\omega_i^{\pi}(X_i,X_{\calN(i)},A_i,A_{\calN(i)})}Y_{i} \Big],
\end{eqnarray*}
where $\omega_i^{\pi}$ denotes the importance sampling ratio. 
This suggests the following important sampling estimator for the policy value $J(\pi)$, 
\begin{eqnarray*}
	\frac{1}{S}\sum_{j=1}^S \sum_{i=1}^{R}\widehat{\omega}_i^{\pi}(X_{i,j},X_{\calN(i),j},A_{i,j},A_{\calN(i),j})Y_{i,j}, 
\end{eqnarray*}
for some estimated ratio $\widehat{\omega}_i^{\pi}$. 

Nevertheless, the above estimator 
exhibits substantial variance introduced by the importance sampling ratio. To elucidate this limitation, we consider the case where the overall ratio is the product of ratios associated with each individual region in the neighborhood set. As the number of neighborhoods for one unit increases, the variances in the individual ratios accumulate multiplicatively. Consequently, the variance of the overall ratio grows exponentially fast with respect to the number of neighborhoods, leading to an estimator with exceedingly high variance. 

To address this limitation, we notice that under Assumption \ref{ass: neighbor}, $Y_i$ is conditionally independent of $(X_{\calN(i)},A_{\calN(i)})$ given $A_i$, $X_i$ and the interference effect function $m_i(X_{\calN(i)},A_{\calN(i)})$. Utilizing this conditional independence property, we obtain that
\begin{eqnarray*}
    \mathbb{E} [\omega^{\pi}_i(X_i,X_{\calN(i)},A_i,A_{\calN(i)})Y_i]=\mathbb{E} [\overline{\omega}_i^{\pi}(X_i,A_i,m_i(X_{\calN(i)},A_{\calN(i)}))Y_i],
\end{eqnarray*}
where 
\begin{eqnarray*}
\overline{\omega}^{\pi}_i(X_i,A_i,m_i(X_{\calN(i)},A_{\calN(i)}))=\mathbb{E} [\omega^{\pi}_i(X_i,X_{\calN(i)},A_i,A_{\calN(i)})| X_i,A_i,m_i(X_{\calN(i)},A_{\calN(i)})],
\end{eqnarray*}
whose closed-form is given by
\begin{eqnarray*}
	\frac{\mathbb{I}(A_{i}=\pi(X_{i}), m_i(X_{\calN(i)},A_{\calN(i)})=m_i(X_{\calN(i)},\pi(X_{\calN(i)})))}{\prob(A_{i}=\pi(X_{i}), m_i(X_{\calN(i)},A_{\calN(i)})=m_i(X_{\calN(i)},\pi(X_{\calN(i)}))|X_{i},X_{\calN(i)})}. 
\end{eqnarray*}
In practice, the likelihood of the two equalities in the numerator perfectly coincide is quite low. 
Instead of insisting on exact equality, we adopt a different approach to reduce the variance. 
Specifically, we only require that the Euclidean difference between the left-hand side and right-hand side of each equation is smaller than a predefined $\tau$. This yields the following importance sampling ratio, denoted by $\overline{\omega}_{i,\tau}^{\pi}(X_i,A_i,m_i(X_{\calN(i)},A_{\calN(i)}))$, 
\begin{eqnarray*}
    \label{eq:select_data}
    \frac{\mathbb{I}(|A_{i}-\pi(X_{i})|<\tau, |m_i(X_{\calN(i)},A_{\calN(i)})-m_i(\pi(X_{\calN(i)},X_{\calN(i)}))|<\tau)}{\prob(|A_{i}-\pi(X_{i})|<\tau, |m_i(X_{\calN(i)},A_{\calN(i)})-m_i(X_{\calN(i)},\pi(X_{\calN(i)}))|<\tau|X_{i},X_{\calN(i)})}. 
\end{eqnarray*} 
The resulting importance sampling estimator is given by 
\begin{eqnarray*}
    \widehat{J}_{\textrm{IS}}(\pi)=\frac{1}{S}\sum_{j=1}^S \sum_{i=1}^R\widehat{\omega}_{i,\tau}(X_{i,j},A_{i,j},m_i(X_{\calN(i),j},A_{\calN(i),j}))Y_{i,j},
\end{eqnarray*}
where $\widehat{\omega}_{i,\tau}$ denotes some consistent estimator for $\overline{\omega}_{i,\tau}^{\pi}$. The consistency of $\widehat{J}_{\textrm{IS}}(\pi)$ relies on the estimation accuracies of $\widehat{\omega}_{i,\tau}$ and $\widehat{m}_i$. 

\subsubsection{Doubly Robust Method}\label{subsection:dr}
The doubly robust estimator is given by
\begin{eqnarray*}
    \widehat{J}_{\textrm{DR}}(\pi)=\widehat{J}_{\textrm{VB}}(\pi)+\frac{1}{S}\sum_{j=1}^S \sum_{i=1}^R\widehat{\omega}_{i,\tau}(X_{i,j},X_{\calN(i),j},A_{i,j},A_{\calN(i),j})[Y_{i,j}-\widehat{f}_i(X_{i,j},A_{i,j},\widehat{m}_i(X_{\calN(i),j},A_{\calN(i),j}))].
\end{eqnarray*}
By definition, $\widehat{J}_{\textrm{DR}}(\pi)$ can be decomposed into two terms. The first term is essentially the value-based estimator. The second term, is a ``centered" importance sampling estimator where we replace the response $Y_i$ by the residual $Y_{i,j}-\widehat{f}_i(X_{i,j},A_{i,j},\widehat{m}_i(X_{\calN(i),j},A_{\calN(i),j}))$. The purpose of adding the second term is to debias the bias of $\widehat{J}_{\textrm{VB}}(\pi)$ resulting from the estimation of $\{\widehat{f}_i\}_i$. In particular, when $\{\widehat{m}_i\}_i$ is consistent, it can be shown that the consistency of $\widehat{J}_{\textrm{DR}}(\pi)$ relies only on the consistency of $\{\widehat{f}_i\}_i$ or $\{\widehat{\omega}_{i,\tau}\}_i$. Meanwhile, it can be shown that $\widehat{J}_{\textrm{DR}}(\pi)$ converges at a much faster rate than these nuisance functions themselves  \citep{chernozhukov2018double}.

So far, we have presented our approach to estimating the mean outcome in non-dynamic settings. However, interference effects are not confined solely to the spatial realm; they also manifest temporally. The treatments applied at present can impact not only immediate outcomes but also future ones, through their influence on the distribution of forthcoming confounding variables. This necessitates an extension of spatial interference models to account for temporal interference.

\section{Spatio-temporal Interference in Dynamic Settings}

In this section, we will integrate concepts from multi-agent reinforcement learning (MARL) to address additional temporal interference, treating the process as an MDP. This allows us 
model how actions taken at one point in time can affect outcomes at subsequent points. 
We will introduce several estimators designed to accurately evaluate the average treatment effect in scenarios where both spatial and temporal interference are present.

\subsection{Interference Modelling via MDP}

In a dynamic setting, the independence assumption among confounder-treatment-outcome triplets at different time points is no longer valid. Thus, we encounter both spatial and temporal interference, necessitating the introduction of new assumptions to appropriately model both interference structures. 
Spatial interference continues to be described by the spatial interference function $m$, as in the non-dynamic setting. Temporal interference, however, is twofold. Firstly, an individual's response is influenced by the confounders and treatments in neighboring regions. Secondly, the present confounders are affected by the history of past confounder-treatment interactions.

Mirroring Assumption \ref{ass: neighbor} from the non-dynamic context, we approach temporal interference by imposing specific conditional independence assumptions which relate to the future confounder, current confounder-treatment-response triplet, and their historical context. 
For ease of explanation, let $X_{\cdot, t}$ represent the set $\{X_{i,t}\}_i$ at each time $t$, and similarly for $A$ (treatment) and $Y$ (response). The following assumptions detail the conditional independence properties in this dynamic setting.

\begin{assumption}[Conditional mean independence assumption (CIMA)]
     There exist functions $\{r_i\}_i$ such that for each region $i$ (where $1 \leq i \leq R$), the following equation is almost surely satisfied:
   \begin{equation*}
        \mathbb{E}(Y_{i,t}|A_{\cdot,t},X_{\cdot,t},\{A_{\cdot,k}, X_{\cdot,k}, Y_{\cdot,k}\}_{0\leq k<t})=r_i({A}_{i,t},{X}_{i,t},m_i({A}_{\Ncal(i),t},{X}_{\Ncal(i),t})), 
    \end{equation*}
    where  $m_i(\cdot)$ represents the interference effect function. 
\end{assumption}

\begin{assumption}[Markov assumption (MA)]
    There exist Markov transition kernel $\Pcal_i(\cdot)$ such that for any $t\geq 0$ and $x\in\bbs$, we have almost surely that 
    \begin{equation*}
        \Pr\{{X}_{i,t+1}=x|A_{\cdot, t},X_{\cdot,t},\{A_{\cdot, k},Y_{\cdot, k},X_{\cdot,k}\}_{0\leq k<t}\} = \Pcal_i(x;{A}_{i,t},{X}_{i,t},m_i({A}_{\Ncal(i),t},{X}_{\Ncal(i),t})). 
    \end{equation*}
\end{assumption}

To verify these assumptions, one can employ the approach proposed by \cite{shi2020does} and test the conditional independence using appropriate statistical tests \citep[see e.g.,][]{chen2012testing,shah2020hardness,kim2022local,polo2023conditional}. Additionally, these assumptions can be satisfied by concatenating measurements over several decision points and selecting the optimal order. By combining the assumption on spatial interference (Assumption \ref{ass: neighbor}) with these temporal interference assumptions, we establish the assumption on spatio-temporal interference effects.

In a similar vein to Assumption \ref{ass:PI} in a static setting, we propose that both the response function $r_i(\cdot)$ and the transition kernel $\Pcal_i(\cdot)$ are unaffected by the order of neighboring regions. This means for any given region $i$ and any permutation $\varPi: \Ncal(i) \rightarrow \Ncal(i)$, the following holds true:
\begin{align*}
    r_i(X_{i,t}, A_{i,t}, m_i(X_{\Ncal(i),t}, A_{\Ncal(i),t}))=r_i(X_{i,t}, A_{i,t}, m_i(X_{\varPi_{\Ncal(i)},t}, A_{\varPi_{\Ncal(i)},t})), \\
    \Pcal_i(X_{i,t}, A_{i,t}, m_i(X_{\Ncal(i),t}, A_{\Ncal(i),t}))=\Pcal_i(X_{i,t}, A_{i,t}, m_i(X_{\varPi_{\Ncal(i)},t}, A_{\varPi_{\Ncal(i)},t})).  
\end{align*} 
With the interference assumptions in place, conventional reinforcement learning models can be adapted to form PIE estimators. The details and discussions will be presented in the following section.

\subsection{Application to Offline Policy Evaluation}

In the dynamic setting, the observed data consists of $S$ i.i.d samples of $\{(X_{i, t}, A_{i, t}, Y_{i, t}): 1\le i\le R, 1\le t\le T\}$ generated under a specific behavior policy, each of which is a confounders-treatment-outcome triplet measured over time. Consequently, the observed data can be represented as $\{(X_{i, t, j}, A_{i, t, j}, Y_{i, t, j}): 1\le i\le R, 1\le t\le T , 1 \le j\le S \}$, where $j$ indicates the index of the sample. In our illustrative ridesharing example, the index 
$i$ denotes the $i$th spatial unit, $j$ corresponds to the 
$j$th day, and $t$ represents the $t$th time interval within each day. This dynamic model, as compared to the nondynamic one discussed in Section \ref{sec:nondynamic}, facilitates the analysis of scenarios where the company may implement different treatments on each day.

In this section, we consider a known deterministic stationary policy $\pi=(\pi_1,\ldots,\pi_N)$, where each $\pi_i$ is a deterministic decision rule mapping $X_{i,t}\in \R^M$ to $\{0,1\}$ and remains constant across all time periods $t$. 
We assume the conditions of CMIA and MA are met, and extend the concept of average reward to a dynamic context by introducing the discounted cumulative outcome. This is calculated for a given discount factor $0\le \gamma\le 1$ over a time horizon of $T$ steps, as shown in the equation: 
\begin{equation}
J_{\gamma}(\pi)= \sum_{i=1}^R \sum_{t=1}^{T} \gamma^t \mathbb{E}[\mathbb{E}(Y_{i,t}| A_{\cdot,t}=\pi(X_{\cdot,t}), X_{\cdot,t})]. 
\label{eq: dynamic estimand}
\end{equation} 
Here, the first expectation accounts for the distribution of the Markov transition kernels $\{\Pcal_i\}_i$ and the distribution of the initial confounders $X_0$. The second expectation pertains to the conditional distribution of the outcome given the confounder-treatment pairs.

We adapt the three estimators discussed in Section \ref{sec: one-stage} to the dynamic  setting. These are the value-based estimator, the importance sampling estimator, and the doubly robust estimator. Each is tailored to estimate the discounted cumulative outcome function $J_{\gamma}(\pi)$.

\subsubsection{Value-based Method}\label{subsection: vb}

To introduce the value-based estimator, we first simplify the notations. For $0\leq t \leq T$, let 
\begin{eqnarray}
 \label{M}
 (X_{\calN^*(i),t},A_{\calN^*(i),t}) = (X_{i,t}, A_{i,t},X_{\calN(i),t},A_{\calN(i),t}),\nonumber\\
M_i(X_{\calN^*(i),t},A_{\calN^*(i),t})=(X_{i,t}, A_{i,t},m_i(X_{\calN(i),t},A_{\calN(i),t})),   
\end{eqnarray}
where   $\calN^*(i)=\calN(i)\cup \{i\}$ denotes the index of the central and neighboring regions. 
Similarly, we use $\widehat{M}_i$ to denote the corresponding estimated $M_i$ by substituting  $m_i$ with its estimator $\widehat{m}_i$. 
Next, we introduce the Q-function as the expected discounted accumulative outcome given current confounder-treatment pair:
\begin{align}
    Q^{(i,k)}_{\pi}(X_{\cdot,k},A_{\cdot,k}) &= \mathbb{E}[Y_{i,k} + \sum_{t=k+1}^{T} \gamma^{t-k} \mathbb{E}(Y_{i,t}|  X_{\cdot,t},\pi(X_{\cdot,t})) \vert X_{\cdot,k},A_{\cdot,k}] \nonumber \\
    & = \mathbb{E}[Y_{i,k} + \sum_{t=k+1}^{T} \gamma^{t-k} \mathbb{E}\left(Y_{i,t}|  X_{\Ncal^*(i), t},\pi(X_{\Ncal^*(i), t})\right)  \Big|  X_{\Ncal^*(i), k} ,A_{\Ncal^*(i), k}],
    \label{eq: def-Q}
\end{align}
where the second equation comes from assumptions CMIA and MA and we rewrite $Q_{\pi}^{(i,k)}(X_{\cdot,k},A_{\cdot,k})$ to $Q_{\pi}^{(i,k)}(X_{\Ncal^*(i), k},A_{\Ncal^*(i), k})$ due to  the conditional independence. The discounted cumulative outcome $J_{\gamma}(\pi)$ can be represented using the Q-function:
\begin{equation}
    J_{\gamma}(\pi) = \mathbb{E}\left[\sum_{i=1}^R Q^{(i,1)}_{\pi}( X_{\Ncal^*(i), 1},A_{\Ncal^*(i), 1}) \bigg\vert A_{\Ncal^*(i), 1} \sim \pi_i(X_{\Ncal^*(i), 1})\right],
    \label{eq: V-Q}
\end{equation}
where the expectation is taken with respect to the  distribution of $X_{\cdot,1}$. This motivates us to first estimate the Q-function and then   construct the value based estimator $\widehat{J}_{\gamma}(\pi)$.

The Q-function can be estimated with the fitted Q-evaluation algorithm \citep{le2019batch}. For each region $1\le i\le R$, we first set $\widehat{Q}_{\pi}^{(i,T+1)} = 0$ and for $t = T,  \ldots, 1$, we iteratively solve:
\begin{eqnarray}
(\widehat{Q}_{\pi}^{(i,t)},  \widehat{M}_i) \leftarrow \arg\min_{(Q, M)} \sum_{j=1}^S \left[Y_{i,t,j}+\gamma \widehat{Q}^{(i,t+1)}_{\pi}\left(\widehat{M}_i(X_{\calN^*(i),t+1,j},\pi(X_{\calN^*(i),t+1,j}))\right) \right.
\\  \left. - Q(M(X_{\calN^*(i),t,j},A_{\calN^*(i),t,j}))\right]^2,
\label{opt_qvb}
\end{eqnarray} 
Utilizing the relationship between Q-function and $J_{\gamma}$ in \eqref{eq: V-Q}, we can construct   $\widehat{J}_{\gamma}^{VB}(\pi)$ as follows: 
\begin{eqnarray}
\widehat{J}_{\gamma}^{VB}(\pi) 
= \frac{1}{S}\sum_{j=1}^S\sum_{i=1}^R \widehat{Q}^{(i,1)}_{\pi}(\widehat{M}_i(\pi_{i}(X_{\calN^*(i),1,j}),X_{\calN^*(i),1,j}))).
\label{eq: VB}
\end{eqnarray}

\subsubsection{Importance Sampling Method}\label{subsection: is}
The importance sampling estimator is constructed based on the marginal density ratio proposed in \citet{liu2018breaking}:
\begin{align*}
\mu_{\pi}^{(i)}(M_i(X_{\calN^*(i),k},A_{\calN^*(i),k}))&=\frac{\sum_{t = 1}^T \gamma^t p_{\pi,t}(M_i(X_{\calN^*(i),k},A_{\calN^*(i),k}))/\sum_{t=0}^T \gamma^t}{ p_{i}(M_i(X_{\calN^*(i),k},A_{\calN^*(i),k}))} \\ 
&\approx (1-\gamma)\frac{\sum_{t = 1}^T \gamma^t p_{\pi,t}(M_i(X_{\calN^*(i),k},A_{\calN^*(i),k}))}{p_{i}(M_i(A_{\calN^*(i),k},X_{\calN^*(i),k}))} ,
\end{align*} 
where $p_{\pi,t}$ denotes the probability mass function of confounder-treatment pair $(X_{\calN^*(i),t},A_{\calN^*(i),t})$ under $\pi$ at time step $t$, 
and $p_{i}$ denotes the stationary distribution  
under the behavior policy. 
The above approximation  requires $T\rightarrow \infty$, a condition that is reasonable in ridesharing applications where $T$ typically equals 24 or 48. This is considered moderately large relative to the sample size $S$ (number of days). 

Note that 
$\mu_{\pi}^{(i)}$ is the ratio between the average discounted probability distribution of $M_i(X_{\calN^*(i)},A_{\calN^*(i)})$ under the target policy and its data distribution under the behavior policy. Using the change of measure theorem, 
$J_{\gamma}(\pi)$ can be represented as:
\begin{equation*}
    J_{\gamma}(\pi) = \mathbb{E}_{X,A \sim p}\left[\frac{1}{(1-\gamma)} \sum_{i=1}^R \mu_{\pi}^{(i)}(M_i(X_{\calN^*(i), k},A_{\calN^*(i), k }))Y_{i,k}\right]
\end{equation*}
Its estimator $\widehat{\mu}_{\pi}^{(i)}$ can be computed by solving the following minimax optimization, 
\begin{eqnarray*}
	\arg\min_{\mu\in \Omega}\sup_{f\in \mathcal{F}} \{\Mean L(\mu,f)\}^2,
\end{eqnarray*}
for some function classes $\Omega$ and $\mathcal{F}$, where $L(\mu,f)$ equals
\begin{eqnarray*}
	\Mean_{A_{\calN^*(i),t+1}\sim \pi(X_{\calN^*(i),t+1})} \mu(M_i(X_{\calN^*(i),t},A_{\calN^*(i),t}))\{\gamma f(M_i(X_{\calN^*(i),t+1},A_{\calN^*(i),t+1}))\\-f(M_i(X_{\calN^*(i),t},A_{\calN^*(i),t}))\}+(1-\gamma) \Mean_{A_{\calN^*(i),0}\sim \pi(X_{\calN^*(i),0})} f(M_i(X_{\calN^*(i),1},A_{\calN^*(i),1})). 
\end{eqnarray*}
The expectation in the above expression can be approximated by the sample mean. To simplify the
calculation, we may choose $\mathcal{F}$ to be a reproducing kernel Hilbert space (RKHS). This yields a closed form expression
for $\sup_{f\in \mathcal{F}} \{\Mean L(\mu,f)\}^2$ \citep[see e.g.,][]{liu2018breaking,uehara2020minimax,kallus2022efficiently}. 
Notice that in the above optimization, although $M_i$ is unknown, we can replace it by  $\{\widehat{M}_i\}_{i=1}^{R}$ derived in \ref{opt_qvb}. 
Then resulting estimator is given by
\begin{eqnarray*}
    \widehat{V}_{1}^{IS}(\pi) = \frac{1}{ST(1-\gamma)} \sum_{i=1}^R \sum_{j=1}^S\sum_{t=1}^T \widehat{\mu}_{\pi}^{(i)}(\widehat{M}_i(X_{\calN^*(i),t,j},A_{\calN^*(i),t,j}))Y_{i,t,j}.
\end{eqnarray*}

\subsubsection{Doubly Robust Method}\label{subsection: dr}
Combining the direct estimator and the importance sampling estimator yields the following doubly-robust estimator
\begin{eqnarray*}
    &&\widehat{V}_1^{DR}(\pi) = \widehat{J}_{\gamma}^{VB}(\pi) + \frac{1}{ST(1-\gamma)} \sum_{i=1}^R \sum_{j=1}^S \sum_{t=1}^{T} \widehat{\mu}_{\pi}^{(i)} (\widehat{M}_i( X_{\calN^*(i),t,j},A_{\calN^*(i),t,j})) \\
    &\times& \{Y_{i,t,j}+\gamma \widehat{Q}^{(i,t+1)}_i(\widehat{M}_i( X_{\calN^*(i),t+1,j},\pi(X_{\calN^*(i),t+1,j})))-\widehat{Q}^{(i,t)}_i(\widehat{M}_i( X_{\calN^*(i),t,j},A_{\calN^*(i),t,j})) \}.
\end{eqnarray*}
Similar to the doubly robust estimator in non-dynamic settings (see Section \ref{subsection:dr}), the first term is the value based estimator where the second term differs from the importance sampling estimator by substituting the outcome $Y$ with the  temporal difference residual $Y+\gamma Q^{(t+1)}- Q^{(t)}$. 
Either when the value based or important sampling estimator is consistent, the doubly robust estimator is consistent. 

\section{Theoretical Results}\label{sec:theory}

This section provides a detailed analysis of the statistical properties of  PIE. 
Initially, we explore the theoretical properties of PIE in a generic supervised learning setting. This includes an examination of their consistency, convergence rate, and minimax optimality. Subsequently, we apply these theoretical results to evaluate the sample efficiency of value-based PIE  in both non-dynamic (see Section \ref{subsection: de}) and dynamic settings (see Section \ref{subsection: vb}).

\subsection{Theoretical properties of PIE}

We establish the theoretical properties of PIE in a generic supervised learning setting. Let $X = [X_1, X_2, \ldots, X_N] \in \R^{M \times N}$ denote the predictor and $Y \in \R$ denote the outcome variable. Our goal is to estimate the conditional mean function $f^*(x)=\Mean (Y|X=x)$. Toward that end, 
we restrict our attentions to the class of permutation invariant estimators $\Fcal$ 
defined as
\begin{equation*}
\label{eq: naive-PIE}
\Fcal = \left\{f:\mathbb{R}^{M\times N}\mapsto\mathbb{R}\bigg|f(X) = \psi\left(\sum_{q=1}^N \varphi(X_q)\right),\varphi,\psi\in\Fcal_{\text{ReLU}}\right\}, 
\end{equation*}
where the detailed definition of  ReLU neural networks is given in Appendix. Consider the following empirical risk minimization (ERM) predictor 
\begin{equation}
\widehat{f} = \arg \min_{f\in \Fcal} \frac{1}{n}\sum_{j=1}^n [Y_{(j)}-f(X_{(j)})]^2,
\label{eq: MSE}
\end{equation}
 where $\{(X_{(j)},Y_{(j)}):1\le j\le n\}$ denote $n$ i.i.d. copies of $(X,Y)$. 
 

We introduce the following assumptions

\begin{definition}[Permutation Invariant Functions]
\label{def: PI}
A function $f$ is defined to be permutation invariant function if    for any permutation $\varPi \in \mathcal{S}_{N}$.
\begin{equation*}
        f(X) = f(X_{\varPi(1)}, X_{\varPi(2)}, \ldots, X_{\varPi(N)}),
\end{equation*}
where the permutation operator $\varPi$ is defined in Definition \ref{def: PO}.
\end{definition}



\begin{assumption}[Permutation Invariance]
    Assume the target function $f^*$ is a permutation invariant function as defined in Definition \ref{def: PI}.
\end{assumption}

Before establishing the consistency of the PIE, we need the following assumptions on the boundedness and continuity of target function: 
\begin{assumption}[Boundedness]
    \label{ass: bounded}
Let $\|f\|_{\infty}=\sup_{x}|f(x)|$ and $F$ denote the envelope function of $\mathcal{F}$, i.e., $F(x)=\sup_{f\in \Fcal} |f(x)|$. 
Assume there exist some constant $J>0$ such that
\begin{eqnarray*}
\max\left\{\|f^*\|_{\infty},\|F\|_{\infty},|Y|\right\}
    \leq J.
\end{eqnarray*} 
\end{assumption}


\begin{assumption}[Continuity]
    \label{ass: continues}
Assume the target function $f^*$ is continuous function defined on $[0,1]^{M\times N}$.
\end{assumption}
The assumption of a continuous and bounded target function is commonplace in nonparametric settings. 
The response variable $Y$ is often assumed to be bounded in the RL literature \citep[see e.g.,][]{fan2020theoretical}. 

The following theorem specifies the consistency of PIE.
\begin{theorem}[Consistency of PIE]
\label{theory: consistency}
Under Assumption \ref{ass: bounded}, the  ERM estimator $\widehat{f}$ is consistent for the target function $f^*$ in the sense that as $n\to\infty$,
\begin{equation*}
    \|\widehat{f}-f^*\|_{L_2(P_X)}^2 \to 0,
\end{equation*}
in which the convergence occurs with probability at least $1 - 2/n$. Here, $\|\cdot\|_{L_2(P_X)}$ denotes the $L_2$ norm associated with the probability distribution function of $X$ (denoted by $P_X$), i.e., $\|f\|^2_{L_2(P_X)}=\int_x f(x)^2  d P_X$.
\end{theorem}
To further establish the convergence rate of  $\widehat{f}$, an additional smoothness assumption is required.

\begin{assumption}[Smoothness of target function]
    \label{ass: Sobolev}
    Consider a Sobolev ball $\Wcal_J^{\beta,\infty}([0, 1]^{M\times N})$, characterized by smoothness parameter $\beta \in \N_{+}$, is defined as follows, 
    \begin{equation*}
        \Wcal_J^{\beta,\infty}([0, 1]^{M\times N})) = \left\{ f:  \max_{\p: |\p| \leq \beta} ess\sup_{x \in [0, 1]^{M\times N}}  |D^{\p} f(x)| \leq J \right\},
    \end{equation*}
	where $\p = (p_{1,1}, \ldots, p_{M, N})$, $|\p| = p_{1,1} + \ldots + p_{M,N}$, and $D^{\p} f$ denotes  the weak derivative.
 We assume that $f^*$ lies in the Soblev ball, i.e., $f^*\in \Wcal_J^{\beta,\infty}([0, 1]^{M\times N}))$.
\end{assumption}  

The smoothness assumption implies that all derivatives of $f^*$ up to order $\beta - 1$ must be Lipschitz continuous. This assumption is frequently employed to establish the convergence rate of neural network-type estimators \citep[see e.g.,][]{approx, Tengyuan} and supports the effective approximation of smooth functions by ReLU networks. 
The following theorem establishes the convergence rate of the PIE.

\begin{theorem}[Convergence rate of PIE]
\label{theory: convergence}
Under Assumptions  \ref{ass: bounded} and \ref{ass: Sobolev}, the finite sample convergence rate of $\widehat{f}$ in approximating  $f^*$ is described by 
\begin{equation*}
    \| \widehat{f} - f^*\|^2 _{L^2(P_X)} = O_p(n^{-\frac{2\beta}{MN+2\beta}}). 
\end{equation*}
\end{theorem}

Theorems \ref{theory: consistency} and \ref{theory: convergence} introduce pioneering theoretical findings regarding the statistical estimation properties of PIEs. Previous research has primarily focused on the representational capabilities of PIEs and suffered from some limitations:

\begin{itemize}
    \item \textbf{Theorem 2 in \citep{zaheer2017deep}}: This theorem establishes that any continuous function \( g(X) \) is permutation invariant if and only if it can be expressed as \( g(X) = \psi(\sum_{x\in X} \phi(x)) \). However, this result was proven for \( X \in \R^N \), and its applicability to \( X \in \R^{M\times N} \) — crucial for processing confounder-treatment vectors in neighboring regions — is less straightforward.
    \item \textbf{Theorem 3.2 in \citep{pine}}: Extending the input domain from vectors to matrices, it proves that any permutation invariant function \( g(X) \) for \( X \in \R^{M\times N} \) can be approximated with arbitrary precision by a function in the form \( g(X) = \psi(\sum_{x\in X} \phi(x)) \), with \( x \in \R^M \). However, this theorem does not detail the precise rate of convergence.
\end{itemize}
To the best of our knowledge, our work is the first to derive the convergence rate of PIEs and extend its application to value-based methods in OPE; see Sections \ref{subsection:nd} and \ref{subsection:dy}.

It is noteworthy that standard neural network-type estimators, even without leveraging the permutation invariant property, can achieve a convergence rate of 
$O(n^{-2\beta/(MN+2\beta)})$ \citep[see e.g.,][]{Tengyuan}. To explore further, we have rigorously derived the minimax rates for PI estimators in Theorem \ref{theo: minimax}. Our findings reveal that no estimator, even with the incorporation of the permutation invariance assumption, can exceed a convergence rate of 
$O(n^{-2\beta/(MN+2\beta)})$ in a minimax sense. This 
suggests that the cannot enhance the convergence rate beyond this established bound. 


\begin{assumption}
\label{ass: bounded-p}
Let $p_X$ denote the probability density function of $X$. There exist two positive constants $\underline{P_X}$ and $\overline{P_X}$ such that 
\begin{equation*}
0<\underline{P_X} \leq p_X(x) \leq \overline{P_X}<\infty, 
\end{equation*}
for any $x\in [0,1]^{M\times N}$.
\end{assumption}

Building on this assumption, we establish the minimax optimality of the PIE for estimating permutation invariant functions.

\begin{theorem}[Minimax Optimality of PIE]
\label{theo: minimax}
Given the distribution of the confounding variable $X \in \R^{M\times N}$ adhering to Assumption \ref{ass: bounded-p}, the minimax risk for approximating any permutation invariant target function that meets Assumptions \ref{ass: bounded} and \ref{ass: Sobolev} is subject to the following lower bound:
\begin{equation}
\inf_{\widehat{f} \in \Acal_n}\sup_{f \in P^{\beta, M \times N}} |\widehat{f} - f|^2_{L^2(P_X)} \geq C(\beta, M, N)n^{-\frac{2\beta}{2\beta + MN}}, 
\end{equation}
where $C(\beta, M, N)$ is a constant depending only on $M,N$ and $\beta$, $\Acal_n$ represents the space of all measurable functions of the confounding variable in $L_2(P_X)$, and $P^{\beta, M\times N}$ includes all functions that satisfy Assumptions \ref{ass: bounded} and \ref{ass: Sobolev}.
\end{theorem}
Finally, we remark that despite that employing the permutation invariance assumption cannot improve the minimax rate of convergence, it does significantly reduce the variance of the estimator. This variance reduction has been substantiated in our numerical studies, as seen in Sections \ref{sec:numerical}.

\subsection{Nondynamic Results}
\label{subsection:nd} 

We establish convergence rate for the OPE estimator in a nondynamic setting in this subsection. 
Accordingly, the notation used here aligns with that in Section \eqref{sec: one-stage}. 


Recall that under Assumption \ref{ass:PI}, the conditional mean of the response of the $i$th region is denoted by $f_i({X}_i,{A}_i, m_i( {X}_{\Ncal(i)},{A}_{\Ncal(i)}))$.
Let $\Gcal$ represent the PIE function class defined in Theorem \ref{theo: PIE} which approximates the mean-outcome function $f_i$. 


For the nondynamic setting, we impose similar boundedness and smoothness assumptions to Assumptions \ref{ass: bounded} and \ref{ass: Sobolev}. 

\begin{assumption}
\label{ass: bounded-non}
For any $i=1,\cdots,R$ and  any $\tilde{f}\in \Gcal$, there exist a constant $J$ that:
\begin{align*}
 \max\{\|f_i\|_{\infty},|Y_i|,\|\tilde{f}\|_{\infty}\} \leq J.
\end{align*}
Furthermore, it is assumed that each $f_i$ resides within the Sobolev space $\Wcal_J^{\beta,\infty}([0, 1]^{(M+1)\times(N+1)})$.
\end{assumption} 

We extend the results of Theorem \ref{theory: convergence} to derive the convergence rate for estimator in Theorem \ref{theo: PIE}. 

\begin{corollary}[Convergence Rate of  PIE]
\label{co: PPIE rate}
Under Assumption \ref{ass:PI}, Smoothness Assumption that  $f_i\in\Wcal_J^{\beta,\infty}([0, 1]^{(M+1)(N+1)}))$ for any $i=1,\cdots,R$. We derive $\widehat{f}_i\in\Gcal$ in estimation scheme shown in \eqref{eq: MSE}, the PIE  exhibits consistency and the following finite sample convergence rate:
\begin{equation}
        \quad \|\widehat{f}_i-f_i\|^2_{L^2(P_X)} = O_p(S^{-\frac{2\beta}{(M+1)(N+1)+2\beta}}).
\end{equation}
\end{corollary}

To facilitate the derivation of the convergence rate of the OPE estimator $\widehat{J}_{\text{VB}}(\pi)$, we employ a cross-fitting scheme as outlined in \citep{chernozhukov2018double} and \citep{athey2021policy}, the detail of cross-fitting scheme can be found in Appendix.
With above preparation, the following corollary is establishes the convergence rate of OPE estimator $\widehat{J}_{VB}(\pi)$.

\begin{corollary}
\label{coro: nondynamic}
Under Assumption \ref{ass: bounded-non}, the finite sample convergence rate of $\widehat{J}_{\textrm{VB}}(\pi)$ in approximating the true value $J(\pi)$ is described by
\begin{equation}
|\widehat{J}_{\textrm{VB}}(\pi) - J(\pi)| \leq O(RS^{-\frac{\beta}{(M+1)(N+1) + 2\beta}}) + e,
\end{equation}
where this holds for any $e > 0$ with a probability of at least $1 - \exp(-\frac{Se^2}{8J^2})$.
\end{corollary}



\subsection{Dynamic Results}
\label{subsection:dy}






To establish the statistical properties of $\widehat{J}_{\gamma}^{VB}(\pi)$ in dynamic settings,
we begin by defining the transition operator $\Pcal^{\pi}$:
\begin{align}
    &(\Pcal^{\pi}f^{(t)}_i)\left(X_{\calN^*(i),t,j},A_{\calN^*(i),t,j}\right)\nonumber \\=  &\mb{E}\left\{ f^{(t+1)}_i\left(X_{\calN^*(i),t+1,j},\pi(X_{\calN^*(i),t+1,j})\right)\bigg|X_{\calN^*(i),t,j},A_{\calN^*(i),t,j}\right\},
\end{align}
where $M_i$ has been defined in \eqref{M}.
The Bellman operator $\Tcal^{\pi}$ can therefore be defined as
\begin{eqnarray*}
    (\Tcal^{\pi}f^{(t)}_i)\left(X_{\calN^*(i),t,j},A_{\calN^*(i),t,j}\right) = r_{i}\left(X_{\calN^*(i),t,j},A_{\calN^*(i),t,j}\right) + \gamma  (\Pcal^{\pi}f^{(t)}_i)\left(X_{\calN^*(i),t,j},A_{\calN^*(i),t,j}\right)
\end{eqnarray*}

Based on the Bellman operator, we redefine the Smoothness Assumption.
\begin{assumption}
\label{ass: Tass} 
For any $i = 1,\cdots,R$, it is assumed that:
\begin{itemize}
    \item $r_i\in \Wcal_J^{\beta,\infty}([0, 1]^{(M+1)\times(N+1)})$.
    \item For any $f_i\in\Gcal$, we have $\Tcal f_i\in\Wcal_J^{\beta,\infty}([0, 1]^{(M+1)\times(N+1)})$.
\end{itemize}
\end{assumption}


Assumption \ref{ass: Tass} addresses the smoothness of the mean-outcome function and target function in training. 
Recall $p_{i}$  denote the stationary probability measure  under behavior policy, and  $p_{\pi,t}$ denote probability measure  under policy $\pi$ at time step $t$. Distribution shift caused by  policy distinction can be controlled by the following assumption

\begin{assumption}[Concentration]
\label{ass: concentration}
The concentration coefficient at time $t$ is defined as
\begin{equation}
\kappa_t^2 = \E_{p_i}\left\vert\frac{dp_{\pi, t}}{dp_i}\right\vert^2,
\end{equation}
and it is assumed that for any $1 \leq t \leq T$, there exists a constant $\kappa$ such that $\kappa_t \leq \kappa$.
\end{assumption}

The concentration assumption is standard in batch reinforcement learning \citep[see e.g.,][]{fan2020theoretical,chen2019information,xie2020q,xie2021batch}. It essentially stipulates that the policy used in the dataset should be sufficiently diverse to ensure adequate coverage of the $(X \times A)$ probability space, a condition easily met in our application as the data follows a random policy. 
Moreover, the the cross-fitting scheme is employed in this subsection as well.
With these assumptions in place, we can now extend to the OPE estimator in dynamic setting.

\begin{corollary}
\label{coro: dynamic}
Under Assumptions \ref{ass: Tass} and \ref{ass: concentration}, for any $e > 0$, the finite sample convergence rate of $\widehat{J}_{\gamma}^{VB}(\pi)$ in approximating the true value $J_{\gamma}(\pi)$ is:
\begin{align}
|\widehat{J}_{\gamma}^{VB}(\pi) - J_{\gamma}(\pi) | \leq O(\kappa TRS^{-\frac{\beta}{(N+1)(M+1) + 2\beta}}) + e,
\end{align}
with a probability of at least $1-\exp(-\frac{Se^2}{8J^2})$.
\end{corollary}

The bound in Corollary \ref{coro: dynamic} differs from that in Corollary \ref{coro: nondynamic} due to the added complexity of  temporal interference in the dynamic setting, resulting in a linear dependence on $T$. However, the convergence rate relative to the sample size $S$ remains consistent with the non-dynamic setting. Importantly, $\widehat{J}_{VB}(\pi)$ converges with high probability.

\section{Numerical Study}\label{sec:numerical}

In this section, we conduct comprehensive numerical experiments to demonstrate the superior empirical performance of our proposed method compared to established alternatives. 
Section \ref{simu:one_stage} focuses on a comparison between our proposed method and the mean-field approach in a nondynamic setting. 
In Section \ref{simu:multi_stage}, we expand our analysis to dynamic settings. 
Finally, Section \ref{simu:real_data} shifts the focus to a synthetic environment derived from a real-world ride-sharing platform dataset. 


\subsection{Nondynamic Simulation}
\label{simu:one_stage}

In this subsection, we focus on the nondynamic setting  introduced in Section \ref{sec:nondynamic}. We begin by dividing the entire space into $R=l\times l$ non-overlapping spatial units, where $l$ is chosen from $\{5,10\}$. The data are generated as follows:
\begin{enumerate}
    \item \textbf{Confounder}: Each confounder vector $X_i$ is two-dimensional, represented by $(U_i,V_i)$, generated under the CAR model detailed in Section 2.1 of \citet{reich2021review};
    \item \textbf{Treatment}: $\{A_i\}_i$ are independently sampled from a Bernoulli distribution with a success probability of 0.5.
    \item \textbf{Response}: $\{Y_i\}_{1\le i\le R}$ are generated according to the following additive model: 
\begin{eqnarray} \label{simulationModel}
    Y_i = 0.1\times (A_{i}\beta_1 +A^{\top}_{\Ncal(i)}\beta_{2,1:|\Ncal(i)|}) + g(X_i,A_i)\gamma_1 +  \gamma_{2,1:|\Ncal(i)|}^{\top} g(X_{\Ncal(i)},A_{\Ncal(i)}) +\epsilon_{i},
\end{eqnarray}
with the following components:
\begin{itemize}
    \item \textbf{Linear component coefficients}:
    \begin{itemize}
        \item $\beta_1 \in \mathbb{R}$: Regression coefficient for the treatment variable $A_i$.
        \item $\beta_{2,1:|\Ncal(i)|}$: Subvector of $\beta_2 \in \mathbb{R}^R$, comprising the first $|\Ncal(i)|$ elements, corresponding to the regression coefficients for the treatments of the $i$th unit's neighboring regions.
    \end{itemize}

    \item \textbf{Nonlinear function and coefficients}:
    \begin{itemize}
        \item $g$: A potentially nonlinear scalar function applied to confounder-treatment pairs
        \item $g(X_{\Ncal(i)},A_{\Ncal(i)})$ applies the scalar function $g$ element-wise to each pair in $\{(X_j, A_j): j \in \Ncal(i)\}$, producing an $|\Ncal(i)|$-dimensional vector.
        \item $\gamma_1 \in \mathbb{R}$: Regression coefficients associated with the function $g$ applied to $(X_i, A_i)$.
        \item $\gamma_{2,1:|\Ncal(i)|}$: Subvector of $\gamma_2 \in \mathbb{R}^R$, consisting of the first $|\Ncal(i)|$ elements, corresponding to the regression coefficients for the treatments of the $i$th unit's neighboring regions.
    \end{itemize}
    \item \textbf{Residuals}: $\{\epsilon_i\}_{1\le i\le R}$: Independent standard normal random variables. 
\end{itemize}
\end{enumerate}
Notice that under Model \eqref{simulationModel}, the interference effect is captured by the terms $0.1A_{\Ncal(i)} \beta_{2,1:|\Ncal(i)|} + \gamma_{2,1:|\Ncal(i)|}^{\top} g_i(X_{\Ncal(i)}, A_{\Ncal(i)})$. 
We evaluate the performance of our method under three distinct scenarios by varying the link function $g$ and the regression coefficients: 

\begin{enumerate}
    \item{} \textbf{Linear Setting}: This setting is relatively straightforward, as both  the treatment and interference effects are linear. The mean-field method is expected to be effective in this linear model. Specifically, we set $g(X_{i},A_{i}) = U_{i}+V_{i}$, which is independent of the treatment. The parameters are chosen as $\gamma_1 = \beta_1 = 1.5$ and $\gamma_2 = \beta_2 = (-0.5, -0.5, -0.5, -0.5, \ldots)$.
    \item{} \textbf{Nonlinear Setting I}:  
This scenario incorporates an interaction term between the confounder and the treatment, violating the linear structure and challenging the mean-field method. However, our proposed method can accurately model this interaction, adhering to the permutation invariance assumption.
Specifically, we set $g(X_{i}, A_{i})$ to $A_{i} \times U_{i}$. 
The parameters are given as $\gamma_1 = \beta_1 = 1.5$ and $\gamma_2 = \beta_2 = (-0.5, -0.5, -0.5, -0.5, \ldots)$. 
    
    \item{} \textbf{Nonlinear Setting II}:  
    This is a more complex scenario where $\beta_2$ and $\gamma_2$ are not constant vectors, deviating from both permutation invariance and linearity assumptions. This setting tests the robustness of the mean-field method and our proposed method under model mis-specification.
Specifically, we set $g(X_i, A_i) = A_i \times U_i$. The parameters are given as $\gamma_1 = \beta_1 = 1.5$, and $\gamma_2 = \beta_2 = (-0.2, -0.8, -0.2, -0.8, \ldots)$. 
\end{enumerate}   
Additionally, we consider a deterministic linear target policy $\pi$ such that for any $X_i$, $\pi(X_i)=1$ if and only if $U_i\kappa+V_i(1-\kappa)>0.5$ for some $\kappa \in [0,1]$. We vary $\kappa\in \{0.2, 0.5, 0.8\}$ to investigate the performance with different target policies. 
We also 
fix the number of days $S$ to 100.  

Figure \ref{one_results} presents 
the results of our proposed value-based estimator (denoted as PIE) and the baseline value-based estimator using mean-field approximation (denoted as Mean-field). In the linear setting, both methods show comparable performance, with our method achieving a marginal improvement in accuracy. However, the mean-field method struggles in the nonlinear settings, failing to produce consistent causal estimators due to the violation of the linearity assumption. In contrast, our estimator exhibits significantly lower mean squared errors (MSEs), maintaining robust performance even when the permutation invariance assumption is not fully met. 

\begin{figure}[h]   
\center{\includegraphics[width=0.75\linewidth]  {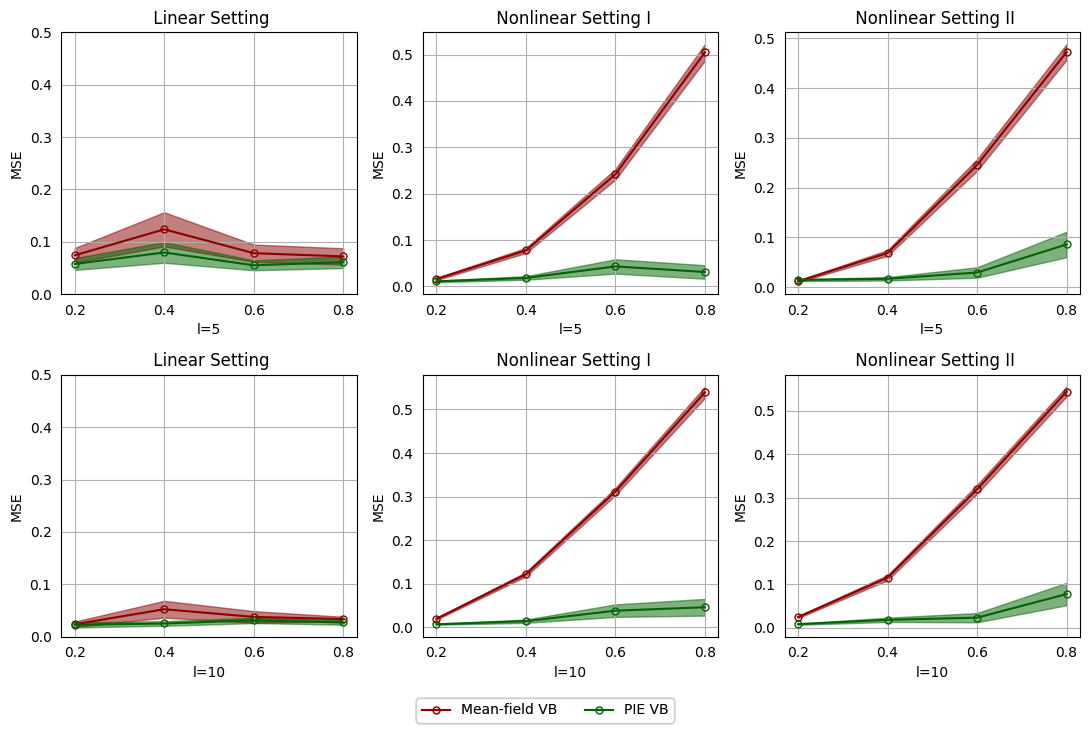}}   
\caption{Nondynamic simulation results: Mean Squared Errors (MSEs) of various policy value estimators are aggregated over 50 simulation replications. The top panels display results for $l=5$, while the bottom panels are for $l=10$. The left panels correspond to the linear setting, the middle panels to nonlinear setting I, and the right panels to nonlinear setting II.}
\label{one_results}  
\end{figure}

\subsection{Dynamic Simulation}
\label{simu:multi_stage}
In this subsection, we design a synthetic environment to simulate the dynamics of order dispatching within a ridesharing platform. We analyze a map grid of size \( l \times l \), with \( l \) taking values from the set \{5, 10, 15\} to represent different scales of the operational environment. This results in \( R = l^2 \) spatial units. The data generation process is outlined as follows:

\begin{enumerate}
    \item \textbf{Confounder}: For each spatial unit \( i \), at time \( t \), on day \( j \), we construct a five-dimensional observation \( X_{i,t,j} \), comprising:
    \begin{enumerate}
        \item \textbf{Order counts} (\( O_{i,t,j} \)): Generated by first drawing a mean \( \mu_i \) from a uniform distribution ranging between 40 and 180. Then sampling from a normal distribution with mean \(\mu_i\) and standard deviation \( 1 \). The order counts, indicative of the intrinsic demand within each region, are not altered by treatments. The actual order counts post-treatment are calculated using the formula \(O_{i,t,j}^* = O_{i,t,j} + 0.3 A_{i,t,j}O_{i,t,j}\). 
        
        \item \textbf{Connectivity factor} (\( C_{i} \)): The connectivity factor, reflecting a grid's road condition, is generated following a uniform distribution ranging from 0.1 to 1. A higher connectivity factor facilitates more rapid redistribution of drivers. 
        \item \textbf{Number of neighboring grids }(\(\Ncal(i)\)).
        \item \textbf{Driver counts} (\( D_{i,t,j} \)): Initially set 130 drivers per grid. Redistribution of these drivers to neighboring grids is determined by the actual orders counts, past drivers counts, and the connectivity factor of both the central grid and its neighboring grids. Specifically, this redistribution process adheres to the methodology outlined in Algorithm \ref{alg: driver}.
        \item \textbf{Mismatch rate} (\( M_{i,t,j} \)): Calculated as \( 0.9 [1 - |D_{i,t,j} - O^*_{i,t-1,j}| / (1 + D_{i,t,j} + O^*_{i,t-1,j})] + 0.1 M_{i,t-1,j} \), reflecting higher values when driver and order counts are closely aligned. 

    \end{enumerate}
    \item \textbf{Treatment} (\( A_{i,t,j} \)): Represents whether a discount is offered for orders in the \( i \)-th grid at time \( t \) on day \( j \), independently sampled from a Bernoulli distribution with a success probability of 0.5.
    \item \textbf{Response} (\( Y_{i,t,j} \)): Computed as \( M_{i,t+1,j}^2 \min(D_{i,t+1,j}, O^*_{i,t,j}) - 2|D_{i,t+1,j} - O^*_{i,t,j}|\).
\end{enumerate}

\SetKwComment{Comment}{/* }{ */}

\begin{algorithm}
\caption{Driver transition process }\label{alg: driver}
$\forall ~ i,t,j ~~ V_{i,t,j} \gets 0$\;
\For{$i\leq N$}{
    $\Delta_{i,t,j} = |D_{i,t-1,j} - O^*_{i,t-1,j}|$\Comment*[r]{Calculate the surpass of drivers}
    \For{$k\in \Ncal(i)$}{
         $C_{i-k} = \min\{C_{i}, C_{k}\} $\Comment*[r]{Calculate bi-connectivity}
        $V_{i,t,j} = V_{i,t,j} - C_{i-k}(\Delta_{i,t,j}-\Delta_{k,t,j})$ \Comment*[r]{Update transition vectors}
        $V_{k,t,j} = V_{i,t,j} + C_{i-k}(\Delta_{i,t,j}-\Delta_{k,t,j})$ \;
    }
} 
$D_{i,t,j} = D_{i,t-1,j} + V_{i,t,j}/\Ncal(i)\ $\Comment*[r]{Update the driver numbers}
\end{algorithm}



 Data span $S=200$ days, each with $T=40$ time points. We evaluate top-$Q$ policies that subsidize the $Q$ spatial units with highest average number of orders. These units receive treatment 1, while others receive treatment 0. The discount factor $\gamma$ is set to 0.9. We employ Monte Carlo simulations to estimate the oracle policy values, and implement the proposed value-based (VB) and doubly robust (DR) methods detailed in Sections \ref{subsection: vb} and \ref{subsection: dr}, comparing them against mean-field versions. Their MSEs, for various $l$ and $Q$ combinations, are shown in Figure \ref{multi_results}. It can be seen that our VB and DR estimators yield much smaller MSEs than their mean-field counterparts, highlighting the effectiveness of the proposed neural network architecture in capturing complex interference structures.

\begin{figure}[htb]   
\center{\includegraphics[width=0.75\linewidth]  {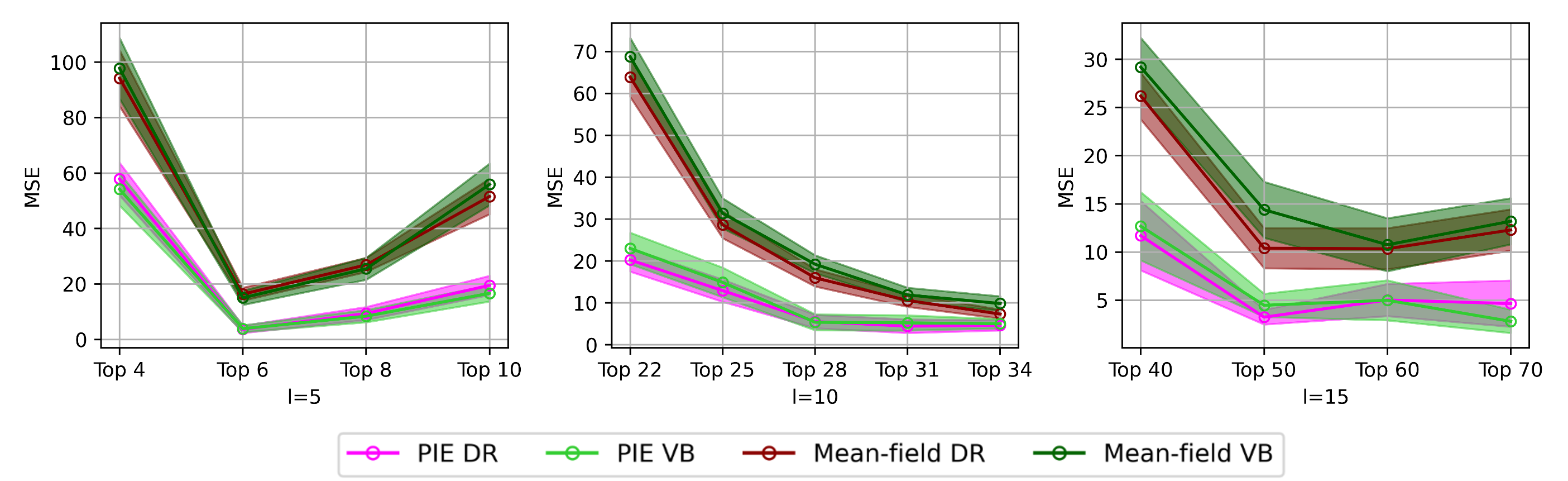}}   
\caption{Dynamic simulation results: MSEs of VB and DR estimators corresponding to different combinations of $l$ and $Q$.}
\label{multi_results}  
\end{figure}

\subsection{Real Data-based Simulation}
\label{simu:real_data}
In this subsection, we leverage a historical dataset from a globally renowned ridesharing company to develop a simulator, aimed at evaluating our proposed method. This dataset encompasses a five-day period in a particular city, capturing drivers' trajectories, order requests, idle drivers' movements, and their status of being online or offline. We follow \citet{tang2019deep} to simulate key dynamics of real-world ridesharing markets. The process involves:

\begin{enumerate}
    \item \textbf{Spatial Division}: The city is divided into \( R = 85 \) hexagonal regions, providing a detailed spatial framework for the simulation.
    \item \textbf{Time Partitioning}: Each day, spanning from 4 am to 12 pm, is segmented into numerous slices, with each slice lasting 2 seconds. 
\end{enumerate}

At each simulation's start, drivers are distributed across the city based on their empirical distribution from the offline dataset. The simulator, during each time slice, updates the states of drivers and orders through the following steps:

\begin{enumerate}
    \item Drivers assigned to orders decide whether to accept these based on probabilities calculated by a pre-trained LightGBM model, taking into account driver and order characteristics.
    \item Idle drivers, not assigned to orders, are directed to specific locations according to the historical idle movement data.
    \item Drivers subjected to repositioning follow the platform's instructions, determined by a pre-trained repositioning algorithm.
    \item Drivers engaged with accepted orders move to pick-up locations, collect passengers, and head to the destinations.
    \item The driver pool is dynamically updated, reflecting new drivers entering the city and existing drivers going offline.
\end{enumerate}
The order update process in the simulator comprises two main components:
\begin{enumerate}
    \item[(I)] Generation of new orders based on historical data, influenced by the passenger-side subsidy policy under evaluation. This policy offers discounts to passengers placing orders within certain spatial units at specific times, thereby increasing order volume in those areas.
    \item[(II)] Processing and dispatching of both existing unassigned and new orders, following a dispatch algorithm detailed in \citep{tang2019deep}.
\end{enumerate}

For evaluation, we set the immediate outcome as the gross merchandise value (GMV), sample treatments from a Bernoulli distribution (with success probability \( p = 0.5 \)), and run the simulator \( S \) times to generate an offline dataset where \( S \) varies among \{4, 7, 14\}. $T$ is fixed to $120$. We use the total numbers of orders and drivers in each region at each time as the time-varying confounding vector.

Analyzing this dataset poses multiple challenges: 
\begin{enumerate}
    \item[(I)] The underlying dynamics demonstrate significant variability throughout the day and across distinct spatial units, thus exhibiting both temporal nonstationarity and spatial heterogeneity;
    \item[(II)] The Markov test developed by \citep{shi2020does} reveals that the data does not conform to the Markov property. 
\end{enumerate}
To address the first challenge, we include the location of each region and the timestamp in the set of confounding variables. Previous studies, such as \citep{xu2018large}, have shown that a region's proximity to the city center and whether the timestamp corresponds to rush hour are crucial factors in estimating the value function. To account for spatial heterogeneity, we employ one-hot encoding based on the region's index to encapsulate spatial information. 
Regarding temporal nonstationarity, we consider two encoding schemes: one utilizing one-hot encoding based on the current hour (denoted by ``hour"), and the other dividing a day into five periods separated by the morning peak (7-9am) and afternoon peak (5-7pm) times (denoted by ``rush"). 
To overcome the second challenge related to the non-Markov nature of the data, we reconstruct the time-varying confounder by concatenating the past $L$ measurements to satisfy the Markov assumption. We consider two choices for $L$, corresponding to lengths of $4$ and $8$.

Similar to Section \ref{simu:multi_stage}, 
we set the target policy to the top-$Q$ policy that assigns treatments based on the order demand and driver availability. In our experiment, the degree of mismatch is quantified by the difference in the number of orders and drivers in each region. We then apply the subsidying policies to the top $Q$ regions exhibiting the highest degree of mismatch. 
The parameter $Q$ determines the size of the treated regions, and we consider various values: $Q=10,15,20,25,$ and $30$. The discounted factor is fixed to 0.98. 
The MSEs of the proposed value-based, doubly robust estimators and their mean-field counterparts are reported in Figure \ref{didi_results}, with different combinations of encoding schemes, Markov orders $L$, days of experiment $S$, and the number of treated regions $Q$ .

We draw several conclusions: 
\begin{enumerate}
    \item[(i)] The performance of the mean-field method is highly dependent on different encoding schemes and the number of trajectories, exhibiting substantial variation. In certain circumstances, it performs extremely poorly. For instance, when $(S, L, Q)=(7, 8, 15)$ under one-hot encoding, MSEs larger than 1000 are observed;
    \item[(ii)] The proposed methods consistently perform well across the majority of scenarios. They especially excel in conditions where the mean-field method struggles to achieve effective results;
    \item[(iii)] When the mean-field function is mis-specified, both the density ratio and the value function are impacted. This could partially explain why the mean-field DR method underperforms in comparison to the mean-field VB method in certain situations; see e.g., the subplots in row 1 column 3 and row 2 column 4;
    \item[(iv)] Our proposed method shows strong performance when used in conjunction with either DR or VM, showcasing the precision of the proposed interference effect function estimator.
\end{enumerate}

\begin{figure}[t]   
    \center{\includegraphics[width=\linewidth]  {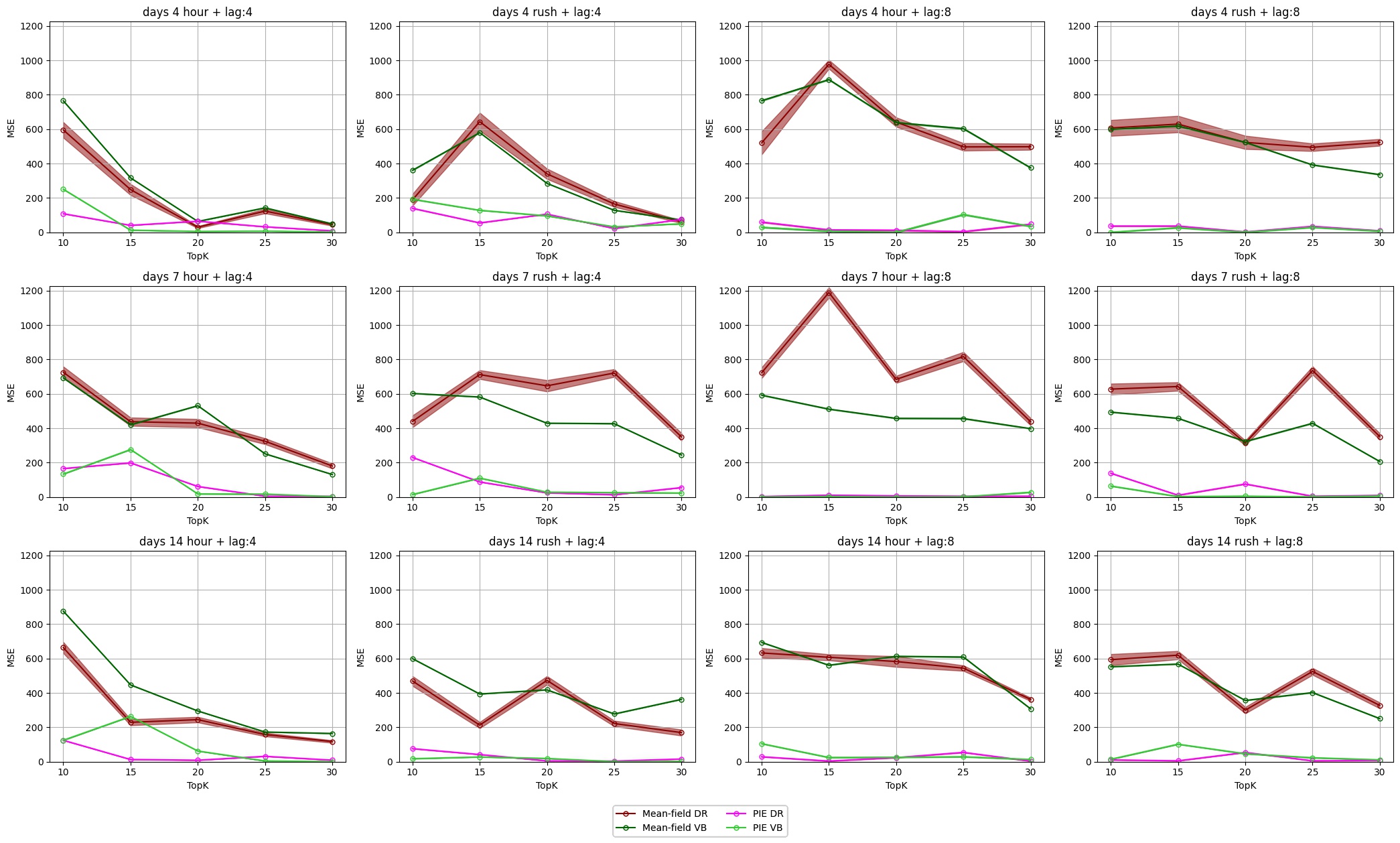}}   
    \caption{MSE Comparison in Real Data-Based Simulations. This figure presents the performance of proposed estimators and their mean-field counterparts in the simulated ridesharing environment. Each row's subplot corresponds to a different number of trajectories, denoted as $S_0$. Columns distinguish between various encoding schemes and Markov orders. The terms 'days', 'lag', 'hour', and 'rush' indicate the number of trajectories $S_0$, the Markov order $L$, and the temporal encoding schemes, respectively. All spatial encoding is performed using the 'one-hot' scheme. In the legend, 'Mean-field VB' refers to the Value-based method combined with the mean-field method, and 'PIE VB' denotes the Value-based method implemented with the PIE method. 'DR' stands for the Doubly Robust method. Each experiment was replicated 21 times, and the confidence bands represent the variance among these replications.
    }
    \label{didi_results}  
    \end{figure}


\section{Discussion}

We have introduced a novel causal deepsets model that effectively addresses the challenge of spatial interference. 
The versatility of our framework opens up a plethora of applications in diverse fields. In urban planning, it could be pivotal in modeling the impact of policy interventions on traffic flow and public transportation. Environmental studies could leverage this model to assess the effects of policy changes on pollution patterns across different regions. In public health, our methodology could be instrumental in analyzing the spread of diseases and the efficacy of intervention strategies over varying spatial scales.

Meanwhile, our theoretical analysis, which establishes the consistency, convergence rate, and minimax optimality of the PIE estimator, is a cornerstone of our research, 
not only validate the robustness of our methodology but also provide a solid foundation for future advancements in the field. They also encourage the exploration of PIE in other contexts where permutation invariance is a key characteristic, potentially broadening the scope of its applicability.
Several important avenues for extension and further research emerge from our study: 

\underline{(i) Spatial-Temporal Interference and Invariance Property}: Our paper represents an early yet significant effort in addressing spatial-temporal interference. The methodologies we have developed for handling either spatial or temporal interference offer avenues for extension and refinement. 

In terms of temporal dynamics, while our model employs MDPs, it is possible that they do not completely capture the intricacies of interference structures. Therefore, more complex models, such as higher-order MDPs or Partially Observable Markov Decision Processes (POMDPs), might be necessary for a more comprehensive understanding. Additionally, the latest advancements in transformer-based models, celebrated for their proficiency in managing temporal correlations \citep{vaswani2017attention, zhou2021informer, li2023survey}, point towards exciting possibilities for future investigations in this area.

In the spatial dimension, our current model primarily considers interference effects among geographically proximate neighbors. This can be expanded by utilizing more complex metrics than Euclidean distance, such as transport costs and actual road conditions, to better reflect real-world scenarios \citep{zhou2021graph}. Moreover, incorporating semantic aspects like functional similarity into the concept of proximity could enhance our understanding of spatial dynamics \citep{geng2019spatiotemporal, wang2019origin}. Finally, transformer-based invariant neural networks, known for their effectiveness in various domains \citep{lee2019set, zhao2021point}, have yet to be explored in spatial modeling contexts. 

Finally, our method is predicated on the assumption of permutation invariance among neighboring regions within a network interference context, which leads to an essential empirical question regarding the validity of this assumption. One solution could involve developing a non-parametric test based on collected data. Alternatively, strategies like adaptively learning invariance properties from the data \citep{benton2020learning}, as well as balancing representational capacity and model invariance through regularization \citep{cohen2020regularizing}, could be explored. An empirical and theoretical examination of the aforementioned extensions would provide invaluable insights, enhancing our understanding and application of spatial-temporal interference models.



\underline{(ii) Unmeasured Confounding}: 
 Our framework operates under the assumption of `unconfoundedness', which presumes the absence of unmeasured confounders that could simultaneously influence both the treatment and the response (or future confounders in a dynamic setting). This is a standard assumption in scenarios where data is generated through automated, data-adaptive policies. However, its applicability is less certain in contexts involving human decision-making. For instance, in ridesharing platforms, human interventions during unexpected events like severe weather or large public gatherings can significantly alter passenger behavior, leading to an inherently confounded dataset due to the omission of these contextual factors \citep{shi2022off}.

Addressing unmeasured confounding in spatial causal inference and OPE algorithms has garnered significant interest recently \citep{jarner2002estimation,thaden2018structural,papadogeorgou2019adjusting,kallus2020confounding,keller2020selecting,tennenholtz2020off,giffin2021instrumental,fu2022offline,shi2022minimax,bennett2023proximal,bruns2023robust,xu2023instrumental}. However, a critical gap remains in these approaches: they often focus solely on either spatial or temporal dependencies, without considering scenarios where both types of confounders coexist. Bridging this gap would significantly enhance the robustness and applicability of spatial-temporal causal inference methods, particularly in complex real-world settings that involve decision making over time and space. 

\underline{(iii) Design of Spatial and/or Temporal Experiments}: 
Our research contributes to the burgeoning field of policy evaluation 
given a pre-collected dataset. 
A compelling area for future exploration is the design of spatial and/or temporal experiments to generate such data, 
in order to optimize the estimation accuracy of 
the resulting treatment effect estimator. 
While there is a growing body of literature on optimal experimental designs \citep[see e.g.,][]{ugander2013graph, li2019randomization, kong2021approximate, wager2021experimenting, hu2022switchback, leung2022rate, bojinov2023design, li2023optimal, Xiong2023}, these studies typically focus on experiments that either primarily feature spatial or temporal dependencies. Rarely do they explore designs that simultaneously account for both dependencies. 
Future research in this domain could potentially revolutionize our approach to understanding and predicting the impact of complex, dynamic interventions in a wide range of real-world scenarios.

\section*{Acknowledgments}
This research received partial support from the National Science Foundation of China (Grant Numbers: \#1636933 and \#1920920) and a grant from the Engineering and Physical Sciences Research Council in the United Kingdom (Grant Number: EP/W014971/1).

\newpage

\begin{appendices}
    \vspace*{2em}
    \section*{\Large Appendix}
    \vspace*{2em}
    \section{Definition and Notations}
    \label{sec:def}

    Consider a random variable governed by a probability measure $Q$ that is compactly supported on $\mathbb{R}^d$. The space $L_2(Q)$ represents the set of real-valued functions on $\mathbb{R}^d$, equipped with an inner product defined as $\langle f, g \rangle_{L_2(Q)} = \int f(x)g(x) dQ(x)$ and a corresponding norm $\|f\|_{L_2(Q)} = \left( \int f(x)^2  dQ(x) \right)^{1/2}$. 
    When considering the Lebesgue measure $\lambda$, the space $L_2$ denotes the set of real-valued functions on $\mathbb{R}^d$, with the norm $\|f\|_{L_2} = \left( \int f(x)^2  dx \right)^{1/2}$.
    The notation $\|\cdot\|_n$ represents the empirical norm, defined for a function $f$ as $\|f\|_n^2 = \frac{1}{n} \sum_{i=1}^n f(x_i)^2$. The empirical norm of a random variable $X$ is defined analogously as $\|X\|_n = \left( \frac{1}{n} \sum_{i=1}^n x_i^2 \right)^{1/2}$. 
    Finally, the infinity norm for a real-valued vector $Y=[Y_1, \ldots, Y_d]^T \in \mathbb{R}^d$ is defined as $\|Y\|_{\infty} = \max_k |Y_k|$. The infinity norm for a real-valued function $f$ is given by $\|f\|_{\infty} = \max_{X \in \mathbb{R}^d} |f(X)|$ and the infinity norm for a vector function $f(X)=[f_1(X), \ldots, f_K(X)]$ is $\max_{k} \|f_k(X)\|_{\infty}$.
    
    To lay the groundwork, we expand upon the definition of the permutation operator (Definition \ref{def: PO}). Specifically, consider a matrix $X = [X_1, X_2, \ldots, X_N]\in \mathbb{R}^{M\times N}$. Under the permutation operator $\varPi$, we define its permutation as another $M\times N$ matrix $\varPi_X = [X_{\varPi(1)}, X_{\varPi(2)}, \ldots, X_{\varPi(N)}]$.

    \section{Proof of Theorem \ref{theory: consistency}}
    \label{sec: consistency}

    To establish the consistency of \(\hat{f}\), it suffices to show that \(\|\hat{f}-f^*\|^2_{L_2} = o_p(1)\). First, let's formalize the definition of the PIE function class. The PIE, as defined in (\ref{eq: naive-PIE}), comprises two components: \(\psi\) and \(\varphi\). Denote \(L_{\psi}\) as the number of hidden layers between the input and output of \(\varphi\). The total number of non-zero parameters in \(\psi\) is capped by \(S_{\psi}\). Similarly, define \(L_{\varphi}\) and \(S_{\varphi}\) as the number of hidden layers and the upper bound for the number of non-zero parameters in \(\varphi\), respectively. Additionally, we assume that the magnitude of all parameters in both \(\psi\) and \(\varphi\) is bounded by \(B\), and all the PIE functions under consideration are bounded such that \(|f| \leq J\). Therefore, we can represent the PIE function class using \(\mathcal{F} = \mathcal{F}(S_{\psi}, S_{\varphi}, L_{\psi}, L_{\varphi}, B)\).

    To  proceed with our proof, let us first introduce the concept of the 'uniform best predictor', denoted by \( f_n \). This predictor is determined by the following criterion
    \begin{equation*}
        f_{n} = \arg \min_{f\in \mathcal{F}} \|f-f^*\|_{\infty} ,\quad \epsilon = \| f_{n}-f^*\|_{\infty}.
    \end{equation*}
    Subsequently, we apply a standard error decomposition, akin to that described in \citep{Tengyuan}
    \begin{equation*}
        \begin{aligned}
            c_1\|\hat{f}-f^*\|^2_{L_2} &\leq \mathbb{E}[l(\hat{f})] - \mathbb{E}[l(f^*)]  \\
            &\leq \mathbb{E}[l(\hat{f})] - \mathbb{E}[l(f^*)] + \mathbb{E}_n[l(f_n,Z) - l(\hat{f},Z)] \\
            &= \mathbb{E}[l(\hat{f})] - \mathbb{E}[l(f^*)] + \mathbb{E}_n[l(f_n,Z) - l(\hat{f},Z)] + \mathbb{E}_n[l(f^*,Z) - l(f^*,Z)] \\
            &= (\mathbb{E}-\mathbb{E}_{n})[l(\hat{f}) - l(f^*)] + \mathbb{E}_n[l(f_n)-l(f^*)].
        \end{aligned}
    \end{equation*}
    In this formulation, the second inequality stems from the definition of \( f_n \). These inequalities effectively disaggregate the error into two principal components: the empirical process term (the first term) and the approximation error (the second term). We intend to establish bounds for both components.

    To bound the approximation error, we utilize the principle of Lipschitz continuity on the loss function. Considering that the mean squared error (MSE) loss \( l(f) \) exhibits Lipschitz continuity for a bounded \( f \), and in light of Assumption \ref{ass: bounded} which confines both \( f^* \) and functions within \( \Fcal \) to a maximum of \( J \), it follows that
    \begin{equation*}
        \forall f \in \Fcal, \quad \|l(f)-l(f^*)\|_{\infty} \leq C_l \|f-f^*\|_{\infty},
    \end{equation*}
    where \( C_l \) denotes the Lipschitz constant. Given that \( f_n \) belongs to \( \Fcal \), we can infer
    \begin{equation}
        \mathbb{E}_n\left[l(f_n)-l(f^*)\right] \leq \mathbb{E}_n\left[C_l | f_n - f^* |\right] \leq C_l \epsilon.
        \label{eq: app}
    \end{equation}
    
     To bound the estimation error, we need the following lemmas
    
    \begin{lemma}[Symmetrization, Lemma 5 in \citep{Tengyuan}]
      \label{lem:Symmetrization}
        Given function class \(\mathcal{G}\), for any \( g \in \mathcal{G} \) that \( |g| \leq G \) and \( \mathbb{V}[g] \leq V \), then with probability at least \( 1 - 2e^{-\gamma} \)
        \[
            \sup_{g \in \mathcal{G}} \left\{ \mathbb{E} g - \mathbb{E}_n g \right\} \leq 6 \mathbb{E}_\eta \mathcal{R}_n \mathcal{G} + \sqrt{\frac{2V \gamma}{n}} + \frac{23 G \gamma}{3n} \enspace,
        \]
        where \( \mathcal{R}_n \mathcal{G} = \mathbb{E}_{\xi}\left[\sup_{g\in \mathcal{G}}\left|\frac{1}{n}\sum_{i=1}^n\xi_i g(x_i)\right|\right] \) denote the empirical Rademacher complexity of function class \(\mathcal{G}\).
    \end{lemma}
    
    \begin{lemma}[Chaining Theorem 2.3.7 in \citep{foundation}] 
        \label{lem:F-Chaining}
        Let \((\mathcal{S}, d)\) be a metric space, where \(d\) is a pseudo-distance. Let \(X(t)\) be a sub-Gaussian process relative to \(d\). Assume that \(\int_0^{\infty}\sqrt{\log \mathcal{N}(\epsilon, \mathcal{S}, d)}d\epsilon < \infty\), then
        \begin{equation*}
            \mathbb{E}\sup_{t\in T}\|X(t)\| \leq \mathbb{E}\|X(t_0)\|+4\sqrt{2}\int_0^{\delta/2}\sqrt{\log 2\mathcal{N}(\epsilon, \mathcal{S}, d)}d\epsilon, 
        \end{equation*}
        where \( t_0 \in T \), \(\delta\) is the diameter of \((\mathcal{S}, d)\).
    \end{lemma}
    
     By setting \(X(t) = \frac{1}{\sqrt{n}}\sum_{i=1}^n\xi_i f(x_i)\) and \(X(t_0) = 0\), the empirical Rademacher complexity \(\Rcal_n\Fcal\) can be upper bounded by \(\frac{4\sqrt{2}}{\sqrt{n}}\int_0^{\delta/2}\sqrt{\log\mathcal{N}(\epsilon,\Fcal, \|\cdot\|_{k})}\), where \(\delta/2\) is the diameter of function class \(\mathcal{F}\) with \(k\)-norm.
    
    From the Lipschitz continuity of the MSE loss \(l\), it follows that 
    \begin{equation*}
        \forall f \in \mathcal{F}, \qquad \|l(f)-l(f^*)\|_{\infty} \leq C_l\|f-f^*\|_{\infty} \leq 2JC_l. 
    \end{equation*}
    Then we have
    \begin{align*}
        \mathbb{V}\left[l(f)-l(f^*)\right] \leq \mathbb{E}\left[(l(f)-l(f^*))^2\right] \leq 4J^2C_l^2.
    \end{align*}
    
    We apply Symmetrization in Lemma \ref{lem:Symmetrization} with \(\Gcal=\{g=l(f)-l(f^*): f\in \mathcal{F}\}\), \(G=2JC_l\) and \(V = 4J^2C_l^2\), so with probability at least \(1-2e^{-\gamma}\)
    \begin{equation}
        (\mathbb{E}-\mathbb{E}_n)\left[l(\hat{f}) - l(f^*) \right] \leq 6\mathbb{E}_{\eta}\mathcal{R}_n \mathcal{G} + \sqrt{\frac{8J^2C_l^2\gamma}{n}}+\frac{46JC_l\gamma}{3n} .
        \label{eq: est}
    \end{equation}

    Due to Lemma \ref{lem:F-Chaining}, we have 
    \begin{align*}
        \mathbb{E}_{\eta}\mathcal{R}_n \mathcal{G} &= \mathbb{E}_{\eta}\left[ \sup_{f\in \mathcal{G}}\left| \frac{1}{n} \sum_{i=1}^n \eta_i f(X^{(i)})\right|\right] \\
        & \leq \frac{4\sqrt{2}}{\sqrt{n}}\int^{JC_l}_0 \sqrt{\log(2\mathcal{N}(\epsilon, \mathcal{G}, \|\cdot\|_{\infty}))}d\epsilon,
    \end{align*}
    where the diameter of \(\mathcal{G}\) comes from
    \begin{equation*}
        \forall g, g' \in \mathcal{G}, \quad \|g-g'\|_{\infty} = \|l(f)-l(f'); f, f' \in \mathcal{F} \|_{\infty}  \leq 2JC_l.
    \end{equation*}

    Putting \eqref{eq: app} and \eqref{eq: est} together, by setting \(\gamma= \log n\),  we get
    \begin{align*}
        &c_1\|\hat{f}-f^*\|^2_{L_2} \\ 
        &\leq c_2\epsilon +\frac{1}{\sqrt{n}}\int^{JC_l}_0 \sqrt{\log(\mathcal{N}(\eta, \mathcal{F}, \|\cdot\|_{\infty}))}d\eta  +\sqrt{\frac{8J^2C_l^2 \log n}{n}}+\frac{67JC_l\log n}{3n}.
    \end{align*}
    
    Following Theorem 3.2 of \citep{pine}, there exists such PIE function class \(\mathcal{F}_0\). By setting \(\Fcal = \Fcal_0\), \(\epsilon = \min_{f\in \mathcal{F}_0}\|f-f^*\|_{\infty}\) can be arbitrarily small. So that \(\|\hat{f}-f^*\|_{L_2}^2 \rightarrow_{n} 0\), with probability \(1-\frac{2}{n}\) converging to one.

    \section{Proof of Theorem \ref{theory: convergence}}
    \label{appendix:convergence-rate}
    \subsection{Error decomposition}

    In this section, we  explore the convergence rate of PIE networks, employing an error decomposition approach akin to the one described in \citep{imaizumi2019deep}. Given that \(\hat{f}\) is the Empirical Risk Minimization (ERM) predictor, it holds for all \(f \in \mathcal{F}\) that
    \begin{equation*}
        \|Y-\hat{f}\|_n^2 \leq \|Y-f\|_n^2.
    \end{equation*}
    This is followed by the representation \(Y = f^*(X)+\xi\), leading to
    \begin{equation*}
        \|f^*+\xi-\hat{f}\|_n^2 \leq \|f^*+\xi-f\|_n^2.
    \end{equation*}
    A straightforward calculation results in
    \begin{align}
        \|f^*-\hat{f}\|_n^2 &\leq \|f^*-f\|_n^2 +\frac{2}{n}\sum_{i=1}^n\xi_i(\hat{f}(x_i)-f(x_i)) \label{eq: basic_inequality} \\
        &\leq \|f^*-f\|_{\infty}^2 +\frac{2}{n}\sum_{i=1}^n\xi_i(\hat{f}(x_i)-f(x_i)) \nonumber.
    \end{align}
    In the above  expression, the first term represents the approximation error, which quantifies the capacity of $\Fcal$ to uniformly approximate \(f^*\). The second term is the estimation error, gauging the variance of \(\hat{f}\).

    \subsection{Approximation Error}
    \label{sec: app}
    We first define uniform best predictor
    \begin{equation*}
        f_{n} = \arg \min_{f\in \Fcal} \|f-f^*\|_{\infty}.  
    \end{equation*}
    
    In this section, we follow a two-step approach to upper bound $\|f^*-f_n\|_{\infty}^2$. We first construct a specially modified Taylor  expansion of the target function $f^*$, hereinafter referred to as $f_1$. Subsequently, this modified  expansion is uniformly approximated using functions from $\Fcal$.
    
    We first perform a Taylor  expansion in line with Theorem 1 from \citep{approx}, we revisit the construction details. Let $N'$ be a positive integer and $d=M\times N$. We partition unity by a grid of $(N'+1)^d$ functions $\phi_m$ on the domain $[0,1]^d$, satisfying
    \begin{equation*}
        \Sigma_{\m} \phi_{\m}(x) \equiv 1, ~~~ x \in [0,1]^{M\times N}.
    \end{equation*} 
    In this context, $\m$ represents a matrix given by $\m =(m_{ij})\in \{0,1,\ldots,N'\}^{M\times N}$, and the function $\phi_{\m}$ is defined as
    \begin{equation*}
        \phi_{\m}(x) = \prod_{i=1}^M\prod_{j=1}^N\omega (3N'(x_{i,j}-\frac{m_{i,j}}{N'})),
    \end{equation*}
    where
    $$\omega(x)=\begin{cases}
        1, & |x|< 1,\\
        0, & 2 < |x|, \\
        2-|x|, & 1\le |x|\le 2 .
    \end{cases}$$
    For any $\m \in \{0,\ldots, N'\}^d$, we consider the Taylor polynomial of degree $-(\beta-1)$ for the function $f^*$, centered at the point $x=\frac{\m}{N'}$. This polynomial is expressed as
    \begin{equation*}
        P_{\m}(x) = \sum_{\p: |\p|\leq\beta}\frac{D^{\p} f^*}{\p!}\bigg|_{x = \frac{\m}{N'}}(x-\frac{\m}{N'})^{\p},
    \end{equation*} 
    where $|\p| = \sum_{i=1}^M\sum_{j=1}^N p_{i,j}$, $\p! = \prod_{i=1}^M\prod_{j=1}^N p_{i,j}!$ and $D^{\p}$ is the respective weak derivative. $(x-\frac{\m}{N'})^{\p} = \prod_{i=1}^M\prod_{j=1}^N(x_{i,j} - \frac{m_{i,j}}{N'})^{p_{i,j}}$.
    
    We now establish an approximation of $f^*$, defined as
    \begin{equation*}
        f^*_1(X) = \sum_{\m \in [0,\ldots,N']^{d}} \phi_{\m} P_{\m}(X).
    \end{equation*}
    Referring to insights from \citep{approx}, the approximation error can be upper bounded as
    \begin{equation*}
        \|f^*-f^*_1\|_{\infty} \leq \frac{2^d d^{\beta} J}{\beta!}\left(\frac{1}{N'}\right)^{\beta}.
    \end{equation*}
    Then we construct a permutation invariant variant approximation of $f^*$, denoted as $f_1$, based on $f^*_1$
    \begin{equation*}
        f_1(X) = \frac{1}{|\Scal_N|}\sum_{\varPi \in \Scal_N}f_1^*(\varPi_X).
    \end{equation*}
    Considering that $f^*$ is inherently permutation invariant, we proceed to assess the approximation error associated with $f_1$
    \begin{align}
        |f^*(X)-f_1(X)| &= \frac{1}{|\Scal_N|}\big| \sum_{\varPi \in \Scal_N} \big(f_1^*(\varPi_X)- f^*(\varPi_X)\big)\big| \nonumber \\
        &\leq \frac{1}{|\Scal_N|}\sum_{\varPi \in \Scal_N} \big|  f_1^*(\varPi_X)- f^*(\varPi_X)\big| \nonumber \\
        &\leq \frac{2^d d^{\beta} J}{\beta!}\left(\frac{1}{N'}\right)^{\beta}, 
        \label{eq: Taylor_error}
    \end{align}
    which leads to
    \begin{equation*}
        \|f^*-f_1\|_{\infty} \leq \frac{2^d d^{\beta} J}{\beta!}\left(\frac{1}{N'}\right)^{\beta}.
    \end{equation*}
    
    Then, we construct the PIE to approximate $f_1$. Let $\omega_{i,j}^k = \omega(3N'(X_{i,j}-\frac{k}{N'}))$ and construct $Y(X,N') \in \R^{K \times N}$ given $X$ and $N'$ as follows
    \begin{align}
        \label{eq: Y}
        Y(X,N')_{\cdot , i} = (\omega_{1,i}^{0},\ldots \omega_{1,i}^{N'},x_{1,i},\ldots ,\omega_{M,i}^{0},\ldots \omega_{M,i}^{N'}
        , X_{M,i}),
    \end{align}
    where $K = (N'+1)\times M$. Note $Y(X,N')_{\cdot, i}$ contains all the terms in $f_1(X)$ that involve $X_i$. For ease of notation, $Y$ and $Y_i$ will be used to denote $Y(X,N')$ and $Y(X,N')_{\cdot,i}$, respectively.
    
    It is noteworthy that $Y$ represents a piece-wise linear transformation of $X$, which can be precisely modeled using a ReLU network, as delineated in Proposition 1 of \citep{approx}. Consequently, the conversion from $X$ to $Y$ does not incur any approximation error. Furthermore, $Y$ is a function contingent on both $X$ and $N'$. By selecting an appropriate $N'$ to minimize the approximation error, the function class and, consequently, $Y$ become fixed entities. Hence, $Y(X,N')$ also does not add to the estimation error. Therefore, in our subsequent analysis, we will focus solely on $Y$, having transformed $X$ into this new representation.
    
    \begin{align}
        \label{eq: g*}
        f_1^*(X) &= \sum_{\m} \prod_{i=1}^M\prod_{j=1}^N\omega (3N'(X_{i,j}-\frac{m_{i,j}}{N'})) \sum_{\p: |\p|\leq \beta}\frac{D^{\p} f}{\p!}\bigg|_{X = \frac{\m}{N'}}(X-\frac{m}{N'})^{\p} \nonumber\\
        & = \sum_{\m} \prod_{i=1}^M\prod_{j=1}^N\omega_{i,j}^{m_{i,j}} \sum_{\p: |\p|\leq \beta}\frac{D^{\p} f}{\p!}\bigg|_{X = \frac{m}{N'}}(X-\frac{\m}{N'})^{\p} \nonumber \\
        & = \sum_{\m} \prod_{i=1}^M\prod_{j=1}^N\omega_{i,j}^{m_{i,j}} \sum_{\p: |\p|\leq \beta}\frac{D^{\p} f}{\p!}\bigg|_{X = \frac{\m}{N'}}\prod_{s=1}^M\prod_{t=1}^N(X_{s,t} - \frac{m_{s,t}}{N'})^{p_{s,t}}\nonumber \\
        & = g^*(Y(X,N')) ,
    \end{align}
    so that we have 
    \begin{equation*}
        f_1(X) = \frac{1}{|\Scal_N|}\sum_{\varPi \in \Scal_N}f_1^*(\varPi_X) = \frac{1}{|\Scal_N|}\sum_{\varPi \in \Scal_N}g^*(Y(\varPi_X,N')).
    \end{equation*}
    
    Consider the function $g(Y)$, defined as 
    \begin{equation*}
        g(Y) = \frac{1}{|\Scal_N|}\sum_{\varPi \in \Scal_N}g^*(Y(\varPi_X, N')),
    \end{equation*}
    which ensures that $g(Y)$ possesses permutation invariance. As indicated by \eqref{eq: g*}, $g(Y)$ can be expressed as a polynomial in terms of $Y$, characterized by a degree of $MN + \beta$. To further dissect the structure of $g(Y)$, additional lemmas will be introduced and utilized in the analysis.

    \begin{lemma}[Weyl's Polarization \citep{pine}]
        \label{lem: weyl}
        For any polynomial permutation invariant function $f:\R^{K\times N} \mapsto \R$, there exist a series of $1\times K$ vectors $\{ \af_t\}_{t=1}^T$ and a series of permutation invariant functions $f_t:\R^{1\times N} \mapsto \R$, such that $f$ can be represented by
        \begin{equation*}
            f(X) = \sum_{t=1}^Tf_t(\af_t^{\top}X). 
        \end{equation*}
    \end{lemma}
    \begin{lemma}[Hilbert finiteness Theorem \citep{pine}]
        \label{lem: Hilbert}
        There exists finitely many permutation invariant polynomial basis $f_1,\ldots,f_N: \R^N \rightarrow \R$ such that any permutation invariant polynomial $f:\R^N \rightarrow \R$ can be expressed as
        \begin{equation*}
            f(X) = \tilde{f}(f_1(X), \ldots, f_N(X)),
        \end{equation*}
        with some polynomial $\tilde{f}$ of $N$ variables. Especially, the following power sum basis is one possible permutation invariant polynomial basis
        \begin{equation*}
            f_j(X) = \sum_{n=1}^NX^j_n ,\quad j=1,\ldots, N, 
        \end{equation*}
        where $X_n$ is the $n-$th entry of $X$.
    \end{lemma}
    
    Given that $g(Y)$ is a permutation invariant polynomial function, we can apply Lemma \ref{lem: weyl}
    \begin{equation*}
        g(Y) = \sum_{t=1}^T s_t(a_t^{\top}Y) = \sum_{t=1}^T s_t(a_t^{\top}Y_1, \ldots, a_t^{\top}Y_N), 
    \end{equation*}
    where $s_t:\R^{N} \mapsto \R$ is a permutation invariant polynomial and $a_t=[a_{t,1}, \ldots, a_{t,K}] \in \R^K$. We postulate that $T = O(N'^{MN})$, assuming the components of $a_t$ are uniformly bounded by a constant $a$, with $|a_{i,j}| < a$. While Lemma \ref{lem: weyl} is broad and doesn't impose an upper limit on $T$, our specific application of $g$ as a polynomial with unique structure allows us to estimate the order of $T$ using linear equations, making our claim about $T$'s order plausible. Subsequently, Lemma \ref{lem: Hilbert} is employed to express $s_t(a_t^{\top}Y)$ in terms of $h_t(g_t(Y))$. Here, $g_t(Y) = [\sum_{n=1}^N(a_t^{\top} Y_n)^1,\ldots, \sum_{n=1}^N(a_t^{\top} Y_n)^N]$ forms the power sum basis, and $h_t : \R^N \to \R$ is a polynomial
    \begin{equation*}
        h_t(X) = \sum_{i = 0}^{MN+\beta} \sum_{\s \in \{ 1,..,N \}^i} c_{t,\s} \prod_{j \in \s}X_j, \quad  X \in \R^N. 
    \end{equation*}
    Recall that $g(Y)$ is a polynomial of degree $MN+\beta$, where each term in $g_t(Y)$ is of at least first order. Consequently, we assert that $h_t$ is a polynomial whose degree does not exceed $MN + \beta$, and the coefficients, denoted by $c_{t,\s}$, are constrained by a bound $C$. Integrating these elements yields the following expression for $g(Y)$
    \begin{equation*}
        g(Y) = \sum_{t=1}^T h_t(g_t(Y)) = \sum_{t=1}^T h_t(\sum_{n=1}^N(a_t^{\top}Y_n), \ldots, \sum_{n=1}^N(a_t^{\top}Y_n)^N).  
    \end{equation*}
    
    To approximate $f_1$ above, we construct a PIE $\tilde{f} \in \Fcal(S_{\psi}, S_{\varphi}, L_{\psi}, L_{\varphi}, B)$  and is characterized by the following structure
    \begin{gather*}
        \tilde{f}(X) = \psi(\varphi(Y)) = \tilde{g}(Y),  \\
        \tilde{g}(Y) = 
        \sum_{t=1}^T\psi_t(\varphi_t(Y)) 
        = \sum_{t=1}^T\psi_t(\sum_{n=1}^N \varphi_{t,1,n}(Y_n),\ldots, \sum_{n=1}^N \varphi_{t,N, n}(Y_n)),
    \end{gather*}
    where $\psi_t : \R^N \mapsto \R$ and $\varphi_{t,1,n} :\R^K \mapsto \R $ are fully connected ReLU neural networks. The detailed structure and attributes of these networks  will be provided later.
    
    In light of the definition of $f_n$, the following inequality holds
    \begin{equation*}
        \|f_1-f_n\|_{\infty} \leq \|f_1-\tilde{f}\|_{\infty} \leq |g(Y) - \tilde{g}(Y)|, ~~ \forall Y .
    \end{equation*}
    
    To establish a bound for $\|f_1-f_n\|_{\infty}$, it suffices to bound the term $|g(Y) - \tilde{g}(Y)|$. We approach this through the following decomposition
    \begin{align}
        |g(Y) - \tilde{g}(Y)| \nonumber &\leq |\sum_{t=1}^T h_t(\varphi_t(Y)) - \sum_{t=1}^T h_t(g_t(Y))| + |\sum_{t=1}^T\psi_t(\varphi_t(Y)) - \sum_{t=1}^T h_t(\varphi_t(Y))| \nonumber \\
        & \leq \sum_{t=1}^T \Big[ |h_t(\varphi_t(Y)) - h_t(g_t(Y))| + |\psi_t(\varphi_t(Y)) - h_t(\varphi_t(Y))| \Big].
        \label{decomp}
    \end{align}
    
    \begin{lemma}[ReLU Approximation \citep{approx}]
        \label{lem: relu}
        Given $M'>0$ and $\epsilon \in (0,1)$, there's a ReLU network $\eta$ with two input units that implements a function $\tilde{x}:\R^2 \rightarrow \R$ so that
        \begin{itemize}
            \item for any inputs $x,y$, if $|x|\leq M' $ and $|y|\leq M'$, then $|\tilde{x}(x,y)-xy|\leq \epsilon$;
            \item the depth and the number of computation units in $\eta$ are $O(ln(1/\epsilon)+ln(M'))$;
            \item the width of $\eta$ is a constant $C_w$ independent of $x,y$.
        \end{itemize}
    \end{lemma}
    
    \textbf{Bound of the first term in (\ref{decomp})}

    We commence by outlining the detailed structure of $\varphi_t$
    
    \begin{equation}
        \varphi_t(Y) = \left[\varphi_{t,1}, \ldots, \varphi_{t,N}\right] 
        = \left[\sum_{n=1}^N \varphi_{t,1,n}, \ldots, \sum_{n=1}^N \varphi_{t,N,n}\right],
        \label{eq: varphi_t}
    \end{equation}
    where
    \begin{equation}
        \varphi_{t,1,n}(Y_n) = w_t^{\top} Y_i,\quad \varphi_{t,k+1,n}(Y_n) = \tilde{x}_1(\varphi_{t,1,n}(Y_n), \varphi_{t,k,n}(Y_n)),
        \label{eq: varphi_construction}
    \end{equation}
    with $w_t \in [-B,B]^{K}$ representing learnable weights approximating $a_t^{\top}$, and $\tilde{x}_1$ being a ReLU network as per Lemma \ref{lem: relu}. Given that $\|Y\|_{\infty} \leq 1$, it follows that $|(a_t^{\top} Y_n)^j| \leq |a^{\top}_{t}Y_n|^j \leq (aK)^j$. Consequently, the deviation $|\varphi_t(Y) - g_t(Y)|$ can be bounded by iteratively applying Lemma \ref{lem: relu} with parameters $M' = (aK)^N$ and $\epsilon = \delta_1$, leading to 
    \begin{align*}
        &|\varphi_{t,1,n}(Y_n) - (a_t^{\top} Y_n)^1| = 0,  \\
        &|\varphi_{t,2,n}(Y_n) - (a_t^{\top} Y_n)^2|= |\tilde{x}_1(\varphi_{t,1,n}(Y_n),\varphi_{t,1,n}(Y_n))- (a_t^{\top}Y_n)^2| \leq \delta_1, \\
        &|\varphi_{t,k+1,n}(Y_n) - (a_t^{\top}Y_n)^{k+1}| \\
        &= |\tilde{x}_1(\varphi_{t,1,n}(Y_n), \varphi_{t,k,n}(Y_n))-\varphi_{t,1,n}(Y_n)\varphi_{t,k,n}(Y_n)+a_t^{\top}Y_n\varphi_{t,k,n}(Y_n)-(a_t^{\top}Y_n)^{k+1}| \\
        &\leq \delta_1 + aK|\varphi_{t,k,n}(Y_n) - (a_t^{\top}Y_n)^k|. 
    \end{align*}
    Through recursion, we obtain
    \begin{equation}
        \forall ~ t,n, \quad |\varphi_{t,k,n}(Y_n) - (a_t^{\top}Y_n)^k| \leq \delta_1 (\sum_{n=0}^{k-1}(aK)^n) \leq \delta_1(aK)^k.
        \label{eq: reduction}
    \end{equation}
    So we can bound the error between $\varphi_t(Y)$ and $g_t(Y)$ as follows
    \begin{align*}
        &\|\varphi_t(Y)-g_t(Y)\|_{\infty} \nonumber\\
        &= \left\|\left[\sum_{n=1}^N\varphi_{t,1,n}(Y_n),\ldots, \sum_{n=1}^N\varphi_{t,N,n}(Y_n)\right]-\left[\sum_{n=1}^N(a_t^{\top}Y_n), \ldots, \sum_{n=1}^N(a_t^{\top}Y_n)^N\right]\right\|_{\infty} \nonumber \\
        & =\|\sum_{n=1}^N\varphi_{t,N,n}(Y_n) - \sum_{n=1}^N(a_t^{\top}Y_n)^N\|_{\infty} \nonumber \\
        & \leq \sum_{n=1}^N\|\varphi_{t,N,n}(Y_n)-(a_t^{\top}Y_n)^N\|_{\infty} \nonumber \\
        &\leq \delta_1 N(aK)^N.
    \end{align*}
    
    \begin{lemma}[Lipschitz continulty of polynomial functions]
        \label{lem: Lipschitz}
    Suppose $f(x): [-k,k]^n \rightarrow \R$ is a polynomial with degree $\beta$ and coefficients bounded by $C$. $f$ is Lipschitz continues with Lipschitz constant $C\beta k^{\beta-1}$
    \begin{equation*}
        |f(x)-f(y)| = |f(x) -(f(x) - \nabla f(\xi)^{\top} (y-x))| = |\nabla f(\xi)^{\top}(x-y)| \leq C\beta k^{\beta-1}\|x-y\|_{\infty}.
    \end{equation*}
    \end{lemma}
    
    Recall that $h_t$ is a polynomial over N variables, having a maximum order of $MN + \beta$. Its coefficients are bounded by $C$ and the infinity norm of the input $\|g_t\|_{\infty} = \sum_{n=1}^N(a^{\top}_tY_n)^N$ can be upper bounded by $ N(aK)^N$. By applying Lemma \ref{lem: Lipschitz} with a parameter $k = N(aK)^N$ to $h_t$, we obtain
    \begin{align}
        &|h_t(\varphi_t(Y)) - h_t(g_t(Y))| \nonumber \\
        & \leq C(MN+\beta)(N(aK)^N)^{MN+\beta}\|\varphi_t(Y)-g_t(Y)\|_{\infty} \nonumber \\
        & \leq C(MN+\beta)(N(aK)^N)^{MN+\beta}N^2(aK)^N \delta_1.
        \label{eq: inner_diff}
    \end{align}
    
    \textbf{Bound of the second term in (\ref{decomp})}

    Then we provide the detailed structure of $\psi_t$
    \begin{equation}
        \psi_t(x) = \sum_{i = 0}^{MN+\beta} \sum_{\s \in \{ 1,\ldots ,N \}^i} \gamma_{t,\s} \tilde{\psi}_{\s}(x), \quad  x \in \R^N, 
        \label{eq: psi_t}
    \end{equation} 
    where $\gamma_{t,\s}\in [-B,B]$ are learnable weights and $\tilde{\psi}_{\s}(x)$ is the following ReLU network approximating $\prod_{j\in \s}x_j$ by 
    $$
    \left\{ 
    \begin{array}{l}
        \tilde{\psi}_{\s,1}(x) = x_{\s_1}, \\
        \tilde{\psi}_{\s,k}(x) = \tilde{x}_2(\tilde{\psi}_{\s,k-1}(x), x_{\s_k}), \quad k \leq |\s|,  \\
        \tilde{\psi}_{\s} = \tilde{\psi}^{\s}_{|\s|}, 
        \end{array} 
        \right. 
    $$
    where $\tilde{x}_2$ is the ReLU network from Lemma \ref{lem: relu}, with parameter settings of $\epsilon = \delta_2$ and $M' = (\|\varphi_{t}(X)\|_{\infty})^{MN+\beta} \leq N^{MN+\beta}(aK)^{MN^2+\beta N}$. Adhering to the procedure in (\ref{eq: reduction}), we derive
    \begin{equation*}
        |\tilde{\psi}_{\s}(\varphi_t(Y))-\prod_{j\in \s}\varphi_{t,j}(Y)| \leq \delta_2 (MN+\beta)\|\varphi_t(Y)\|_{\infty}^{MN+\beta} \leq \delta_2(MN+\beta)(N(aK)^N)^{MN+\beta}.
    \end{equation*}
    By setting $\gamma_{t,\s}= c_{t,\s}$, we establish a bound for the error as follows
    \begin{align}
        & |\psi_t(\varphi_t(Y)) - h_t(\varphi_t(Y))| \nonumber \\
        & = \sum_{i = 0}^{MN+\beta} \sum_{\s \in \{ 1,\ldots,N \}^i} c_{t,\s} |\tilde{\psi}_{\s}(\varphi_t(Y))-\prod_{j\in \s}\varphi_{t,j}(Y)| \nonumber \\
        & \leq N^{MN+\beta} C (MN+\beta)(N(aK)^N)^{MN+\beta}\delta_2 .
        \label{eq: outer_diff}
    \end{align}
    
    In summary, we have demonstrated the following
    \begin{align}
    \label{eq: in-out diff}
        \|f_1-f_n\|_{\infty} &\leq |g(Y) - \tilde{g}(Y)| \nonumber \\
        &\leq \sum_{t=1}^T \Big[ |h_t(\varphi_t(Y)) - h_t(g_t(Y))| + |\psi_t(\varphi_t(Y)) - h_t(\varphi_t(Y))| \Big] \nonumber \\
        &\leq TC(MN+\beta)\left[(N(aK)^N)^{MN+\beta+2} \delta_1 +(N^2(aK)^N)^{MN+\beta}\delta_2\right],
    \end{align}
    where the final inequality is deduced from \eqref{eq: inner_diff} and \eqref{eq: outer_diff}, along with certain calculations.

    To put things together, recall that
    \begin{equation}
        \|f_n-f^*\|_{\infty} \leq \|f_1-f^*\|_{\infty} + \|f_n-f_1\|_{\infty}.
        \label{eq: approx_decomp}
    \end{equation}
    To ensure $\|f_n-f^*\|_{\infty} \leq \epsilon$, we set $\|f_1-f^*\|_{\infty}$ and $\|f_1-f_n\|_{\infty}$ to be less than or equal to $\frac{\epsilon}{2}$. Based on the upper bound for $\|f_1-f^*\|_{\infty}$ as given in (\ref{eq: Taylor_error}), we establish
    \begin{equation*}
        \frac{2^d d^{\beta} J}{\beta!}\left(\frac{1}{N'}\right)^{\beta}= \frac{\epsilon}{2}  .
    \end{equation*}
    This setting confirms that $\|f_1-f^*\|_{\infty} \leq \frac{\epsilon}{2}$, leading to $N' = \left( \frac{{\beta}! \epsilon}{2^{d+1} d^{\beta} J} \right)^{-1/{\beta}} = O(\epsilon^{-1/\beta})$. Consequently, we have $T = O((N')^{MN}) = O(\epsilon^{-\frac{MN}{\beta}})$ and $K = (N'+1)M = O(M\epsilon^{-1/\beta})$.

    Finally, we set both terms in \eqref{eq: in-out diff} equal to $\frac{\epsilon}{4}$, leading to 
    \begin{align*}
        TC(MN+\beta)(N(aK)^N)^{MN+\beta+2}\delta_1 = \frac{\epsilon}{4}, \\
        TC(MN+\beta)(N^2(aK)^N)^{MN+\beta}\delta_2 = \frac{\epsilon}{4}, \\
    \end{align*}
    disregarding lower-order terms and constants, we have
    \begin{eqnarray}
        \delta_1 = O(\frac{\epsilon}{T N^{MN+\beta}(aK)^{MN^2+\beta N}}) = O(\epsilon^{\frac{MN^2+N\beta}{\beta}}),
        \label{delta1} \\
        \label{delta2}
        \delta_2 = O(\frac{\epsilon}{T N^{2MN+2\beta}(aK)^{MN^2+\beta N}}) = O(\epsilon^{\frac{MN^2+N\beta}{\beta}}).
    \end{eqnarray}
     
    Utilizing $\delta_1$ and $\delta_2$, we determine the bounds for the structural parameters of the PIE $\psi(\varphi(Y))$, referencing Lemma \ref{lem: relu}. Notably, $\psi(\varphi(Y)) = \sum_{t=1}^T \psi_t(\varphi_t(Y))$ signifies that the outlier network $\psi$ is a composition of $T$ sub-networks $\psi_t$, each being a weighted average of $N^{MN+\beta}$ distinct $\tilde{\psi}_{\s}$ functions. Each $\tilde{\psi}_{\s}$ is recursively constructed using at most $MN+\beta$ instances of $\tilde{x}_2$, as indicated in (\ref{eq: approx_decomp}), where $\tilde{x}_2(\cdot)$ possesses a depth and weight count of $O(\ln(1/\delta_2))$. Thus, each $\tilde{\psi}_{\s}$ exhibits a depth and parameter count of at most $(MN+\beta)O(\ln(1/\delta_2))$. Consequently, $\psi_t$ represents a weighted aggregation of up to $N^{MN+\beta}$ distinct $\tilde{\psi}_{\s}$ functions. Therefore, we establish that 
    \begin{align*}
        &L_{\psi} = (MN+\beta)O(\ln(1/\delta_2)) = O(\ln(\epsilon^{-\frac{MN^2+N\beta}{\beta}})),  \\
        &S_{\psi} = TN^{MN+\beta}(MN+\beta)O(\ln(1/\delta_2)) = O(\epsilon^{-\frac{MN}{\beta}}).
    \end{align*}
    Likewise, the inner network $\varphi$ comprises $T$ sub-networks $\varphi_1,\ldots, \varphi_T$, each defined iteratively using $N$ instances of $\tilde{x}_1$ and $T$ distinct learnable weights $w_t$, as outlined in \eqref{eq: varphi_construction}. Hence, we ascertain
    \begin{align*}
        &L_{\varphi} = NO(\ln(1/\delta_1)) = O(\ln(\epsilon^{-\frac{MN^2+N\beta}{\beta}})),\\
        &S_{\varphi} = KT + TN O(\ln(1/\delta_1)) = O(\epsilon^{-\frac{MN}{\beta}}).
    \end{align*}

    \subsection{Estimation error}
    \label{sec: est}
    Here, we evaluate the term
    \begin{equation*}
        \frac{2}{n} \sum_{i=1}^n \xi_i(\hat{f}(Y_{(i)}) - f(Y_{(i)})).
    \end{equation*}
    As defined in (\ref{eq: Y}), here $Y_{(i)} \in \R^{K\times N}$ is the transformation of the $i$-th sample $X_{(i)} \in \R^{M\times N}$ in the dataset.
    To employ concentration inequalities, we examine the expectation  of the upper bound of this term
    
    \begin{equation*}
        \E_{\xi}\left[\sup_{f\in \Fcal}\left|\frac{2}{n}\sum_{i=1}^n \xi_i(\hat{f}(Y_{(i)}) - f(Y_{(i)}))\right|\right].
    \end{equation*} 
    
    Initially, we define a subset $\Fcal_{\delta} \subset \Fcal$ as
    \begin{equation*}
        \Fcal_{\delta} = \{f-\hat{f} : \|f-\hat{f}\|_n \leq \delta, f\in \Fcal\}.
    \end{equation*}
    where $\delta$ is finite, given that functions in $\Fcal_{\delta}$ are bounded as
    \begin{equation*}
        \|f-\hat{f}\|_n \leq \|f-\hat{f}\|_{\infty} \leq \|f\|_{\infty} + \|\hat{f}\|_{\infty} \leq 2J, \quad f\in \Fcal.
    \end{equation*}

    Thus, by applying the chaining technique from Lemma \ref{lem:F-Chaining}, we deduce
    \begin{equation}
        E_{\xi} \left[ \sup_{f'\in \Fcal_{\delta}} \left| \frac{1}{n} \sum_{i=1}^n\xi_if'(Y_{(i)})\right| \right] \leq \frac{4\sqrt{2}\sigma}{n^{1/2}}\int_0^{\delta/2}\sqrt{2\log\Ncal(\epsilon', \Fcal_{\delta}, \| \cdot \|_{\infty})} d\epsilon',
    \label{expectation}
    \end{equation}
    and subsequently, we establish a bound for $\log\Ncal(\epsilon', \Fcal_{\delta}, \|\cdot\|_{\infty})$, analogous to Theorem 2 in \citep{cover}.
    \begin{lemma}
        \label{lem: cover}
        Let $f:[-M,M]^d \mapsto \R$ be fully connected ReLU neural network. Suppose $f$ has $L$ layers $W_1,\ldots, W_L$ with weights bounded by $B$ and let $\sigma$ denotes the ReLU activation function, we can represent $f$ as
        \begin{equation*}
            f(X) = W_L\sigma(W_{L-1} \cdots\sigma(W_1 X)\cdots).
        \end{equation*}
        Let $p_l$ denotes the width of layer $W_l$, then $f(X)$ is upper bounded by $\prod_{l=1}^L(p_lB)M = MB^{L}\prod_{l=1}^Lp_l$ and is Lipschitz continues with Lipschitz constant $\prod_{l=1}^L(p_lB)=B^{L}\prod_{l=1}^L p_l$. Suppose $f^*$ is another neural network with same structure as $f$ but with all the parameters $\epsilon$ away from $f$, then
        \begin{equation*}
            \|f-f^*\|_{\infty}\leq \epsilon \prod_{l=1}^L(p_lB)L M = \epsilon LM B^{L}\prod_{l=1}^Lp_l.
        \end{equation*}
    \end{lemma}
    Lemma \ref{lem: cover} is an application of Lemma 8 in \citep{cover}, we skip the detailed proof.
    
    Consider $\psi(\varphi), \psi^*(\varphi^*) \in \Fcal$, which are two PIEs with identical structures. However, each corresponding parameter in these PIEs differs by no more than $\zeta$. We aim to establish a bound for $\|\psi(\varphi) - \psi^*(\varphi^*)\|_{\infty}$
    
    \begin{equation*}
        |\psi(\varphi(Y))-\psi^*(\varphi^*(Y))| \leq \sum_{t=1}^T|\psi_t(\varphi_t(Y))-\psi^*_t(\varphi^*_t(Y)) |.
    \end{equation*}
    
    To further analyze this, we decompose the last term as follows
    \begin{equation}
        |\psi_t(\varphi_t(Y))-\psi^*_t(\varphi^*_t(Y)) |\leq 
        |\psi_t(\varphi_t(Y))-\psi_t(\varphi^*_t(Y))|+|\psi_t(\varphi^*_t(Y))-\psi^*_t(\varphi^*_t(Y)) |.
        \label{eq: cover-decomp}
    \end{equation}

    \textbf{Bound of the first term in \eqref{eq: cover-decomp}}
    
    We first bound $\|\varphi_t(Y)-\varphi_t^*(Y)\|_{\infty}$. Recall that $\varphi_t$ is iteratively defined with ReLU network $\tilde{x}_1$,which is a special case of $f$ of Lemma \ref{lem: cover}. Here we use $L_{\tilde{x}_1}$ to denote the number of layers in $\tilde{x}_1$. From the construction of $\varphi$, we know $L_{\varphi} = N L_{\tilde{x}_1}$. From Lemma\ref{lem: relu}, the width of $\tilde{x}_1$ is bounded by $C_w$. We apply Lemma \ref{lem: cover} with $L$ being $L_{\tilde{x}_1}$ and $p_l = C_w$, so for any $t,n$
    \begin{align*}
        \varphi_{t,1,n}(Y_n) &= w_t^{\top}Y_n \leq BK, \\
        \varphi_{t,2,n}(Y_n) &= \tilde{x}_1(\varphi_{t,1,n}(Y_n), \varphi_{t,1,n}(Y_n)) \leq (C_wB)^{L_{\tilde{x}_1}}BK, \\
        \varphi_{t,k+1,n}(Y_n) &= \tilde{x}_1(\varphi_{t,1,n}(Y_n), \varphi_{t,k,n}(Y_n)) \leq (C_wB)^{L_{\tilde{x}_1}} |\varphi_{t,k,n}(Y_n)| \leq (C_wB)^{i\times L_{\tilde{x}_1}}BK.
    \end{align*}
    Thus, $\varphi_{t,N,n}(Y_n) \leq (C_wB)^{L_{\varphi}} BK$ and $\varphi_t(Y) \leq N (C_wB)^{L_{\varphi}} BK$. By defining $\Delta_{t,k,n} = |\varphi_{t,k,n}(Y_n) - \varphi_{t,k,n}^*(Y_n)|$, we iteratively bound the differences
    
    \begin{align*}
        \Delta_{t,1,n} &= |w_t^{\top}Y_n-w_t^{*\top}Y_n| \leq K \zeta ,\\
        \Delta_{t,k,n} &\leq |\tilde{x}_1(\varphi_{t,1,n}(Y_n),\varphi_{t,k-1,n}(Y_n)) - \tilde{x}_1(\varphi_{t,1,n}(Y_n) \varphi_{t,k-1,n}^*(Y_n))| \\
        &+ |\tilde{x}_1(\varphi_{t,1,n}(Y_n),\varphi_{t,k-1,n}^*(Y_n)) - \tilde{x}_1^*(\varphi_{t,1,n}(Y_n), \varphi_{t,k-1,n}^*(Y_n))|\\
        &\leq (C_wB)^{L_{\tilde{x}_1}}\Delta_{t,k-1,n} + \zeta (C_wB)^{L_{\tilde{x}_1}}L_{\tilde{x}_1}\max(\varphi_{t,1,n}(Y_n),\varphi_{t,k-1,n}^*(Y_n)) \\
        &\leq (C_wB)^{L_{\tilde{x}_1}}\Delta_{t,k,n} + \zeta (C_wB)^{L_{\tilde{x}_1}}L_{\tilde{x}_1}(C_wB)^{i-1\times L_{\tilde{x}_1}}BK.
    \end{align*}
    
    Through recursive application, $\Delta_{t,N,n}$ is bounded by $\zeta (C_wB)^{N \times L_{\tilde{x}_1}} L_{\varphi} BK = \zeta (C_wB)^{L_{\varphi}} L_{\varphi} BK$, allowing us to establish a bound for $\|\varphi_t(Y) - \varphi_t^*(Y)\|_{\infty}$
    \begin{align}
        &\|\varphi_t(Y)-\varphi_t^*(Y)\|_{\infty} \nonumber\\
        &= \left\|\left[\sum_{n=1}^N\Delta_{t,1,n}(Y_n),\ldots, \sum_{n=1}^N\Delta_{t,N,n}(Y_n)\right]\right\|_{\infty} \nonumber\\
        &\leq \sum_{n=1}^N\Delta_{t,N,n}(Y_n) \nonumber\\
        &\leq \zeta (C_wB)^{L_{\varphi}}L_{\varphi}BKN.
        \label{eq: inner_epsilon}
    \end{align}
    
    To finalize the bound on the first term in \eqref{eq: cover-decomp}, it is necessary to establish the Lipschitz continuity of $\psi_t$. Assume $x, y \in \R^N$, then
    \begin{equation*}
        |\psi_t(x)-\psi_t(y)| = | \sum_{i = 0}^{MN+\beta} \sum_{\s \in \{ 1,..,N \}^i} \gamma_{t,\s} (\tilde{\psi}_{\s}(x)-\tilde{\psi}_{\s}(y))|.
    \end{equation*}
    According to Lemma \ref{lem: relu}, $\tilde{x}_2$ is Lipschitz continuous with a Lipschitz constant of $(C_wB)^{L_{\tilde{x}_2}}$. Let $\Gamma_i = |\tilde{\psi}_{\s,i}(x) - \tilde{\psi}_{\s,i}(y)|$, we have 
    \begin{align*}
        \Gamma_1 &= |x_{\s_1}-y_{\s_1}| \leq \|x-y\|_{\infty}, \\
        \Gamma_{i+1} &=\tilde{x}_2(\tilde{\psi}_{\s,i}(x), x_{\m_{i+1}}) - \tilde{x}_2(\tilde{\psi}_{\s,i}(y), y_{\m_{i+1}}) \\
        &\leq (C_wB)^{L_{\tilde{x}_2}} \max\{\Gamma_i, |x_{\s_{i+1}}-y_{\s_{i+1}}|\} \\
        &\leq (C_wB)^{L_{\tilde{x}_2}} \Gamma_i.
    \end{align*}
    After some computations, we find $|\tilde{\psi}_{\s}(x) - \tilde{\psi}_{\s}(y)| \leq (C_wB)^{L_{\psi}}\|x - y\|_{\infty}$. Consequently, $\psi_t$ exhibits Lipschitz continuity with
    \begin{equation*}
        |\psi_t(x)-\psi_t(y)| \leq N^{MN+\beta} (C_wB)^{L_{\psi}}\|x-y\|_{\infty}.
    \end{equation*}
    Therefore, the first term in \eqref{eq: cover-decomp} can be bounded as follows
    \begin{align}
        \|\psi_t(\varphi_t(Y))-\psi_t(\varphi^*_t(Y))\|_{\infty} &\leq N^{MN+\beta} (C_wB)^{L_{\psi}}\|\varphi_t(Y)-\varphi_t^*(Y)\|_{\infty} \nonumber \\
        &\leq \zeta N^{MN+ \beta}(C_wB)^{L_{\psi}+L_{\varphi}}L_{\varphi}K,
        \label{eq: est-first}
    \end{align}
    where the last inequality is derived by incorporating \eqref{eq: inner_epsilon} and disregarding the term $BN$.
    
    \textbf{Bound of the second term in \eqref{eq: cover-decomp}}
    
    With the similar argument as in \eqref{eq: inner_epsilon}, we derive the following bound
    \begin{align*}
        \|\psi_t(\varphi^*_t(Y))-\psi^*_t(\varphi^*_t(Y)) \|_{\infty} \leq |\sum_{i = 0}^{MN+\beta} \sum_{\s \in \{ 1,..,N \}^i} (\gamma_{t, \s}\tilde{\psi}_{\s}-\gamma_{t, \s}\tilde{\psi}^*_{\s})(\varphi^*_t(Y))|.
    \end{align*}
    Additionally, we establish that
    \begin{align*}
        |(\gamma_{t, \s}\tilde{\psi}_{\s}-\gamma_{t, \s}\tilde{\psi}^*_{\s})(\varphi^*_t(Y))| &\leq |(\gamma_{t, \s}\tilde{\psi}_{\s}-\gamma_{t, \s}\tilde{\psi}^*_{\s}+\gamma_{t, \s}\tilde{\psi}^*_{\s}-\gamma_{t, \s}\tilde{\psi}^*_{\s})(\varphi^*_t(Y))| \\
        & \leq |\gamma_{t,\s} \zeta(C_wB)^{L_{\psi}}L_{\psi}\|\varphi^*_t(Y)\|_{\infty}| + |(\gamma_{t,\s}-\gamma_{t,\s}^*)(C_wB)^{L_{\psi}}\|\varphi^*_t(Y)\|_{\infty}| \\
        & \leq [\zeta B(C_wB)^{L_{\psi}}L_{\psi} + \zeta (C_wB)^{L_{\psi}}]\|\varphi^*_t(Y)\|_{\infty} \\
        & \leq \zeta (C_wB)^{L_{\psi}}L_{\psi}(C_wB)^{L_{\varphi}}K \\
        & = \zeta (C_wB)^{L_{\psi}+L_{\varphi}}L_{\psi}K.
    \end{align*}
    In a similar fashion to our earlier derivation for $\varphi$, we apply bounds to $\tilde{\psi}_{\s}$ and evaluate the discrepancy between $\tilde{\psi}_{\s}$ and $\tilde{\psi}^*_{\s}$. For simplicity, we omit the lower order term $B$ in the fourth inequality. Consequently, we obtain
    \begin{equation}
        \|\psi_t(\varphi^*_t(Y))-\psi^*_t(\varphi^*_t(Y)) \|_{\infty} \leq \zeta (C_wB)^{L_{\psi}+L_{\varphi}}L_{\psi}N^{MN+\beta}K .
        \label{eq: est-second}
    \end{equation}
    To summarize, we bring (\ref{eq: est-first}) and (\ref{eq: est-second}) together
    \begin{equation}
    \label{est_error}
        \|\psi(\varphi(Y))-\psi^*(\varphi^*(Y))\|_{\infty}\leq \zeta T(C_wB)^{L_{\psi}+L_{\varphi}}(L_{\psi}+L_{\varphi})N^{MN+\beta}K .
    \end{equation}

    Given that the total number of parameters is constrained by $S_{\psi} + S_{\varphi}$, with each parameter bounded by $B$, we discretize the range of each parameter using a grid of size
    \begin{equation*}
        \Delta = \epsilon'/T(C_wB)^{L_{\psi}+L_{\varphi}}(L_{\psi}+L_{\varphi})N^{MN+\beta}K ,
    \end{equation*}
    which leads to an upper bound on the covering number as follows
    \begin{equation*}
        \log\Ncal(\epsilon', \Fcal_{\delta}, \| \cdot \|_{\infty}) \leq \log((\frac{2B}{\Delta})^{S_{\psi}+S_{\varphi}})= (S_{\psi}+S_{\varphi})\log(\frac{BT(C_wB)^{L_{\psi}+L_{\varphi}}(L_{\psi}+L_{\varphi})N^{MN+\beta}K }{\epsilon'}).
    \end{equation*}
    
    Returning to \eqref{expectation} and integrating, we derive
    \begin{equation*}  
        E_{\xi} \left[ \sup_{f'\in \Fcal_{\delta}} \left| \frac{1}{n} \sum_{i \in [n]}\xi_if'(Y_{(i)})\right| \right] \leq 2\sqrt{2} \frac{\sigma \sqrt{S_{\psi}+S_{\varphi}}\delta }{n^{1/2}} \left( \log  \frac{(T(C_wB)^{L_{\psi}+L_{\varphi}}(L_{\psi}+L_{\varphi})N^{MN+\beta}K}{ \delta} +1 \right).
    \end{equation*}
    Utilizing this bound for the expectation , we apply the Gaussian concentration inequality, as presented in \citep{foundation}, by considering $\frac{1}{n}\sum_{i=1}^n\xi_i f'(Y_{(i)})$ as the Gaussian process $X(t)$. The variance of $\frac{1}{n}\sum_{i=1}^n\xi_i f'(Y_{(i)})$ is upper bounded by $\frac{\sigma^2 \delta^2}{n}$, thus
    \begin{align}
        1 &-exp(-nu^2/2\sigma^2 \delta^2) \nonumber \\
        &\leq P_{\xi}\left(2 \sup_{f'\in \Fcal_{\delta}}|\frac{1}{n}\sum_{i=1}^n\xi_i f'(Y_{(i)})|\leq 2E_{\xi} [ \sup_{f'\in \Fcal_{\delta}} | \frac{1}{n} \sum_{i \in [n]}\xi_if'(Y_{(i)})| ]+u\right) \nonumber \\
        & \leq P_{\xi}\left(2 \sup_{f'\in \Fcal_{\delta}}|\frac{1}{n}\sum_{i=1}^n\xi_i f'(Y_{(i)})|\leq V_n \delta  (\log  \frac{(T(C_wB)^L LN^{MN+\beta}K}{\delta} +1)+u\right), \label{eq: est-concentration}
    \end{align}
    where $L=L_{\psi}+L_{\varphi}$, $S=S_{\psi}+S_{\varphi}$ and $V_n = 4\sqrt{2} \frac{\sigma \sqrt{S}}{n^{1/2}}$.
    
    Let $\delta = \max\{\|\hat{f}-f\|_n, V_n\}$, then by the definition of $\Fcal_{\delta}$, we have
    \begin{equation*}
        \sup_{f\in \Fcal}|\frac{1}{n}\sum_{i=1}^n\xi_i (\hat{f}(Y_{(i)}-f(Y_{(i)}))| \leq \sup_{f'\in \Fcal_{\delta}}|\frac{1}{n}\sum_{i=1}^n\xi_i f'(Y_{(i)})|.
    \end{equation*}
    Thus, for any $f \in \Fcal$, it follows that
    \begin{align}
        &| \frac{2}{n}\sum_{i=1}^n\xi_i(\hat{f}(Y_{(i)})-f(Y_{(i)}))| \nonumber\\
        &\leq \max\{\|\hat{f}-f\|_n, V_n\}\left\{V_n (\log  \frac{(T(C_wB)^L LN^{MN+\beta}K}{V_n} +1)\right\} +u \nonumber\\
        &\leq \frac{1}{4}(\max\{\|\hat{f}-f\|_n, V_n\})^2+2\left\{V_n(\log  \frac{(T(C_wB)^L LN^{MN+\beta}K}{V_n} +1) \right\}^2 +u\label{eq: after_xy},
    \end{align}
    with probability at least $1 -exp(-nu^2/2\sigma^2 \delta^2$). The last inequality holds since $xy\leq \frac{1}{4}x^2+2y^2$. Utilizing the inequality given in \eqref{eq: basic_inequality}, we have
    \begin{equation*}
        -\frac{2}{n}\sum_{i=1}^n\xi_i(\hat{f}(Y_{(i)})-f(Y_{(i)})) + \|f^*-\hat{f}\|_n^2 
        \leq \|f^*-f\|_n^2 .
    \end{equation*}
    Applying the inequality $\frac{1}{2}\|\hat{f} - f\|_n^2 \leq \|f - f^*\|_n^2 + \|f^* - \hat{f}\|_n^2$, we obtain
    \begin{equation}
        -\frac{2}{n}\sum_{i=1}^n\xi_i(\hat{f}(Y_{(i)})-f(Y_{(i)})) + \frac{1}{2}\|\hat{f}-f\|^2_n \leq 2\|f^*-f\|_n^2. \label{eq: modify_basic}
    \end{equation}
    Combining \eqref{eq: modify_basic} and \eqref{eq: after_xy}, we derive
    \begin{equation*}
        -\frac{1}{4}(\max\{\|\hat{f}-f\|_n, V_n\})^2-2\left\{V_n(\log  \frac{(T(C_wB)^L LN^{MN+\beta}K}{V_n} +1) \right\}^2 -u + \frac{1}{2}\|\hat{f}-f\|^2_n \leq 2\|f^*-f\|_n^2.
    \end{equation*}
    It can be verified that whether $\|\hat{f} - f\|_n \geq V_n$ or $\|\hat{f} - f\|_n \leq V_n$, the following holds
    \begin{equation}
        \|\hat{f}-f\|_n \leq 4\left\{V_n(\log  \frac{(T(C_wB)^L LN^{MN+\beta}K}{V_n} +1) \right\}^2 +2u +4\|f^*-f\|_n^2.
        \label{eq: bound-hatf-f}
    \end{equation}
    
    Apply (\ref{eq: bound-hatf-f}) to the inequality $\frac{1}{2}\|\hat{f}-f^*\|_n^2\leq \|f^*-f\|^2_{n}+\|\hat{f}-f\|_n^2$, we obtain
    \begin{equation}
        \|\hat{f}-f^*\|_n^2 \leq 10\|f^*-f\|_n^2+ 8\left\{V_n(\log  \frac{(T(C_wB)^L LN^{MN+\beta}K}{V_n} +1) \right\}^2 +4u,
        \label{eq: final-decomp-results}
    \end{equation}
     with probability at least $1-\exp(-nu^2/2\sigma^2 \delta^2)$ for all $u>0$.

    \subsection{Overall order}
    
    Recall that \eqref{eq: final-decomp-results} is valid for any $f \in \Fcal$, allowing us to select $f$ as $f_n$ from Section \ref{sec: app}. So that $\delta^2 = \max\{\|\hat{f}-f\|^2_n, V_n^2\}$. Then we set $u = \frac{1}{16} \delta^2$
    \begin{align}
        \|\hat{f}-f^*\|_n^2 &\leq 10\|f^*-f_n\|_n^2+ 8\left\{V_n(\log  \frac{(T(C_wB)^L LN^{MN+\beta}K}{V_n} +1) \right\}^2 + \frac{1}{4}(\|\hat{f}-f_n\|_n^2 + V_n^2) \nonumber \\ 
        &\leq 21\|f^*-f_n\|_n^2+ 17\left\{V_n(\log  \frac{(T(C_wB)^L LN^{MN+\beta}K}{V_n} +1) \right\}^2 , 
        \label{eq: u1}
    \end{align}
    with probability at least $1-\exp(-n \delta^2/2\sigma^2)$. The latter inequality follows since $\frac{1}{2}\|\hat{f}-f_n\|_n^2\leq \|f^*-f_n\|^2_{n}+\|\hat{f}-f^*\|_n^2$. Let $\epsilon=n^{-\frac{\beta}{MN+2\beta}}$ and substitute the order of of $S_{\psi}, S_{\varphi}, L_{\psi}, L_{\varphi}$ from Section \ref{sec: app}. We obtain that $V_n = O(n^{-\frac{2\beta}{MN+2\beta}})$, therefore \eqref{eq: u1} holds with probability $1-\exp(-n \delta^2/2\sigma^2) \geq 1-\exp(-n V_n^2/2\sigma^2)$ converging to one. Then we take the expectation of both sides of the inequality 
    with respect to $P_X$
    
    \begin{align}
        \|&\hat{f}-f^*\|_{L_2(P_X)}^2 \nonumber \\
        & \leq 21\|f_n-f^*\|^2_{L_2(P_X)} + 17\left\{V_n(\log  \frac{(T(C_wB)^L LN^{MN+\beta}K}{V_n} +1) \right\}^2  \nonumber\\
        & \leq 21\epsilon^2 + 17\left\{V_n(\log  \frac{(T(C_wB)^L LN^{MN+\beta}K}{V_n} +1) \right\}^2 , \label{eq: final-step}
    \end{align}
    where the latter inequality follows since $\|f_n - f^*\|_{L_2(P_X)} \leq \|f_n - f^*\|_{\infty} \leq \epsilon^2$. Disregarding constants and lower-order terms, we obtain
    \begin{align*}
        \|\hat{f}-f^*\|_{L_2(P_X)}^2 &\leq 21\epsilon^2 + 17\left\{V_n(\log\frac{\epsilon^{-\frac{MN}{\beta}}(C_wB)^{\ln(\epsilon^{-\frac{MN^2+N\beta}{\beta}}+1)}}{V_n})\right\}^2 \\
        & = O(n^{-\frac{2\beta}{MN+2\beta}}),
    \end{align*}  
    with probability converging to one. So we conclude
    \begin{equation}
        \| \hat{f} - f^*\|^2 _{L^2(P_X)} = O_p(n^{-\frac{2\beta}{MN+2\beta}}).
        \label{eq: final-rate}
    \end{equation}
    
    \section{Proof of Theorem \ref{theo: minimax}}
    This section is devoted to deriving a lower bound for the $L_2$ minimax risk associated with the class of permutation invariant functions. Additionally, we aim to demonstrate that the PIE proposed in this study constitutes an optimal estimator in the minimax framework.
    
    \begin{definition}[$L_2$ minimax risk]
        Given a random variable $X$ following a probability measure $P_X$ on $\mathbb{R}^d$, the $L_2$ minimax risk of estimation associated with any function space $I \in L_2(P_X)$ is defined as
        \begin{equation*}
            r^2_n(I,P_X,\sigma) = \inf_{\hat{f}\in \Acal_n}\sup_{f\in I}\|\hat{f}-f\|^2_{L_2(P_X)},
        \end{equation*}
        where $\Acal_n$ is the space of all measurable functions of data in $L_2(P_X)$ and $\sigma$ is the variance of Gaussian noise in the data generating process.
    \end{definition}
    Given the definition of permutation invariance as specified in Assumption \ref{def: PI}, it follows that $P^{\beta, 1\times MN}$ is a subset of $P^{\beta, M\times N}$. Consequently, $r^2_n(P^{\beta, M\times N},P_X,\sigma) \geq r^2_n(P^{\beta, 1 \times MN},P_X,\sigma)$. This implies that any lower bound on the $L_2$ minimax risk of $P^{\beta, 1 \times MN}$ also serves as a valid lower bound for $P^{\beta, M\times N}$. For ease of notation, let us denote $d = MN$ and use $P^{\beta, d}$ to refer to $P^{\beta, 1\times MN}$. 
    
    We try to get the $L_2$ minimax lower bound using the relationship between minimax risk and packing number. So we first bound the packing number $\Tcal(\epsilon, P^{\beta, d}, \|\cdot \|_{L_2})$ using the lemma that follows
    
    \begin{lemma}
        The packing number of $P^{\beta, d}$ is lower bounded with 
        \begin{equation*}
            \Tcal(\epsilon, P^{\beta, d}, \|\cdot \|_{L_2}) \geq\exp\{M_0 (1/\epsilon)^{d/\beta}\}. 
        \end{equation*} 
    In addition,  we have 
        \begin{equation*}
            \Tcal(\epsilon, P^{\beta, d}, \|\cdot \|_{L_2(P_X)}) \geq\exp\{M_0 (\underline{P_X}/\epsilon)^{d/\beta}\}, 
        \end{equation*}
        where $M_0$ depends on only $\beta$ and $d$. 
    \end{lemma}
    
    \begin{proof}
        In alignment with the approach outlined in Lemma 6.1 of \citep{minimax}, we define
        \begin{equation*}
            \Kcal(X_1, \ldots, X_d) = \prod_{j=1}^d \Kcal_0(X_j), \quad \Kcal_0(t) = te^{-1/(1-t^2)}\I(|t|\leq 1),t\in \mathbb{R}. 
        \end{equation*}  
        For arbitrary radius $h \in (0,1/2)$, we evenly split $[0,1]$ into $m = \left\lceil \frac{1}{2h} \right\rceil $ intervals and obtain $D = m^d$ rectangular grid points $\left(\frac{j_1}{m}-h, \ldots, \frac{j_d}{m}-h\right)$  with $\left(j_1,\ldots, j_d\right) \in \left\{ 1,\ldots, m\right\}^d$. We use $\left\{x_k: k=1\ldots, D\right\}$ to represent these grids and assume $h$ is small enough to ensure $D\geq 8$. 
        
        For simplicity, let $\Wcal = \Wcal_J^{\beta,\infty}([0, 1]^{MN})$ and let $\|\cdot\|_{\Wcal}$ denote the norm of the Soblev space. Correspondingly, we construct functions $\phi_k$ for each $1 \leq k \leq D$
        \begin{equation*}
            \phi_k(x) = \frac{1}{\| \Kcal \|_{\Wcal} }h^{\beta}\Kcal\left(\frac{x-x_k}{h}\right), \quad x\in [0,1]^d,
        \end{equation*}
        with each $\phi_k$ having its support confined to $[x_k - h, x_k + h]^d$. It's noteworthy that $\Kcal_0(t)$ is an odd function, ensuring that $\int \psi_k(x) dx = 0$. Moreover, the supports of $\phi_k$ and $\phi_j$ for $k \neq j$ do not overlap, thus allowing for
    \begin{equation}
        \int \phi_k \phi_j dx =0,  \quad 
                \| \phi_1 + \dots + \phi_k \|_{L_2}^2 = \sum_{i=1}^k \| \phi_i \|_{L_2}^2 = kh^{2\beta+d}\frac{\| \Kcal \|_{L_2}}{\| \Kcal \|_{\Wcal} }. 
                \label{eq: sum-ind}
    \end{equation}      
        Furthermore, we define permutation invariant functions $\phi^*_k$ based on $\phi_k$ as 
        \begin{equation}
            \phi^*_k(x) = \frac{1}{|S_d|^{1/2}} \sum_{\varPi \in S_d} \frac{1}{\| \Kcal \|_{\Wcal} }h^{\beta}\Kcal\left(\frac{x-\varPi_{x_k}}{h}\right), \quad x\in [0,1]^d,
            \label{eq: phi*}
        \end{equation}
        so $\phi_k^*(x)$ is a permutation invariant function for any $k$ and
        \begin{align*}
            \| \phi^*_k \|^2_{L_2} &= \frac{1}{|S_d|} \left\| \sum_{\varPi \in S_d}  \frac{1}{\| \Kcal \|_{\Wcal} }h^{\beta}\Kcal\left(\frac{x-\varPi_{x_k}}{h}\right) \right\|_{L_2}^{2} \\
            &\geq \frac{1}{|S_d|}  \sum_{\varPi \in S_d}  \left\| \frac{1}{\| \Kcal \|_{\Wcal} }h^{\beta}\Kcal\left(\frac{x-\varPi_{x_k}}{h}\right) \right\|_{L_2}^{2} \\
            &=h^{2\beta+d}\frac{\| \Kcal \|_{L_2}}{\| \Kcal \|_{\Wcal} },
        \end{align*}
        
        where the second inequality becomes equal when for all permutations $\varPi$ in $S_d$, $\varPi_{x_k} \neq x_k$. Consider the total number of distinct functions $\phi^*_k$, denoted as $D^*$. In (\ref{eq: phi*}), we construct a set of $D$ functions  $\phi_k^*$ derived from $\phi_k$. Specifically, each $\phi^*_k$ is identical to $\phi^*_j$ for distinct indices $k$ and $j$ if and only if there exists a permutation $\varPi$ such that $\varPi_{x_k}$ equals $x_j$. Considering that there are at most $|S_d| \leq d^d$ distinct permutations for the $d$-dimensional vector $x_k$, we can infer that $D^*$ is bounded from below by $D/d^d$.
    
    
        Let $\Omega=\{0,1\}^{D^*}$ and for each $\omega \in \Omega$ define $f_{\omega} = \sum_{k=1}^{D^*}\omega_k \phi^*_k$. It is evident that each $f_{\omega}$ belongs to $P^{\beta, d}$
        \begin{equation*}
            \| f_{\omega} - f_{\omega'}\|_{L_2} = \left\{\sum_{k=1}^{D^*}(\omega_k -  \omega_k') ^2\int {\phi_k^*}^2 dx \right\}^{1/2} \geq  h^{\beta+d/2}\frac{\| \Kcal \|}{\| \Kcal \|_{\Wcal}} \psi^{1/2} (\omega, \omega'),
        \end{equation*}
        where $\psi(\omega, \omega') = \sum_{k=1}^{D^*}\I (\omega_k \neq  \omega_k')$ represents the hamming distance. Applying the Varshamov-Gilbert bound from coding theory, it can be shown that there exist $U \geq 2^{D^*/8}$ binary strings $\omega^{(1)}, \ldots, \omega^{(U)} \in \Omega$ such that $\psi(\omega^{(k)}, \omega^{(k')}) \geq D^*/8$ for $0\leq k \leq k' \leq A$. Then
    
        \begin{equation*}
            \|f_{\omega^{(k)}}-f_{\omega^{(k')}}\|_{L_2}\geq h^{\beta+d/2}\frac{\|\Kcal\|_{L_2}}{\|\Kcal\|_{\Wcal}}\sqrt{\frac{D^*}{8}}\geq M_1h^{\beta},
        \end{equation*}
        where $M_1=\|\Kcal\|_{L_2}/\{2^{(d+3)/2}d^{d/2}\|\Kcal\|_{\Wcal}\}$ is a constant that depends solely on $\beta$ and $d$. By setting $\epsilon = M_1h^{\beta}$, we are able to identify $U$ distinct functions within the function class $P_{\beta, d}$. These functions are separated from each other by at least $\epsilon$ with respect to the $L_2$ norm distance, with
        \begin{align*}
            U &\geq\exp\{D^*\log(2)/8\} \nonumber \\
            & \geq\exp\{m^d \log(2)/d^d8\} \nonumber \\
            & \geq\exp\{(1/2h)^d \log(2)/d^d8\} \nonumber \\
            & \geq\exp\{(M_1/\epsilon)^{d/\beta} \log(2)/d^d2^{d+3}\} \nonumber \\
            & =\exp\{M_0(1/\epsilon)^{d/\beta}\},
        \end{align*} 
        where $M_0 = (M_1^{d/\beta} \log(2))/d^d2^{d+3}$ depends on only $d$ and $\beta$. Consequently, in accordance with the definition of the packing number, we have
        \begin{equation*}
            \Tcal(\epsilon, P_{\beta, d}, \|\cdot \|_{L_2}) \geq\exp\{M_0 (1/\epsilon)^{d/\beta}\}.
        \end{equation*}
    For any two functions $f, f' \in P^{\beta,d}$ that satisfy $\|f - f'\|_{L_2} \geq \epsilon$, it follows that $\|f - f'\|_{L_2(P_X)} \geq \epsilon \underline{P_X}$. Therefore
        \begin{equation*}
            \Tcal(\epsilon \underline{P_X}, P_{\beta, d}, \|\cdot \|_{L_2(P_X)}) \geq\exp\{M_0 (1/\epsilon)^{d/\beta}\},
        \end{equation*}
        and upon setting $\epsilon' = \underline{P_X}\epsilon$, we deduce
        \begin{equation*}
            \Tcal(\epsilon', P_{\beta, d}, \|\cdot \|_{L_2(P_X)}) \geq\exp\{M_0 (\underline{P_X}/\epsilon)^{d/\beta}\}.
        \end{equation*}
    \end{proof}
    
    By Theorem 6 of \citep{yuhongy}, the minimax risk $r_n$ is the solution to $\log(\Tcal(r_n,P^{\beta, d}, \|\cdot\|_{L_2(P_X)}))=\frac{nr_n^2}{\sigma}$. From the definition of packing number, $\Tcal(r_n,P^{\beta, d}, \|\cdot\|_{L_2(P_X)})$ decreases as $r_n$ increases, while $r^2_n$ increases with $r_n$. So any $r$ satisfying $\Tcal(r,P^{\beta, d}, \|\cdot\|_{L_2(P_X)})\geq \frac{nr^2}{\sigma}$ is a lower bound of $r_n$. Choosing $r = (M_0\sigma)^{\frac{\beta}{2\beta+d}}(P_X)^{\frac{d}{2\beta+d}}n^{-\frac{\beta}{2\beta+d}}$, it can be verified that $\Tcal(r, P^{\beta, d}, \|\cdot\|_{L_2(P_X)}) \geq \frac{nr^2}{\sigma}$. Consequently, $r^2$ forms a lower bound for the minimax risk $r_n^2$.

    Returning to the notation where $d = MN$, we deduce that
    \begin{equation*}
        r^2_n(P^{\beta, M\times N},P_X,\sigma) \geq r^2_n(P^{\beta, 1\times MN},P_X,\sigma) \geq C(\beta, M,N)n^{-\frac{2\beta}{2\beta+MN}},
    \end{equation*}
    
    which aligns with the convergence rate of PIE, as established in (\ref{eq: final-rate}), up to a constant factor. Thus, the minimax optimality of PIE in approximating fully permutation invariant functions is established.

    \section{Proof of Corollary \ref{co: PPIE rate}}
    \label{appendix:frate} 
    In this section, our objective is to substantiate Corollary \ref{co: PPIE rate} through a broader perspective. To achieve this, we  expand the scope beyond the conventional permutation invariance assumptions and the associated Permutation Invariant Estimator PIE. We introduce and focus on the concept of partially permutation invariant functions and the corresponding Partially Permutation Invariant Estimator PPIE. Once we ascertain the convergence rate of PPIE, it will provide the necessary foundation to validate Corollary \ref{co: PPIE rate}.

    \begin{definition}[Partially Permutation Invariant Function]
    \label{def:ppie}
    A function $f$ is partially permutation invariant if its input can be separate into $P$ input matrix $\{X_1,\cdots, X_P\}$ , where each $X_p \in \mathbb{R}^{M \times N_p}$. The function remains unchanged under column permutations within each individual input matrix. Formally, this property is  expressed  as
    \begin{eqnarray}
        f(X_1,X_2,\cdots,X_P) = f(\varPi^{(1)}_{X_1},\varPi^{(2)}_{X_2},\cdots,\varPi^{(P)}_{X_P}),
    \end{eqnarray}
        for any permutation $\varPi^{(k)}\in \Scal_{N_p}$.
    \end{definition}
    
    The corresponding estimator, known as the Partially Permutation Invariant Estimator (PPIE), is defined as follows:

    \begin{definition}[Partially Permutation Invariant Estimator(PPIE) Structure]
        \label{def:ppie network}
        Given a partially permutation invariant function $f$ with $P$ input matrix $X = \{X_1,\cdots,X_P\}$,where $X_p \in \mathbb{R}^{M \times N_p}$ as described in Definition \ref{def:ppie}, the PPIE neural network is structured as
        \begin{equation*}
            \psi\left(\frac{1}{N_1}\sum_{n=1}^{N_1}\varphi^{(1)}(X_{1,i}),\frac{1}{N_2}\sum_{n=1}^{N_2}\varphi^{(2)}(X_{2,i}),\cdots,\frac{1}{N_P}\sum_{n=1}^{N_P}\varphi^{(K)}(X_{P,i})\right),   
        \end{equation*}
        where $X_{p,i}$ denotes the $i$-th column of $X_p$, and $\varphi$ and $\psi_p$ are neural networks.
    \end{definition}

    Our goal is to establish the bound for
    \begin{align*}
        \|f^*-\hat{f}\|_n^2 \leq \|f^*-f\|_{\infty}^2 +\frac{2}{n}\xi_i(\hat{f}(x_i)-f(x_i)) .
    \end{align*}
    It can be shown that the fully permutation invariant target function, as stipulated in Assumption \ref{ass:PI}, is a specific case of the partially permutation invariant function. This is achieved by setting $P = 2$, with $X_1 \in \mathbb{R}^{(M+1) \times |\Ncal(i)|}$ and $X_2 \in \mathbb{R}^{(M+1) \times 1}$. Therefore, our focus shifts to examining the convergence rate of the PPIE function class.
    
    \subsection{Approximation Error}
    For ease of discussion, let us denote $N = \sum_{p=1}^P N_p$. Following a similar approach as in Section \ref{sec: app}, we define an approximation of the function $f^*$ by
    \begin{equation*}
        f^*_1(X) = \sum_{\m\in\{0,1,\cdots,N^{\prime}\}^{K\times M \times N}} \phi_{\m}(X)  \sum_{\p: |\p|\leq\beta}\frac{D^{\p} f}{\p!}\bigg|_{X = \frac{\m}{N'}}(X-\frac{\m}{N'})^{\p},
    \end{equation*}
    where $\phi_{\m}$ is defined identically to its counterpart in Section \ref{sec: app}.
    
    Considering $X = \{X_1, \cdots, X_P\}$ and defining $|\widehat{\Scal}| = \prod_{p=1}^P |\Scal_{N_p}|$, we construct the partially permutation invariant (PI) approximation using $f_1^*$
    \begin{equation*}
        f_1(X) = f_1(X_1,\cdots,X_P) = \frac{1}{|\widehat{\Scal}|}\sum_{\substack{\varPi_1 \in \Scal_{N_1} \cdots \\ \varPi_K \in \Scal_{N_P}}} f_1^{*}\left((\varPi_1)_{X_1}, \cdots, (\varPi_K)_{X_P}\right).
    \end{equation*}
    
    Using similar arguments, it can be shown that
    \begin{equation*}
        \|f^*(X)-f_1(X)\|_{\infty} \leq \frac{2^d d^{\beta}J}{\beta!}\left(\frac{1}{N'}\right)^{\beta}.
    \end{equation*}
    
    Let $K'_p = (N' + 1) \times N_p$ and define $K' = \sum_{p=1}^P K'_p$. We then construct $Y(X, N') = (Y_1(X_1, N'), \ldots, Y_P(X_P, N'))$ where $Y_p(X_p, N') \in \mathbb{R}^{K'_p \times M}$ is defined analogously to (\ref{eq: Y}). Consequently, $f_1(X)$ can be transformed into $g(Y)$ by 
    \begin{align*}
        f_1^*(X)  &= g^*(Y(X,N')), \\
        f_1(X) &= \frac{1}{|\widehat{\Scal}|} \sum_{\substack{\varPi_1 \in \Scal_{N_1} \cdots \\ \varPi_K \in \Scal_{N_P}}} g^*\left(Y_1((\varPi_1)_{ X_1},N'), \cdots, Y_p((\varPi_K)_{X_P},N')\right)= g(Y). 
    \end{align*}
    
    From the arguments in Section \ref{sec: app}, we know $g(Y)$ is a partially permutation invariant polynomial. To further decompose $g(Y)$, we introduce the following lemma which is similar to Lemma \ref{lem: weyl}.
    \begin{lemma}[Partially PI polynomial decomposition]
    \label{lem: PPIE decompo}
        Given $Y = \{Y_1,\cdots,Y_P\}$ where  $Y_p\in\mathbb{R}^{M\times N_p}$, any partially permutation invariant polynomial $g(Y)$ can be  expressed  in the following form
        \begin{equation*}
    g(Y)=\sum_{q=1}^Q h_{1, q}\left(Y_1\right) h_{2, q}\left(Y_2\right) \cdots h_{P, q}\left(Y_P\right),
    \end{equation*}
    where $Q$ is an integer, and $h_{p,q}(Y_p)$ are permutation invariant polynomials.
    \end{lemma}
    With Lemma \ref{lem: PPIE decompo}, we can show $g(Y) = \sum_{q=1}^Q h_{1, q}\left(Y_1\right) h_{2, q}\left(Y_2\right) \cdots h_{P, q}\left(Y_P\right)$, where each function $h_{p, q}$ is a fully permutation invariant polynomial. Therefore, we can further decompose each $h_{p,q}$ into a power sum basis as described in Lemma \ref{lem: Hilbert}
    \begin{equation*}
        h_{p,q}(Y_p) = \sum_{t=0}^{T_{p,q}}s_{p,q,t}(\sum_{n=1}^{N_p}(a_{p,q,t}^{\top}Y_{p,n})^1,\cdots,\sum_{n=1}^{N_p}(a_{p,q,t}^{\top}Y_{p,n})^{N_p}).
    \end{equation*}
    
    With this decomposition of $g(Y)$, we can construct a PPIE, $\widetilde{g}(Y)$, as follows
    \begin{equation*}
        \widetilde{g}(Y) = \tau(\widetilde{h}_{1,q}(Y_1), \ldots, \widetilde{h}_{p,q}(Y_p), \ldots, \widetilde{h}_{P,Q}(Y_P)),
    \end{equation*}
    where $\widetilde{h}_{p,q}(Y_p)$ approximates $h_{p,q}(Y_p)$, and $\tau$ approximates the product of the $h_{p,q}$ functions. Then we propose the structure of $\tau$ and $\widetilde{h}_{p,q}$ in detail. $\tau$ is iterativly defined with $\tilde{x}$ such that
    \begin{equation}
        \tau(\widetilde{h}_{1,1}, \ldots, \widetilde{h}_{1,Q}, \widetilde{h}_{2,1}, \ldots, \widetilde{h}_{2,Q},\ldots, \widetilde{h}_{P,Q}) = \sum_{q=1}^{Q}\tilde{x}(\widetilde{h}_{1,q},\tilde{x}(\widetilde{h}_{2,q}\ldots,\widetilde{h}_{P,q})),
    \end{equation}
    where $\tilde{x}$ is a ReLU network from Lemma \ref{lem: relu} with an error of $\epsilon = \delta_3$. Following the methodology in Section \ref{sec: app}, each $h_{p,q}$ can approximated by a PIE, denoted as $\widetilde{h}_{p,q}$.
    \begin{equation*}
        \widetilde{h}_{p,q}(Y_p) = \sum_{t=0}^{T_{p,q}}\psi_{p,q,t}(\varphi_{p,q,t,1}(Y_p),\cdots,\varphi_{p,q,t,N_p}(Y_p)),
    \end{equation*}
    where $\varphi_{p,q,t,n}$ are iteratively defined using $\tilde{x}_1$, akin to $\varphi_{t,n}$ in \eqref{eq: varphi_t}. Similarly, $\psi_{p,q,t}$ is defined iteratively with $\tilde{x}_2$, as is $\psi_t$ in \eqref{eq: psi_t}. The approximation errors for $\tilde{x}_1$ and $\tilde{x}_2$ are denoted as $\delta_1$ and $\delta_2$, respectively. We establish that $T_{p,q} = O((N')^{MN_p})$ and accordingly define $T = O((N')^{MN})$.
    
    With some calculation, we can show function $\tilde{g}$ is a PPIE as in Definition \ref{def:ppie network}.  The function class of PPIE, denoted as $\Fcal = \Fcal(S_{\tau}, L_{\tau}, S_{\psi}, L_{\psi}, S_{\varphi}, L_{\varphi}, B)$, encapsulates the structure of $\tilde{g}$. Here, $S_{\tau}, S_{\psi}, S_{\varphi}$ and $L_{\tau}, L_{\psi}, L_{\varphi}$ represent the total number of parameters and the depth of the networks for $\tau, \psi, \varphi$, respectively. All the paramters in these networks are bounded by $B$.
    
    Considering the definition of $f_n$, we have
    \begin{equation*}
        \|f_1-f_n\|_{\infty} \leq \|f_1-\tilde{f}\|_{\infty} \leq |g(Y) - \tilde{g}(Y)|, \quad \forall Y .
    \end{equation*}
    
    To establish the bound for $\|f_1 - f_n\|_{\infty}$, it suffices to bound $|g(Y) - \tilde{g}(Y)|$. This decomposition and bounding process will be carried out following the approach described in Section \ref{sec: app}.
    
    \begin{align*}
        |\tilde{g}(Y)-g(Y)| \nonumber 
        &\leq\sum_{q=1}^{Q}|\tilde{x}(\widetilde{h}_{1,q}(Y_1),\tilde{x}(\widetilde{h}_{2,q}(Y_2),\tilde{x}(\widetilde{h}_{3,q}(Y_3),\cdots)))-h_{1, q}\left(Y_1\right) h_{2, q}\left(Y_2\right) \cdots h_{P, q}\left(Y_P\right)|\\
        &\leq\sum_{q=1}^{Q}\left[\underbrace{|\tilde{x}(\widetilde{h}_{1,q}(Y_1),\tilde{x}(\widetilde{h}_{2,q}(Y_2),\tilde{x}(\widetilde{h}_{3,q}(Y_3),\cdots)))-\widetilde{h}_{1, q}\left(Y_1\right) \widetilde{h}_{2, q}\left(Y_2\right) \cdots \widetilde{h}_{P, q}\left(Y_P\right)|}_{\eta_1}\right . \\
        &\left . +\underbrace{|\widetilde{h}_{1, q}\left(Y_1\right) \widetilde{h}_{2, q}\left(Y_2\right) \cdots \widetilde{h}_{P, q}\left(Y_P\right)-h_{1, q}\left(Y_1\right) h_{2, q}\left(Y_2\right) \cdots h_{P, q}\left(Y_P\right)|}_{\eta_2}\right].
        \label{PPIE decomp}
    \end{align*}

    From results in Section \ref{sec: app}, it is established that for each pair of indices $p$ and $q$, the bound $\|h_{p,q}\|_{\infty} \leq T_{p,q}(MN_p+\beta)(aK_k^{\prime})^{MN_p+\beta}$ holds true. If we define $G$ as $T(MN+\beta)(aK^{\prime})^{MN+\beta}$, then it follows that $\|h_{p,q}\|_{\infty} \leq G$ for all $p, q$. According to \eqref{eq: in-out diff}, we note that each $\|\widetilde{h}_{p,q} - h_{p,q}\|_{\infty}$ can be upper bounded by $T_{p,q}(N(aK'_k)^N)^{MN_p+\beta}(\delta_1+ \delta_2)$, provided we disregard lower order terms. Setting $\delta$ to be $T(N(aK')^N)^{MN+\beta}(\delta_1+ \delta_2)$, we find that $\|\widetilde{h}_{p,q} - h_{p,q}\|_{\infty} \leq \delta$ for every $p, q$, upon neglecting lower-order components. Consequently, $\|\widetilde{h}_{p,q}\|_{\infty}$ is bounded by $G+\delta$.
    
    Accordingly, the bound for $\eta_1$ can be similarly derived, utilizing the results in \eqref{eq: reduction}
    \begin{equation*}
        \eta_1\leq (G+\delta)^{P+1}\delta_3.
    \end{equation*}
    
    To establish an upper bound for $\eta_2$, we consider the following inequalities
    \begin{align*}
        \eta_2&\leq |\widetilde{h}_{1} \widetilde{h}_{2} \cdots \widetilde{h}_{P} - h_{1}h_{2}\cdots h_P|\\
        &\leq|\widetilde{h}_{1} \widetilde{h}_{2} \cdots \widetilde{h}_{P}-\widetilde{h}_{1} \widetilde{h}_{2} \cdots \widetilde{h}_{P-1}h_P|\\
        &+|\widetilde{h}_{1} \widetilde{h}_{2} \cdots \widetilde{h}_{P-1}h_P- \widetilde{h}_{1} \widetilde{h}_{2} \cdots h_{P-1}h_P|\\
        &+\cdots\\
        &+|\widetilde{h}_{1} h_{2}\cdots h_P- h_{1}h_{2}\cdots h_P|\\
        &\leq P (G+\delta)^{P-1}\delta.
    \end{align*}
    
    Given that $\delta$ is a relatively small error term compared to $G$, and considering that $P$ is independent of $\delta$, we can neglect the lower order terms. Consequently, we arrive at the bounds $\eta_1 \leq G^P \delta_3$ and $\eta_2 \leq PG^P\delta$. To summarize
    \begin{equation*}
        \|f_n-f_1\|_{\infty} \leq |\tilde{g}(Y)-g(Y)| \leq Q(\eta_1+\eta_2) \leq QPG^P\delta + QG^P\delta_3.
    \end{equation*}
    
    Remembering that $\|f_n - f^*\|_{\infty}$ can be bounded by $\|f_1 - f^*\|_{\infty} + \|f_n - f_1\|_{\infty}$, we aim to ensure $\|f_n - f^*\|_{\infty} \leq \epsilon$. To achieve this, we set $\|f_1 - f^*\|_{\infty} = \frac{\epsilon}{2}$ and constrain $\|f_n - f_1\|_{\infty}$ to be at most $\frac{\epsilon}{2}$.
    
    Let $\|f_1-f^*\| =\frac{\epsilon}{2}$, we can get $N' =O(\epsilon^{-1/\beta})$.
    
    To ensure $\|f_n-f_1\| \leq \frac{\epsilon}{2}$ , we set $QPG^P\delta = QG^P\delta_3 = \frac{\epsilon}{4}$. By setting $QG^P\delta_3 = \frac{\epsilon}{4}$
    \begin{eqnarray*}
        \delta_3 = O(\frac{\epsilon}{QT^P(MN+\beta)^P(aK^{\prime})^{P(MN+\beta)}})=O(\epsilon^{\frac{(MN+\beta)P}{\beta}}).
    \end{eqnarray*}
    
    Given that $\delta = (N(aK')^N)^{MN+\beta}(\delta_1 + \delta_2)$, and aiming for $QPG^P\delta = \frac{\epsilon}{8}$, we set
    \begin{align*}
        &\delta_1 = \delta_2 = O(\frac{\epsilon}{QN^{MN+\beta}T^P(MN+\beta)^K(aK^{\prime})^{(N+K)(MN+\beta)}})=O(\epsilon^{\frac{(MN^2+\beta N)P}{\beta}}).\\
    \end{align*}
    With $\delta_3$ established, we can determine the network parameters for $\tau$. Recalling that $\tau$ sums up $Q$ distinct sub-networks, each constructed iteratively with $P$ instances of $\tilde{x}$, the depth and total number of parameters of $\tau$ can be approximated as
    \begin{align*}
        L_{\tau} = P O (\ln(1/\delta_3))=O(\ln(\epsilon^{-\frac{(MN+\beta)P}{\beta}})),\\
        S_{\tau} = QK O(\ln(1/\delta_3)) = O(\ln(\epsilon^{-\frac{(MN+\beta)P}{\beta}})).
    \end{align*}

    The network parameters for $\psi$ and $\varphi$ are derived in a similar fashion to those in Section \ref{sec: app}, using $\delta_1$ and $\delta_2$. The details are omitted for brevity.
    \begin{align}
        &L_{\psi} = (MN+\beta)\times O(\ln(1/\delta_2)) = O(\ln(\epsilon^{-\frac{(MN^2+N\beta)P}{\beta}})),\nonumber\\
        &S_{\psi} = TN^{MN+\beta}O(\ln(1/\delta_2)) = O(\epsilon^{-\frac{MN}{\beta}}),\nonumber\\
        &L_{\varphi}=N O(\ln(1/\delta_1)) = O(\ln(\epsilon^{-\frac{(MN^2+N\beta)P}{\beta}})),\nonumber\\
        &S_{\varphi} = PT+TNO(ln(1/\delta_1)) =O(\epsilon^{-\frac{MN}{\beta}}) \nonumber.
    \end{align}

    \subsection{Estimation Error}
    \label{ppie est}
    In this section, our focus is on bounding the term 
    \begin{eqnarray*}
        \frac{2}{n}\sum_{i=1}^{n}\xi_i(\hat{f}(x_{i})-f(x_i)).
    \end{eqnarray*}
    The primary objective here is to establish an upper limit for the covering number associated with the PPIE function class.
    
    Adopting the approach used in Section \ref{sec: est}, we consider two functions from the class $\Fcal$, denoted as $f$ and $f^*$, with the all corresponding parameters of these functions differ by at most $\zeta$. With this setup, we proceed to the following decomposition
    \begin{align*}
        \|f-f^*\|_{\infty} &= \| \sum_{q=1}^{Q}\tilde{x}(\widetilde{h}_{1,q},\tilde{x}(\widetilde{h}_{2,q}\ldots,\widetilde{h}_{P,q}))
        -\sum_{q=1}^{Q}\tilde{x}^*(\widetilde{h}^*_{1,q},\tilde{x}^*(\widetilde{h}^*_{2,q}\ldots,\widetilde{h}^*_{P,q}))\|_{\infty}\nonumber\\
        &\leq \underbrace{\sum_{q=1}^{Q} \| \tilde{x}(\widetilde{h}_{1,q},\tilde{x}(\widetilde{h}_{2,q}\ldots,\widetilde{h}_{P,q}))
        -\tilde{x}(\widetilde{h}^*_{1,q},\tilde{x}(\widetilde{h}^*_{2,q}\ldots,\widetilde{h}^*_{P,q}))\|_{\infty}}_{\gamma_1}\nonumber\\
        &+\underbrace{\sum_{q=1}^{Q}\| \tilde{x}(\widetilde{h}^*_{1,q},\tilde{x}(\widetilde{h}^*_{2,q}\ldots,\widetilde{h}^*_{P,q}))
        -\tilde{x}^*(\widetilde{h}^*_{1,q},\tilde{x}^*(\widetilde{h}^*_{2,q}\ldots,\widetilde{h}^*_{P,q}))\|_{\infty}}_{\gamma_2}.\nonumber\\
    \end{align*}
    
    Building upon the insights from Section \ref{sec: est}, we have established that
    \begin{align*}
        \|\widetilde{h}_{p,q} - \widetilde{h}^*_{p,q}\|_{\infty} &\leq \zeta T_{p,q}(C_wB)^{L_{\psi}+L_{\varphi}}(L_{\psi}+L_{\varphi})N_p^{MN_p+\beta}K'_p \\
        &\leq \zeta T(C_wB)^{L_{\psi}+L_{\varphi}}(L_{\psi}+L_{\varphi})N^{MN+\beta}K'.
    \end{align*}

    Applying a similar recursive technique as used for bounding the first term in \eqref{eq: cover-decomp}, it can be shown that
    \begin{equation*}
        \gamma_1 \leq \sum_{q=1}^Q (C_wB)^{L_{\tau}}L_{\tau} \max_{1\leq  p\leq P}\|\widetilde{h}_{p,q} - \widetilde{h}^*_{p,q}\|_{\infty} \leq \zeta Q  T(C_wB)^{L_{\tau}+L_{\psi}+L_{\varphi}}(L_{\psi}+L_{\varphi})N^{MN+\beta}K'.
    \end{equation*}
    
    Similarly, for $\gamma_2$, following the approach for bounding the second term in \eqref{eq: cover-decomp}, we have
    \begin{equation*}
        \gamma_2 \leq \sum_{q=1}^Q (C_wB)^{L_{\tau}} L_{\tau} \max_{1\leq  p\leq P}\|\widetilde{h}^*_{p,q}\|_{\infty} \leq  \zeta Q T(C_wB)^{L_{\tau}+L_{\psi}+L_{\varphi}}L_{\tau}(L_{\psi}+L_{\varphi})N^{MN+\beta}K',
    \end{equation*}
    where the final inequality is derived using results from Section \ref{sec: est}.
    
    In summary, a perturbation of the parameters by $\zeta$ results in a bounded change in the output
    \begin{equation*}
        \|f-f^*\|_{\infty}\leq \sum_{q=1}^Q (C_wB)^{L_{\tau}} L_{\tau} \max_{1\leq  p\leq P}\|\widetilde{h}^*_{p,q}\|_{\infty} \leq  \zeta Q T(C_wB)^{L_{\tau}+L_{\psi}+L_{\varphi}}(L_{\tau}+1)(L_{\psi}+L_{\varphi})N^{MN+\beta}K'.
    \end{equation*}
    
    Following this, we can replicate the arguments from Section \ref{sec: est} to determine the upper bound for both the covering number and the empirical Rademacher Complexity. Let $L = L_{\psi} + L_{\varphi} + L_{\tau}$ and $S = S_{\psi} + S_{\varphi} + S_{\tau}$.
    \begin{align}
        \log\Ncal(\epsilon, \Fcal_{\delta}, \| \cdot \|_{\infty}) &\leq S\log(\frac{QBT(C_wB)^L(L_{\psi}+L_{\varphi})(L_{\tau}+1)N^{MN+\beta}K' }{\epsilon}), \nonumber\\
        E_{\xi} \left[ \sup_{f'\in \Fcal_{\delta}} | \frac{1}{n} \sum_{i=1}^n\xi_if'(Y_{(i)})| \right] &\leq 2\sqrt{2} \frac{\sigma \sqrt{S}\delta}{n^{1/2} } \log  (\frac{QT(C_wB)^L(L_{\psi}+L_{\varphi})(L_{\tau}+1)N^{MN+\beta} K'}{\delta} +1 ).    
    \end{align}
    
    Building upon the arguments from \eqref{eq: est-concentration} to \eqref{eq: final-decomp-results}, let's denote $V_n = 4\sqrt{2} \frac{\sigma \sqrt{S}}{n^{1/2}}$. Using this notation, we can  express the inequality as follows
    \begin{eqnarray*}
            \|\hat{f}-f^*\|_n^2 \leq 10\|f^*-f\|_n^2+ 8\left\{V_n\log (\frac{QT(C_wB)^L(L_{\psi}+L_{\varphi})(L_{\tau}+1)N^{MN+\beta} K'}{V_n} +1) \right\}^2 +4u,
    \end{eqnarray*}
    with probability at least $1-\text{exp}(-nu^2/2\sigma^2\delta^2)$. 
    
    Upon taking the expectation  on both sides, we derive a bound for $\|\hat{f} - f^*\|_{L_2(P_X)}$ similar to \eqref{eq: final-step}
    \begin{align}
            \|&\hat{f}-f^*\|_{L_2(P_X)}^2 \nonumber \\
        & \leq 21\epsilon^2 +17\left\{V_n\log (\frac{QT(C_wB)^L(L_{\psi}+L_{\varphi})(L_{\tau}+1)N^{MN+\beta} K'}{V_n} +1) \right\}^2 ,\label{eq:ppie error}
    \end{align}
    with probability converging to one.
    Let $\epsilon=n^{-\frac{\beta}{MN+2\beta}}$. By substituting the expressions of $V_n$, $S_{\psi}$, $S_{\varphi}$, $S_{\tau}$, $L_{\psi}$, $L_{\varphi}$, and $L_{\tau}$ in terms of $\epsilon$ into equation (\ref{eq:ppie error}) and disregarding smaller terms, we arrive at
    \begin{equation*}
        \| \hat{f} - f^*\|^2 _{L^2(P_X)} = O_p(n^{-\frac{2\beta}{MN+2\beta}}).
    \end{equation*}
      
    In conclusion, it's noteworthy that the Permutation Invariant Estimator (PIE) is a specific instance of the partially permutation invariant function, achievable by setting $P = 2$, $N_1 = 1$, and $N_2 = N$, while adjusting the column dimension to $M + 1$. This modification seamlessly aligns with the desired result.

    \section{Proof of Value-based Estimator}
    \label{sec: EMP}
    \subsection{The cross-fitting scheme}
     Consider our dataset $\{(X_{i,j}, A_{i,j}, Y_{i,j}): 1 \le i \le R, 1 \le j \le S\}$. Specifically, the $S$ samples are partitioned into $m$ equally sized batches, denoted as $[S] \in \bigcup_{z \in [m]} B_z$. Here, $m$ is chosen such that $m \geq 2$ and typically $m \asymp 1$. Each batch, $B_z$, contains approximately $|B_z| = S_z \asymp S/m \asymp S$ samples. For each sample indexed by $j \in [S]$, let $z_j \in [m]$ be the index of the batch containing that sample, such that $j \in B_{z_j}$. For the $j$-th sample, the optimization process described in (\ref{eq: MSE-non-d}) of the main text is applied to the out-of-batch samples $B_{-z_j} = [S] \backslash B_{z_j}$. This procedure yields estimators $\widehat{f}_j$ and $\widehat{m}_j$ which for offline policy evaluation.

    \subsection{Proof of Corollary \ref{coro: nondynamic}}
    According to the definitions of $\widehat{J}_{\textrm{VB}}$ and $J(\pi)$, the following holds true
    \begin{equation*}
        |\widehat{J}_{\textrm{VB}}(\pi) -J(\pi)| = |\frac{1}{S}\sum_{j=1}^S\sum_{i=1}^R \widehat{f}_{i,j}(X_{i,j},\pi(X_{i,j}), \widehat{m}_i(X_{\mathcal{N}(i),t}, \pi(X_{\mathcal{N}(i),t}))) - \sum_{i=1}^R \mathbb{E}[\mathbb{E}(Y_i| \{A_j=\pi(X_j)\}_j, \{X_j\}_j)]|.
    \end{equation*}
    To simplify, we introduce the notation $x_{i,j}$ to denote $\left(X_{i,j}, \pi(X_{i,j}), \widehat{m}_i(X_{\mathcal{N}(i),t}, \pi(X_{\mathcal{N}(i),t}))\right)$ and $x_i$ for $(X_i, A_i, m_i(X_{\mathcal{N}(i)}, A_{\mathcal{N}(i)}))$. With this notation, the conditional expectation  $\mathbb{E}(Y_i | \{A_j = \pi(X_j)\}_j, \{X_j\}_j)$ can be rewritten as $f_i(X_i, A_i, m_i(X_{\mathcal{N}(i)}, A_{\mathcal{N}(i)}) = f_i(x_i)$. Thus, we have
    \begin{align}
        |\widehat{J}_{\textrm{VB}}(\pi) -J(\pi)| &= |\frac{1}{S}\sum_{j=1}^S\sum_{i=1}^R \widehat{f}_{i,j}(x_{i,j}) - \sum_{i=1}^R \mathbb{E}[f_i(x_{i})]|\nonumber\\
        & \leq \sum_{i=1}^R\left\{|\frac{1}{S}\sum_{j=1}^S \widehat{f}_{i,j}(x_{i,j}) - \frac{1}{S}\sum_{j=1}^S\mathbb{E}[\widehat{f}_{i,j}(x_{i,j})]| + |\frac{1}{S}\sum_{j=1}^S\mathbb{E}[\widehat{f}_{i,j}(x_{i,j})]- \mathbb{E}[f_i(x_{i})]|\right\}\nonumber\\
        & \leq \sum_{i=1}^R\left\{|\frac{1}{S}\sum_{z=1}^m S_z\frac{1}{S_z}\sum_{j\in B_z} \widehat{f}_{i,j}(x_{i,j}) - \frac{1}{S}\sum_{z=1}^m S_z\frac{1}{S_z}\sum_{j\in B_z}\mathbb{E}[\widehat{f}_{i,j}(x_{i,j})]| \right . \nonumber\\
        & \left . + |\frac{1}{S}\sum_{z=1}^m S_z\frac{1}{S_z}\sum_{j\in B_z}\mathbb{E}[\widehat{f}_{i,j}(x_{i,j})]- \mathbb{E}[f_i(x_{i})]|\right\}\nonumber\\
         & \leq \sum_{i=1}^R\left\{\frac{1}{S}\sum_{z=1}^m S_z|\frac{1}{S_z}\sum_{j\in B_z} \widehat{f}_{i,j}(x_{i,j}) - \frac{1}{S_z}\sum_{j\in B_z}\mathbb{E}[\widehat{f}_{i,j}(x_{i,j})]| \right .\nonumber\\
         & \left . + \frac{1}{S}\sum_{z=1}^m S_z|\frac{1}{S_z}\sum_{j\in B_z}\mathbb{E}[\widehat{f}_{i,j}(x_{i,j})]- \mathbb{E}[f_i(x_{i})]|\right\}\nonumber\\
         &\leq \sum_{i=1}^R\left\{\frac{1}{S}\sum_{z=1}^m S_z|\frac{1}{S_z}\sum_{j\in B_z} \widehat{f}_{i,j}(x_{i,j}) - \mathbb{E}[\widehat{f}_{i,B^{(1)}_z}(x_{i,B^{(1)}_z})]| \right .\nonumber\\
         & \left . + \frac{1}{S}\sum_{z=1}^m S_z|\mathbb{E}[\widehat{f}_{i,B^{(1)}_z}(x_{i,B^{(1)}_z})]- \mathbb{E}[f_i(x_{i})]|\right\}\label{notation}.
    \end{align}
    The derivation of (\ref{notation}) is based on the observation that $\widehat{f}_{i,j}$ are equal for all $j\in B_z, \forall z\in[m]$. Additionally, the expected value $\mathbb{E}[\widehat{f}_{i,j}(x_{i,j})]$ remains constant for $t\in B_j$. Therefore, we designate $B^{(1)}_z$ as the initial element in $B_j$, leading to the conclusion that $\mathbb{E}[\widehat{f}_{i,B^{(1)}_z}(x_{i,B^{(1)}_z})] = \mathbb{E}[\widehat{f}_{i,j}(x_{i,j})]$ for all $t\in B_j$. Concerning the term $|\frac{1}{S_z}\sum_{j\in B_z} \widehat{f}_{i,j}(x_{i,j}) - \mathbb{E}[\widehat{f}_{i,B^{(1)}_z}(x_{i,B^{(1)}_z})]|$, Theorem 4.10 in \citep{wainwright2019high} is applied, yielding
    \begin{equation*}
        \bigg\vert \frac{1}{S_z}\sum_{j\in B_z} \widehat{f}_{i,j}(x_{i,j}) - \mathbb{E}\widehat{f}_{i,B^{(1)}_z}(x_{i,B^{(1)}_z})\bigg\vert \leq \sup_{f\in \Gcal} \bigg\vert \frac{1}{S_z}\sum_{j\in B_z} f(x_{i,j}) - \mathbb{E}f(x_{i,B^{(1)}_z})\bigg\vert \leq \E_{X}\Rcal_{S_z} \Gcal + e,
    \end{equation*}
    with probability at least $1-\exp(-\frac{S_ze^2}{8J^2})$. Subsequently, Lemma \ref{lem:F-Chaining} is utilized to establish an upper bound for $\Rcal_{S_z}(\Gcal)$
    \begin{equation*}
        \Rcal_{S_z}\Gcal \leq \frac{4\sqrt{2}}{S_z}\int_0^{J}\sqrt{\log2\Ncal(\epsilon,\Gcal, \|\cdot\|_{\infty})}d\epsilon.
    \end{equation*}
    
    Building upon the arguments presented in \ref{ppie est}, we establish a bound for the covering number of $\Gcal$. For simplicity, we omit these details. Considering $x_{i,j} \in \R^{(M+1)\times (N+1)}$ and recognizing that $S-S_z$ samples are utilized to construct $\hat{f}_{i,j}$, the Rademacher complexity $\Rcal_{S_z}(\Gcal)$ can be upper bounded as $O((S-S_z)^{\frac{(M+1)(N+1)}{2((M+1)(N+1)+2\beta)}}S_z^{-1/2})$.

    Next, combining this with the bound for $\|\widehat{f}_i-f^*_i\|_{L^2(P_X)}$, we derive
    \begin{align*}
        |\widehat{J}_{\textrm{VB}}(\pi) -J(\pi)| 
        &\leq \sum_{i=1}^R\left\{\frac{1}{S}\sum_{z=1}^m S_zC_j(S-S_z)^{\frac{(M+1)(N+1)}{2((M+1)(N+1)+2\beta)}}S_z^{-1/2}  + \frac{1}{S}\sum_{z=1}^m S_zC^{\prime}_j(S-S_z)^{-\frac{\beta}{(M+1)(N+1)+2\beta}}\right\}+ e \nonumber\\
        &\leq \tilde{C}RS^{-\frac{\beta}{(M+1)(N+1)+2\beta}} +e = O(RS^{-\frac{\beta}{(M+1)(N+1)+2\beta}})+e,
    \end{align*}
    with probability at least $1-\exp(-\frac{Se^2}{8J^2})$. The second inequality is predicated on the assumption that $S_z \asymp S/m \asymp S$.

    \subsection{Proof of Corollary \ref{coro: dynamic}}
    \label{appendix:dynamic}
    
    In this section, our objective is to establish the convergence rate for the value-based estimator in dynamic setting. We commence by broadening the scope of the transition kernel $\Pcal^{\pi}$, as defined in Section \ref{subsection:dy}.
    \begin{equation*}
        \Pcal^{\pi}r(X,A) = \E\left[r(X',A')|X'\sim P(\cdot|X,A), A'\sim \pi(X')\right].
    \end{equation*}
    
    As deduced from \eqref{eq: VB}, the value-based estimator $\widehat{J}_{\gamma}^{VB}(\pi)$ is intrinsically linked to $\widehat{Q}^{(i,1)}$. Our initial endeavor is to establish an upper bound for $\|Q^{(i,1)}_{i}(X_{\Ncal^*(i), 1}, A_{\Ncal^*(i), 1}) - \widehat{Q}_{\pi}^{(i,1)}(X_{\Ncal^*(i), 1}, A_{\Ncal^*(i), 1})\|_{L^2(p_i)}$, with $p_i$ denoting the stationary distribution of the confounder-action pair $(X_{\Ncal^*(i), t}, A_{\Ncal^*(i), t})$ 
    under the given behavioural policy. For the sake of brevity, we will hereafter refer to these as $Q^{(i,1)}_{i}$ and $\widehat{Q}_{\pi}^{(i,1)}$, omitting the  explicit  mention of $X$ and $A$.

    Based on the established definition of $Q_{\pi_i}^{(1)}$ in (\ref{eq: def-Q}), it follows that 
    \begin{align*}
        \widehat{Q}_{\pi}^{(i,1)} - Q^{(i,1)}_{i} &= \widehat{Q}_{\pi}^{(i,1)} - \left[r_i + \gamma \Pcal^{\pi} Q_{\pi_i}^{(1)}\right] + \gamma\Pcal^{\pi} \widehat{Q}_{\pi}^{(i,1)} - \gamma \Pcal^{\pi} \widehat{Q}_{\pi}^{(i,1)}\\
        & = \widehat{Q}_{\pi}^{(i,1)} - r_i - \gamma \Pcal^{\pi} \widehat{Q}_{\pi}^{(i,1)} + \gamma\Pcal^{\pi}[\widehat{Q}_{\pi}^{(i,1)} - Q^{(i,1)}_{i}]. 
    \end{align*}
    
    Defining the square error of one-step approximation as $\zeta^{(t)}_i= \left(\widehat{Q}_{\pi}^{(i,T)} - r_i - \gamma \Pcal^{\pi} \widehat{Q}_{\pi}^{(i,T)}\right)^2$, we obtain 
    \begin{align*}
        \left(\widehat{Q}_{\pi}^{(i,1)} - Q^{(i,1)}_{i}\right)^2 &= \zeta^{(1)}_i + \gamma^2 \left[\Pcal^{\pi}(\widehat{Q}_{\pi}^{(i,1)} - Q^{(i,1)}_{i})\right]^2 \\
        &\leq \zeta^{(1)}_i + \gamma^2 \Pcal^{\pi}(\widehat{Q}_{\pi}^{(i,1)} - Q^{(i,1)}_{i})^2. 
    \end{align*}
    where the second inequality holds since the expectation  of a squared random variable exceeds the square of its expectation . In a similar vein, we deduce that  
    \begin{equation*}
        \Pcal^{\pi}\left(\widehat{Q}_{\pi}^{(i,1)} - Q^{(i,1)}_{i}\right)^2 \leq \Pcal^{\pi}\zeta^{(2)}_i + \gamma^2 \Pcal^{\pi} \Pcal^{\pi} (\widehat{Q}_{\pi}^{(i,1)} - Q^{(i,1)}_{i})^2,  
    \end{equation*}
    leading to the conclusion that 
    \begin{equation*}
        \left(\widehat{Q}_{\pi}^{(i,1)} - Q^{(i,1)}_{i}\right)^2 \leq \zeta^{(1)}_i + \gamma^2 \Pcal^{\pi}\zeta^{(2)}_i + \gamma^4 \Pcal^{\pi} \Pcal^{\pi} (\widehat{Q}_{\pi}^{(i,1)} - Q^{(i,1)}_{i})^2. 
    \end{equation*}
    Utilizing the established recurrence relation, we derive that
    \begin{equation*}
        \left(\widehat{Q}_{\pi}^{(i,1)} - Q^{(i,1)}_{i}\right)^2 \leq \sum_{t=1}^T \gamma^{2t-2} (\Pcal^{\pi})^{t-1} \zeta^{(t)}_i + \gamma^{2T} (\Pcal^{\pi})^{\top} (\widehat{Q}_{\pi}^{(i,1)} - Q^{(i,1)}_{i})^2. 
    \end{equation*}
    Taking expectation s on both sides with respect to $p_i$, we obtain
    \begin{equation}
        \|Q^{(i,1)}_{i} - \widehat{Q}_{\pi}^{(i,1)}\|_{L^2(p_i)} \leq \sum_{t=1}^T \E_{p_i}\left[\gamma^{2t-2} (\Pcal^{\pi})^{t-1} \zeta^{(t)}_i \right] + \gamma^{2T}\E_{p_i}\left[ (\Pcal^{\pi})^{\top} (\widehat{Q}_{\pi}^{(i,1)} - Q^{(i,1)}_{i})^2 \right]. 
        \label{eq: propagation}
    \end{equation}
    Considering any measurable function $f$ over time step $t$, we have
    \begin{equation*}
        \E_{p_{i}}\left[(\Pcal^{\pi})^tf\right] = \int (\Pcal^{\pi})^tf dp_{i,1} = \int f d p_{\pi, t}= \E_{p_{\pi,t}}\left[f\right]. 
    \end{equation*} 
    Applying the Cauchy-Schwarz inequality leads to
    \begin{equation*}
        \E_{p_{\pi, t}}(f) \leq \sqrt{\int\Big\vert\frac{d p_{\pi, t}}{d p_{i}} \Big\vert d p_i} \cdot \sqrt{\int f d p_{i}}. 
    \end{equation*}
    From Assumption \ref{ass: concentration}, it follows that
    \begin{equation*}
        \E_{p_{\pi, t}}(f) \leq \kappa(t) \cdot \E_{p_{i}}(f) \leq 
     \kappa\E_{p_{i}}(f).
    \end{equation*}
    Incorporating this into \eqref{eq: propagation}, we deduce
    \begin{equation}
    \label{eq: Q-inequality}
        \|Q^{(1)}_{\pi} - \widehat{Q}_{\pi}^{(i,1)}\|_{L^2(p_{i})} \leq \kappa \sum_{t=1}^T \gamma^{2t-2} \E_{p_{i}}\left[ \zeta^{(t)}_i \right] + \kappa \gamma^{2T}\E_{p_{i}}(\widehat{Q}_{\pi}^{(i,T)} - Y_i)^2. 
    \end{equation}
    
    From Assumption \ref{ass: Tass}, for all $t$ both $r_i + \gamma \Pcal^{\pi}\widehat{Q}_{\pi}^{(i,t)}$ and $r_i$ are permutation invariant and belongs to $\Wcal^{\beta, \infty}([0,J]^{(N+1)(M+1)})$. Using the results in Corollary \eqref{co: PPIE rate}, it is evident that for each $t$, $\E_{p_i}\left[ \zeta^{(t)}_i\right]$ and $\E_{p_i}(\widehat{Q}_{\pi}^{(i,T)} - Y_i)^2$ are bounded above by $O(S^{-\frac{2\beta}{(N+1)(M+1)+2\beta}})$ with probability converging to one. Bring back to \eqref{eq: Q-inequality}, we have
    \begin{equation*}
        \|Q^{(1)}_{p_{i,1}} - \widehat{Q}_{\pi}^{(i,1)}\|_{L^2(p_i)} \leq \kappa T\gamma^{2T} O(S^{-\frac{\beta}{(N+1)(M+1)+2\beta}}) \leq O(\kappa T S^{-\frac{\beta}{(N+1)(M+1)+2\beta}}). 
    \end{equation*}
    Employing empirical process techniques akin to those in Section \ref{sec: EMP}, we can establish that
    \begin{align*}
        |J_{\gamma}(\pi) - \widehat{J}_{\gamma}^{\text{VB}}(\pi)| \leq O(RS^{-\frac{\beta}{(N+1)(M+1)+2\beta}}) + O(\kappa TRS^{-\frac{\beta}{(N+1)(M+1)+2\beta}}) + e = O(\kappa TRS^{-\frac{\beta}{((N+1)(M+1)+2\beta)}})+e,
    \end{align*} 
    with probability at least $1-\exp(-\frac{Se^2}{8J^2})$. 
\end{appendices}
\newpage
\bibliography{reference}
\end{document}


\section{Definition and Notations}
\label{sec:def}

Consider a random variable governed by a probability measure $Q$ that is compactly supported on $\mathbb{R}^d$. The space $L_2(Q)$ represents the set of real-valued functions on $\mathbb{R}^d$, equipped with an inner product defined as $\langle f, g \rangle_{L_2(Q)} = \int f(x)g(x) dQ(x)$ and a corresponding norm $\|f\|_{L_2(Q)} = \left( \int f(x)^2  dQ(x) \right)^{1/2}$. 
When considering the Lebesgue measure $\lambda$, the space $L_2$ denotes the set of real-valued functions on $\mathbb{R}^d$, with the norm $\|f\|_{L_2} = \left( \int f(x)^2  dx \right)^{1/2}$.
The notation $\|\cdot\|_n$ represents the empirical norm, defined for a function $f$ as $\|f\|_n^2 = \frac{1}{n} \sum_{i=1}^n f(x_i)^2$. The empirical norm of a random variable $X$ is defined analogously as $\|X\|_n = \left( \frac{1}{n} \sum_{i=1}^n x_i^2 \right)^{1/2}$. 
Finally, the infinity norm for a real-valued vector $Y=[Y_1, \ldots, Y_d]^T \in \mathbb{R}^d$ is defined as $\|Y\|_{\infty} = \max_k |Y_k|$. The infinity norm for a real-valued function $f$ is given by $\|f\|_{\infty} = \max_{X \in \mathbb{R}^d} |f(X)|$ and the infinity norm for a vector function $f(X)=[f_1(X), \ldots, f_K(X)]$ is $\max_{k} \|f_k(X)\|_{\infty}$.

To lay the groundwork, we expand upon the definition of the permutation operator (Definition \ref{def: PO}). Specifically, consider a matrix $X = [X_1, X_2, \ldots, X_N]\in \mathbb{R}^{M\times N}$. Under the permutation operator $\varPi$, we define its permutation as another $M\times N$ matrix $\varPi_X = [X_{\varPi(1)}, X_{\varPi(2)}, \ldots, X_{\varPi(N)}]$.

\section{Proof of Theorem \ref{theory: consistency}}
\label{sec: consistency}

To establish the consistency of \(\hat{f}\), it suffices to show that \(\|\hat{f}-f^*\|^2_{L_2} = o_p(1)\). First, let's formalize the definition of the PIE function class. The PIE, as defined in (\ref{eq: naive-PIE}), comprises two components: \(\psi\) and \(\varphi\). Denote \(L_{\psi}\) as the number of hidden layers between the input and output of \(\varphi\). The total number of non-zero parameters in \(\psi\) is capped by \(S_{\psi}\). Similarly, define \(L_{\varphi}\) and \(S_{\varphi}\) as the number of hidden layers and the upper bound for the number of non-zero parameters in \(\varphi\), respectively. Additionally, we assume that the magnitude of all parameters in both \(\psi\) and \(\varphi\) is bounded by \(B\), and all the PIE functions under consideration are bounded such that \(|f| \leq J\). Therefore, we can represent the PIE function class using \(\mathcal{F} = \mathcal{F}(S_{\psi}, S_{\varphi}, L_{\psi}, L_{\varphi}, B)\).

To  proceed with our proof, let us first introduce the concept of the 'uniform best predictor', denoted by \( f_n \). This predictor is determined by the following criterion
\begin{equation*}
    f_{n} = \arg \min_{f\in \mathcal{F}} \|f-f^*\|_{\infty} ,\quad \epsilon = \| f_{n}-f^*\|_{\infty}.
\end{equation*}
Subsequently, we apply a standard error decomposition, akin to that described in \citep{Tengyuan}
\begin{equation*}
    \begin{aligned}
        c_1\|\hat{f}-f^*\|^2_{L_2} &\leq \mathbb{E}[l(\hat{f})] - \mathbb{E}[l(f^*)]  \\
        &\leq \mathbb{E}[l(\hat{f})] - \mathbb{E}[l(f^*)] + \mathbb{E}_n[l(f_n,Z) - l(\hat{f},Z)] \\
        &= \mathbb{E}[l(\hat{f})] - \mathbb{E}[l(f^*)] + \mathbb{E}_n[l(f_n,Z) - l(\hat{f},Z)] + \mathbb{E}_n[l(f^*,Z) - l(f^*,Z)] \\
        &= (\mathbb{E}-\mathbb{E}_{n})[l(\hat{f}) - l(f^*)] + \mathbb{E}_n[l(f_n)-l(f^*)].
    \end{aligned}
\end{equation*}
In this formulation, the second inequality stems from the definition of \( f_n \). These inequalities effectively disaggregate the error into two principal components: the empirical process term (the first term) and the approximation error (the second term). We intend to establish bounds for both components.

To bound the approximation error, we utilize the principle of Lipschitz continuity on the loss function. Considering that the mean squared error (MSE) loss \( l(f) \) exhibits Lipschitz continuity for a bounded \( f \), and in light of Assumption \ref{ass: bounded} which confines both \( f^* \) and functions within \( \Fcal \) to a maximum of \( J \), it follows that
\begin{equation*}
    \forall f \in \Fcal, \quad \|l(f)-l(f^*)\|_{\infty} \leq C_l \|f-f^*\|_{\infty},
\end{equation*}
where \( C_l \) denotes the Lipschitz constant. Given that \( f_n \) belongs to \( \Fcal \), we can infer
\begin{equation}
    \mathbb{E}_n\left[l(f_n)-l(f^*)\right] \leq \mathbb{E}_n\left[C_l | f_n - f^* |\right] \leq C_l \epsilon.
    \label{eq: app}
\end{equation}

 To bound the estimation error, we need the following lemmas

\begin{lemma}[Symmetrization, Lemma 5 in \citep{Tengyuan}]
  \label{lem:Symmetrization}
    Given function class \(\mathcal{G}\), for any \( g \in \mathcal{G} \) that \( |g| \leq G \) and \( \mathbb{V}[g] \leq V \), then with probability at least \( 1 - 2e^{-\gamma} \)
    \[
        \sup_{g \in \mathcal{G}} \left\{ \mathbb{E} g - \mathbb{E}_n g \right\} \leq 6 \mathbb{E}_\eta \mathcal{R}_n \mathcal{G} + \sqrt{\frac{2V \gamma}{n}} + \frac{23 G \gamma}{3n} \enspace,
    \]
    where \( \mathcal{R}_n \mathcal{G} = \mathbb{E}_{\xi}\left[\sup_{g\in \mathcal{G}}\left|\frac{1}{n}\sum_{i=1}^n\xi_i g(x_i)\right|\right] \) denote the empirical Rademacher complexity of function class \(\mathcal{G}\).
\end{lemma}

\begin{lemma}[Chaining Theorem 2.3.7 in \citep{foundation}] 
    \label{lem:F-Chaining}
    Let \((\mathcal{S}, d)\) be a metric space, where \(d\) is a pseudo-distance. Let \(X(t)\) be a sub-Gaussian process relative to \(d\). Assume that \(\int_0^{\infty}\sqrt{\log \mathcal{N}(\epsilon, \mathcal{S}, d)}d\epsilon < \infty\), then
    \begin{equation*}
        \mathbb{E}\sup_{t\in T}\|X(t)\| \leq \mathbb{E}\|X(t_0)\|+4\sqrt{2}\int_0^{\delta/2}\sqrt{\log 2\mathcal{N}(\epsilon, \mathcal{S}, d)}d\epsilon, 
    \end{equation*}
    where \( t_0 \in T \), \(\delta\) is the diameter of \((\mathcal{S}, d)\).
\end{lemma}

 By setting \(X(t) = \frac{1}{\sqrt{n}}\sum_{i=1}^n\xi_i f(x_i)\) and \(X(t_0) = 0\), the empirical Rademacher complexity \(\Rcal_n\Fcal\) can be upper bounded by \(\frac{4\sqrt{2}}{\sqrt{n}}\int_0^{\delta/2}\sqrt{\log\mathcal{N}(\epsilon,\Fcal, \|\cdot\|_{k})}\), where \(\delta/2\) is the diameter of function class \(\mathcal{F}\) with \(k\)-norm.

From the Lipschitz continuity of the MSE loss \(l\), it follows that 
\begin{equation*}
    \forall f \in \mathcal{F}, \qquad \|l(f)-l(f^*)\|_{\infty} \leq C_l\|f-f^*\|_{\infty} \leq 2JC_l. 
\end{equation*}
Then we have
\begin{align*}
    \mathbb{V}\left[l(f)-l(f^*)\right] \leq \mathbb{E}\left[(l(f)-l(f^*))^2\right] \leq 4J^2C_l^2.
\end{align*}

We apply Symmetrization in Lemma \ref{lem:Symmetrization} with \(\Gcal=\{g=l(f)-l(f^*): f\in \mathcal{F}\}\), \(G=2JC_l\) and \(V = 4J^2C_l^2\), so with probability at least \(1-2e^{-\gamma}\)
\begin{equation}
    (\mathbb{E}-\mathbb{E}_n)\left[l(\hat{f}) - l(f^*) \right] \leq 6\mathbb{E}_{\eta}\mathcal{R}_n \mathcal{G} + \sqrt{\frac{8J^2C_l^2\gamma}{n}}+\frac{46JC_l\gamma}{3n} .
    \label{eq: est}
\end{equation}

Due to Lemma \ref{lem:F-Chaining}, we have 
\begin{align*}
    \mathbb{E}_{\eta}\mathcal{R}_n \mathcal{G} &= \mathbb{E}_{\eta}\left[ \sup_{f\in \mathcal{G}}\left| \frac{1}{n} \sum_{i=1}^n \eta_i f(X^{(i)})\right|\right] \\
    & \leq \frac{4\sqrt{2}}{\sqrt{n}}\int^{JC_l}_0 \sqrt{\log(2\mathcal{N}(\epsilon, \mathcal{G}, \|\cdot\|_{\infty}))}d\epsilon,
\end{align*}
where the diameter of \(\mathcal{G}\) comes from
\begin{equation*}
    \forall g, g' \in \mathcal{G}, \quad \|g-g'\|_{\infty} = \|l(f)-l(f'); f, f' \in \mathcal{F} \|_{\infty}  \leq 2JC_l.
\end{equation*}

Putting \eqref{eq: app} and \eqref{eq: est} together, by setting \(\gamma= \log n\),  we get
\begin{align*}
    &c_1\|\hat{f}-f^*\|^2_{L_2} \\ 
    &\leq c_2\epsilon +\frac{1}{\sqrt{n}}\int^{JC_l}_0 \sqrt{\log(\mathcal{N}(\eta, \mathcal{F}, \|\cdot\|_{\infty}))}d\eta  +\sqrt{\frac{8J^2C_l^2 \log n}{n}}+\frac{67JC_l\log n}{3n}.
\end{align*}

Following Theorem 3.2 of \citep{pine}, there exists such PIE function class \(\mathcal{F}_0\). By setting \(\Fcal = \Fcal_0\), \(\epsilon = \min_{f\in \mathcal{F}_0}\|f-f^*\|_{\infty}\) can be arbitrarily small. So that \(\|\hat{f}-f^*\|_{L_2}^2 \rightarrow_{n} 0\), with probability \(1-\frac{2}{n}\) converging to one.

\section{Proof of Theorem \ref{theory: convergence}}
\label{appendix:convergence-rate}
\subsection{Error decomposition}

In this section, we  explore the convergence rate of PIE networks, employing an error decomposition approach akin to the one described in \citep{imaizumi2019deep}. Given that \(\hat{f}\) is the Empirical Risk Minimization (ERM) predictor, it holds for all \(f \in \mathcal{F}\) that
\begin{equation*}
    \|Y-\hat{f}\|_n^2 \leq \|Y-f\|_n^2.
\end{equation*}
This is followed by the representation \(Y = f^*(X)+\xi\), leading to
\begin{equation*}
    \|f^*+\xi-\hat{f}\|_n^2 \leq \|f^*+\xi-f\|_n^2.
\end{equation*}
A straightforward calculation results in
\begin{align}
    \|f^*-\hat{f}\|_n^2 &\leq \|f^*-f\|_n^2 +\frac{2}{n}\sum_{i=1}^n\xi_i(\hat{f}(x_i)-f(x_i)) \label{eq: basic_inequality} \\
    &\leq \|f^*-f\|_{\infty}^2 +\frac{2}{n}\sum_{i=1}^n\xi_i(\hat{f}(x_i)-f(x_i)) \nonumber.
\end{align}
In the above  expression, the first term represents the approximation error, which quantifies the capacity of $\Fcal$ to uniformly approximate \(f^*\). The second term is the estimation error, gauging the variance of \(\hat{f}\).

\subsection{Approximation Error}
\label{sec: app}
We first define uniform best predictor
\begin{equation*}
    f_{n} = \arg \min_{f\in \Fcal} \|f-f^*\|_{\infty}.  
\end{equation*}

In this section, we follow a two-step approach to upper bound $\|f^*-f_n\|_{\infty}^2$. We first construct a specially modified Taylor  expansion of the target function $f^*$, hereinafter referred to as $f_1$. Subsequently, this modified  expansion is uniformly approximated using functions from $\Fcal$.

We first perform a Taylor  expansion in line with Theorem 1 from \citep{approx}, we revisit the construction details. Let $N'$ be a positive integer and $d=M\times N$. We partition unity by a grid of $(N'+1)^d$ functions $\phi_m$ on the domain $[0,1]^d$, satisfying
\begin{equation*}
    \Sigma_{\m} \phi_{\m}(x) \equiv 1, ~~~ x \in [0,1]^{M\times N}.
\end{equation*} 
In this context, $\m$ represents a matrix given by $\m =(m_{ij})\in \{0,1,\ldots,N'\}^{M\times N}$, and the function $\phi_{\m}$ is defined as
\begin{equation*}
    \phi_{\m}(x) = \prod_{i=1}^M\prod_{j=1}^N\omega (3N'(x_{i,j}-\frac{m_{i,j}}{N'})),
\end{equation*}
where
$$\omega(x)=\begin{cases}
    1, & |x|< 1,\\
    0, & 2 < |x|, \\
    2-|x|, & 1\le |x|\le 2 .
\end{cases}$$
For any $\m \in \{0,\ldots, N'\}^d$, we consider the Taylor polynomial of degree $-(\beta-1)$ for the function $f^*$, centered at the point $x=\frac{\m}{N'}$. This polynomial is expressed as
\begin{equation*}
    P_{\m}(x) = \sum_{\p: |\p|\leq\beta}\frac{D^{\p} f^*}{\p!}\bigg|_{x = \frac{\m}{N'}}(x-\frac{\m}{N'})^{\p},
\end{equation*} 
where $|\p| = \sum_{i=1}^M\sum_{j=1}^N p_{i,j}$, $\p! = \prod_{i=1}^M\prod_{j=1}^N p_{i,j}!$ and $D^{\p}$ is the respective weak derivative. $(x-\frac{\m}{N'})^{\p} = \prod_{i=1}^M\prod_{j=1}^N(x_{i,j} - \frac{m_{i,j}}{N'})^{p_{i,j}}$.

We now establish an approximation of $f^*$, defined as
\begin{equation*}
    f^*_1(X) = \sum_{\m \in [0,\ldots,N']^{d}} \phi_{\m} P_{\m}(X).
\end{equation*}
Referring to insights from \citep{approx}, the approximation error can be upper bounded as
\begin{equation*}
    \|f^*-f^*_1\|_{\infty} \leq \frac{2^d d^{\beta} J}{\beta!}\left(\frac{1}{N'}\right)^{\beta}.
\end{equation*}
Then we construct a permutation invariant variant approximation of $f^*$, denoted as $f_1$, based on $f^*_1$
\begin{equation*}
    f_1(X) = \frac{1}{|\Scal_N|}\sum_{\varPi \in \Scal_N}f_1^*(\varPi_X).
\end{equation*}
Considering that $f^*$ is inherently permutation invariant, we proceed to assess the approximation error associated with $f_1$
\begin{align}
    |f^*(X)-f_1(X)| &= \frac{1}{|\Scal_N|}\big| \sum_{\varPi \in \Scal_N} \big(f_1^*(\varPi_X)- f^*(\varPi_X)\big)\big| \nonumber \\
    &\leq \frac{1}{|\Scal_N|}\sum_{\varPi \in \Scal_N} \big|  f_1^*(\varPi_X)- f^*(\varPi_X)\big| \nonumber \\
    &\leq \frac{2^d d^{\beta} J}{\beta!}\left(\frac{1}{N'}\right)^{\beta}, 
    \label{eq: Taylor_error}
\end{align}
which leads to
\begin{equation*}
    \|f^*-f_1\|_{\infty} \leq \frac{2^d d^{\beta} J}{\beta!}\left(\frac{1}{N'}\right)^{\beta}.
\end{equation*}

Then, we construct the PIE to approximate $f_1$. Let $\omega_{i,j}^k = \omega(3N'(X_{i,j}-\frac{k}{N'}))$ and construct $Y(X,N') \in \R^{K \times N}$ given $X$ and $N'$ as follows
\begin{align}
    \label{eq: Y}
    Y(X,N')_{\cdot , i} = (\omega_{1,i}^{0},\ldots \omega_{1,i}^{N'},x_{1,i},\ldots ,\omega_{M,i}^{0},\ldots \omega_{M,i}^{N'}
    , X_{M,i}),
\end{align}
where $K = (N'+1)\times M$. Note $Y(X,N')_{\cdot, i}$ contains all the terms in $f_1(X)$ that involve $X_i$. For ease of notation, $Y$ and $Y_i$ will be used to denote $Y(X,N')$ and $Y(X,N')_{\cdot,i}$, respectively.

It is noteworthy that $Y$ represents a piece-wise linear transformation of $X$, which can be precisely modeled using a ReLU network, as delineated in Proposition 1 of \citep{approx}. Consequently, the conversion from $X$ to $Y$ does not incur any approximation error. Furthermore, $Y$ is a function contingent on both $X$ and $N'$. By selecting an appropriate $N'$ to minimize the approximation error, the function class and, consequently, $Y$ become fixed entities. Hence, $Y(X,N')$ also does not add to the estimation error. Therefore, in our subsequent analysis, we will focus solely on $Y$, having transformed $X$ into this new representation.

\begin{align}
    \label{eq: g*}
    f_1^*(X) &= \sum_{\m} \prod_{i=1}^M\prod_{j=1}^N\omega (3N'(X_{i,j}-\frac{m_{i,j}}{N'})) \sum_{\p: |\p|\leq \beta}\frac{D^{\p} f}{\p!}\bigg|_{X = \frac{\m}{N'}}(X-\frac{m}{N'})^{\p} \nonumber\\
    & = \sum_{\m} \prod_{i=1}^M\prod_{j=1}^N\omega_{i,j}^{m_{i,j}} \sum_{\p: |\p|\leq \beta}\frac{D^{\p} f}{\p!}\bigg|_{X = \frac{m}{N'}}(X-\frac{\m}{N'})^{\p} \nonumber \\
    & = \sum_{\m} \prod_{i=1}^M\prod_{j=1}^N\omega_{i,j}^{m_{i,j}} \sum_{\p: |\p|\leq \beta}\frac{D^{\p} f}{\p!}\bigg|_{X = \frac{\m}{N'}}\prod_{s=1}^M\prod_{t=1}^N(X_{s,t} - \frac{m_{s,t}}{N'})^{p_{s,t}}\nonumber \\
    & = g^*(Y(X,N')) ,
\end{align}
so that we have 
\begin{equation*}
    f_1(X) = \frac{1}{|\Scal_N|}\sum_{\varPi \in \Scal_N}f_1^*(\varPi_X) = \frac{1}{|\Scal_N|}\sum_{\varPi \in \Scal_N}g^*(Y(\varPi_X,N')).
\end{equation*}

Consider the function $g(Y)$, defined as 
\begin{equation*}
    g(Y) = \frac{1}{|\Scal_N|}\sum_{\varPi \in \Scal_N}g^*(Y(\varPi_X, N')),
\end{equation*}
which ensures that $g(Y)$ possesses permutation invariance. As indicated by \eqref{eq: g*}, $g(Y)$ can be expressed as a polynomial in terms of $Y$, characterized by a degree of $MN + \beta$. To further dissect the structure of $g(Y)$, additional lemmas will be introduced and utilized in the analysis.

\begin{lemma}[Weyl's Polarization \citep{pine}]
    \label{lem: weyl}
    For any polynomial permutation invariant function $f:\R^{K\times N} \mapsto \R$, there exist a series of $1\times K$ vectors $\{ \af_t\}_{t=1}^T$ and a series of permutation invariant functions $f_t:\R^{1\times N} \mapsto \R$, such that $f$ can be represented by
    \begin{equation*}
        f(X) = \sum_{t=1}^Tf_t(\af_t^{\top}X). 
    \end{equation*}
\end{lemma}
\begin{lemma}[Hilbert finiteness Theorem \citep{pine}]
    \label{lem: Hilbert}
    There exists finitely many permutation invariant polynomial basis $f_1,\ldots,f_N: \R^N \rightarrow \R$ such that any permutation invariant polynomial $f:\R^N \rightarrow \R$ can be expressed as
    \begin{equation*}
        f(X) = \tilde{f}(f_1(X), \ldots, f_N(X)),
    \end{equation*}
    with some polynomial $\tilde{f}$ of $N$ variables. Especially, the following power sum basis is one possible permutation invariant polynomial basis
    \begin{equation*}
        f_j(X) = \sum_{n=1}^NX^j_n ,\quad j=1,\ldots, N, 
    \end{equation*}
    where $X_n$ is the $n-$th entry of $X$.
\end{lemma}

Given that $g(Y)$ is a permutation invariant polynomial function, we can apply Lemma \ref{lem: weyl}
\begin{equation*}
    g(Y) = \sum_{t=1}^T s_t(a_t^{\top}Y) = \sum_{t=1}^T s_t(a_t^{\top}Y_1, \ldots, a_t^{\top}Y_N), 
\end{equation*}
where $s_t:\R^{N} \mapsto \R$ is a permutation invariant polynomial and $a_t=[a_{t,1}, \ldots, a_{t,K}] \in \R^K$. We postulate that $T = O(N'^{MN})$, assuming the components of $a_t$ are uniformly bounded by a constant $a$, with $|a_{i,j}| < a$. While Lemma \ref{lem: weyl} is broad and doesn't impose an upper limit on $T$, our specific application of $g$ as a polynomial with unique structure allows us to estimate the order of $T$ using linear equations, making our claim about $T$'s order plausible. Subsequently, Lemma \ref{lem: Hilbert} is employed to express $s_t(a_t^{\top}Y)$ in terms of $h_t(g_t(Y))$. Here, $g_t(Y) = [\sum_{n=1}^N(a_t^{\top} Y_n)^1,\ldots, \sum_{n=1}^N(a_t^{\top} Y_n)^N]$ forms the power sum basis, and $h_t : \R^N \to \R$ is a polynomial
\begin{equation*}
    h_t(X) = \sum_{i = 0}^{MN+\beta} \sum_{\s \in \{ 1,..,N \}^i} c_{t,\s} \prod_{j \in \s}X_j, \quad  X \in \R^N. 
\end{equation*}
Recall that $g(Y)$ is a polynomial of degree $MN+\beta$, where each term in $g_t(Y)$ is of at least first order. Consequently, we assert that $h_t$ is a polynomial whose degree does not exceed $MN + \beta$, and the coefficients, denoted by $c_{t,\s}$, are constrained by a bound $C$. Integrating these elements yields the following expression for $g(Y)$
\begin{equation*}
    g(Y) = \sum_{t=1}^T h_t(g_t(Y)) = \sum_{t=1}^T h_t(\sum_{n=1}^N(a_t^{\top}Y_n), \ldots, \sum_{n=1}^N(a_t^{\top}Y_n)^N).  
\end{equation*}

To approximate $f_1$ above, we construct a PIE $\tilde{f} \in \Fcal(S_{\psi}, S_{\varphi}, L_{\psi}, L_{\varphi}, B)$  and is characterized by the following structure
\begin{gather*}
    \tilde{f}(X) = \psi(\varphi(Y)) = \tilde{g}(Y),  \\
    \tilde{g}(Y) = 
    \sum_{t=1}^T\psi_t(\varphi_t(Y)) 
    = \sum_{t=1}^T\psi_t(\sum_{n=1}^N \varphi_{t,1,n}(Y_n),\ldots, \sum_{n=1}^N \varphi_{t,N, n}(Y_n)),
\end{gather*}
where $\psi_t : \R^N \mapsto \R$ and $\varphi_{t,1,n} :\R^K \mapsto \R $ are fully connected ReLU neural networks. The detailed structure and attributes of these networks  will be provided later.

In light of the definition of $f_n$, the following inequality holds
\begin{equation*}
    \|f_1-f_n\|_{\infty} \leq \|f_1-\tilde{f}\|_{\infty} \leq |g(Y) - \tilde{g}(Y)|, ~~ \forall Y .
\end{equation*}

To establish a bound for $\|f_1-f_n\|_{\infty}$, it suffices to bound the term $|g(Y) - \tilde{g}(Y)|$. We approach this through the following decomposition
\begin{align}
    |g(Y) - \tilde{g}(Y)| \nonumber &\leq |\sum_{t=1}^T h_t(\varphi_t(Y)) - \sum_{t=1}^T h_t(g_t(Y))| + |\sum_{t=1}^T\psi_t(\varphi_t(Y)) - \sum_{t=1}^T h_t(\varphi_t(Y))| \nonumber \\
    & \leq \sum_{t=1}^T \Big[ |h_t(\varphi_t(Y)) - h_t(g_t(Y))| + |\psi_t(\varphi_t(Y)) - h_t(\varphi_t(Y))| \Big].
    \label{decomp}
\end{align}

\begin{lemma}[ReLU Approximation \citep{approx}]
    \label{lem: relu}
    Given $M'>0$ and $\epsilon \in (0,1)$, there's a ReLU network $\eta$ with two input units that implements a function $\tilde{x}:\R^2 \rightarrow \R$ so that
    \begin{itemize}
        \item for any inputs $x,y$, if $|x|\leq M' $ and $|y|\leq M'$, then $|\tilde{x}(x,y)-xy|\leq \epsilon$;
        \item the depth and the number of computation units in $\eta$ are $O(ln(1/\epsilon)+ln(M'))$;
        \item the width of $\eta$ is a constant $C_w$ independent of $x,y$.
    \end{itemize}
\end{lemma}

\textbf{Bound of the first term in (\ref{decomp})}

We commence by outlining the detailed structure of $\varphi_t$

\begin{equation}
    \varphi_t(Y) = \left[\varphi_{t,1}, \ldots, \varphi_{t,N}\right] 
    = \left[\sum_{n=1}^N \varphi_{t,1,n}, \ldots, \sum_{n=1}^N \varphi_{t,N,n}\right],
    \label{eq: varphi_t}
\end{equation}
where
\begin{equation}
    \varphi_{t,1,n}(Y_n) = w_t^{\top} Y_i,\quad \varphi_{t,k+1,n}(Y_n) = \tilde{x}_1(\varphi_{t,1,n}(Y_n), \varphi_{t,k,n}(Y_n)),
    \label{eq: varphi_construction}
\end{equation}
with $w_t \in [-B,B]^{K}$ representing learnable weights approximating $a_t^{\top}$, and $\tilde{x}_1$ being a ReLU network as per Lemma \ref{lem: relu}. Given that $\|Y\|_{\infty} \leq 1$, it follows that $|(a_t^{\top} Y_n)^j| \leq |a^{\top}_{t}Y_n|^j \leq (aK)^j$. Consequently, the deviation $|\varphi_t(Y) - g_t(Y)|$ can be bounded by iteratively applying Lemma \ref{lem: relu} with parameters $M' = (aK)^N$ and $\epsilon = \delta_1$, leading to 
\begin{align*}
    &|\varphi_{t,1,n}(Y_n) - (a_t^{\top} Y_n)^1| = 0,  \\
    &|\varphi_{t,2,n}(Y_n) - (a_t^{\top} Y_n)^2|= |\tilde{x}_1(\varphi_{t,1,n}(Y_n),\varphi_{t,1,n}(Y_n))- (a_t^{\top}Y_n)^2| \leq \delta_1, \\
    &|\varphi_{t,k+1,n}(Y_n) - (a_t^{\top}Y_n)^{k+1}| \\
    &= |\tilde{x}_1(\varphi_{t,1,n}(Y_n), \varphi_{t,k,n}(Y_n))-\varphi_{t,1,n}(Y_n)\varphi_{t,k,n}(Y_n)+a_t^{\top}Y_n\varphi_{t,k,n}(Y_n)-(a_t^{\top}Y_n)^{k+1}| \\
    &\leq \delta_1 + aK|\varphi_{t,k,n}(Y_n) - (a_t^{\top}Y_n)^k|. 
\end{align*}
Through recursion, we obtain
\begin{equation}
    \forall ~ t,n, \quad |\varphi_{t,k,n}(Y_n) - (a_t^{\top}Y_n)^k| \leq \delta_1 (\sum_{n=0}^{k-1}(aK)^n) \leq \delta_1(aK)^k.
    \label{eq: reduction}
\end{equation}
So we can bound the error between $\varphi_t(Y)$ and $g_t(Y)$ as follows
\begin{align*}
    &\|\varphi_t(Y)-g_t(Y)\|_{\infty} \nonumber\\
    &= \left\|\left[\sum_{n=1}^N\varphi_{t,1,n}(Y_n),\ldots, \sum_{n=1}^N\varphi_{t,N,n}(Y_n)\right]-\left[\sum_{n=1}^N(a_t^{\top}Y_n), \ldots, \sum_{n=1}^N(a_t^{\top}Y_n)^N\right]\right\|_{\infty} \nonumber \\
    & =\|\sum_{n=1}^N\varphi_{t,N,n}(Y_n) - \sum_{n=1}^N(a_t^{\top}Y_n)^N\|_{\infty} \nonumber \\
    & \leq \sum_{n=1}^N\|\varphi_{t,N,n}(Y_n)-(a_t^{\top}Y_n)^N\|_{\infty} \nonumber \\
    &\leq \delta_1 N(aK)^N.
\end{align*}

\begin{lemma}[Lipschitz continulty of polynomial functions]
    \label{lem: Lipschitz}
Suppose $f(x): [-k,k]^n \rightarrow \R$ is a polynomial with degree $\beta$ and coefficients bounded by $C$. $f$ is Lipschitz continues with Lipschitz constant $C\beta k^{\beta-1}$
\begin{equation*}
    |f(x)-f(y)| = |f(x) -(f(x) - \nabla f(\xi)^{\top} (y-x))| = |\nabla f(\xi)^{\top}(x-y)| \leq C\beta k^{\beta-1}\|x-y\|_{\infty}.
\end{equation*}
\end{lemma}

Recall that $h_t$ is a polynomial over N variables, having a maximum order of $MN + \beta$. Its coefficients are bounded by $C$ and the infinity norm of the input $\|g_t\|_{\infty} = \sum_{n=1}^N(a^{\top}_tY_n)^N$ can be upper bounded by $ N(aK)^N$. By applying Lemma \ref{lem: Lipschitz} with a parameter $k = N(aK)^N$ to $h_t$, we obtain
\begin{align}
    &|h_t(\varphi_t(Y)) - h_t(g_t(Y))| \nonumber \\
    & \leq C(MN+\beta)(N(aK)^N)^{MN+\beta}\|\varphi_t(Y)-g_t(Y)\|_{\infty} \nonumber \\
    & \leq C(MN+\beta)(N(aK)^N)^{MN+\beta}N^2(aK)^N \delta_1.
    \label{eq: inner_diff}
\end{align}

\textbf{Bound of the second term in (\ref{decomp})}

Then we provide the detailed structure of $\psi_t$
\begin{equation}
    \psi_t(x) = \sum_{i = 0}^{MN+\beta} \sum_{\s \in \{ 1,\ldots ,N \}^i} \gamma_{t,\s} \tilde{\psi}_{\s}(x), \quad  x \in \R^N, 
    \label{eq: psi_t}
\end{equation} 
where $\gamma_{t,\s}\in [-B,B]$ are learnable weights and $\tilde{\psi}_{\s}(x)$ is the following ReLU network approximating $\prod_{j\in \s}x_j$ by 
$$
\left\{ 
\begin{array}{l}
    \tilde{\psi}_{\s,1}(x) = x_{\s_1}, \\
    \tilde{\psi}_{\s,k}(x) = \tilde{x}_2(\tilde{\psi}_{\s,k-1}(x), x_{\s_k}), \quad k \leq |\s|,  \\
    \tilde{\psi}_{\s} = \tilde{\psi}^{\s}_{|\s|}, 
    \end{array} 
    \right. 
$$
where $\tilde{x}_2$ is the ReLU network from Lemma \ref{lem: relu}, with parameter settings of $\epsilon = \delta_2$ and $M' = (\|\varphi_{t}(X)\|_{\infty})^{MN+\beta} \leq N^{MN+\beta}(aK)^{MN^2+\beta N}$. Adhering to the procedure in (\ref{eq: reduction}), we derive
\begin{equation*}
    |\tilde{\psi}_{\s}(\varphi_t(Y))-\prod_{j\in \s}\varphi_{t,j}(Y)| \leq \delta_2 (MN+\beta)\|\varphi_t(Y)\|_{\infty}^{MN+\beta} \leq \delta_2(MN+\beta)(N(aK)^N)^{MN+\beta}.
\end{equation*}
By setting $\gamma_{t,\s}= c_{t,\s}$, we establish a bound for the error as follows
\begin{align}
    & |\psi_t(\varphi_t(Y)) - h_t(\varphi_t(Y))| \nonumber \\
    & = \sum_{i = 0}^{MN+\beta} \sum_{\s \in \{ 1,\ldots,N \}^i} c_{t,\s} |\tilde{\psi}_{\s}(\varphi_t(Y))-\prod_{j\in \s}\varphi_{t,j}(Y)| \nonumber \\
    & \leq N^{MN+\beta} C (MN+\beta)(N(aK)^N)^{MN+\beta}\delta_2 .
    \label{eq: outer_diff}
\end{align}

In summary, we have demonstrated the following
\begin{align}
\label{eq: in-out diff}
    \|f_1-f_n\|_{\infty} &\leq |g(Y) - \tilde{g}(Y)| \nonumber \\
    &\leq \sum_{t=1}^T \Big[ |h_t(\varphi_t(Y)) - h_t(g_t(Y))| + |\psi_t(\varphi_t(Y)) - h_t(\varphi_t(Y))| \Big] \nonumber \\
    &\leq TC(MN+\beta)\left[(N(aK)^N)^{MN+\beta+2} \delta_1 +(N^2(aK)^N)^{MN+\beta}\delta_2\right],
\end{align}
where the final inequality is deduced from \eqref{eq: inner_diff} and \eqref{eq: outer_diff}, along with certain calculations.

To put things together, recall that
\begin{equation}
    \|f_n-f^*\|_{\infty} \leq \|f_1-f^*\|_{\infty} + \|f_n-f_1\|_{\infty}.
    \label{eq: approx_decomp}
\end{equation}
To ensure $\|f_n-f^*\|_{\infty} \leq \epsilon$, we set $\|f_1-f^*\|_{\infty}$ and $\|f_1-f_n\|_{\infty}$ to be less than or equal to $\frac{\epsilon}{2}$. Based on the upper bound for $\|f_1-f^*\|_{\infty}$ as given in (\ref{eq: Taylor_error}), we establish
\begin{equation*}
    \frac{2^d d^{\beta} J}{\beta!}\left(\frac{1}{N'}\right)^{\beta}= \frac{\epsilon}{2}  .
\end{equation*}
This setting confirms that $\|f_1-f^*\|_{\infty} \leq \frac{\epsilon}{2}$, leading to $N' = \left( \frac{{\beta}! \epsilon}{2^{d+1} d^{\beta} J} \right)^{-1/{\beta}} = O(\epsilon^{-1/\beta})$. Consequently, we have $T = O((N')^{MN}) = O(\epsilon^{-\frac{MN}{\beta}})$ and $K = (N'+1)M = O(M\epsilon^{-1/\beta})$.

Finally, we set both terms in \eqref{eq: in-out diff} equal to $\frac{\epsilon}{4}$, leading to 
\begin{align*}
    TC(MN+\beta)(N(aK)^N)^{MN+\beta+2}\delta_1 = \frac{\epsilon}{4}, \\
    TC(MN+\beta)(N^2(aK)^N)^{MN+\beta}\delta_2 = \frac{\epsilon}{4}, \\
\end{align*}
disregarding lower-order terms and constants, we have
\begin{eqnarray}
    \delta_1 = O(\frac{\epsilon}{T N^{MN+\beta}(aK)^{MN^2+\beta N}}) = O(\epsilon^{\frac{MN^2+N\beta}{\beta}}),
    \label{delta1} \\
    \label{delta2}
    \delta_2 = O(\frac{\epsilon}{T N^{2MN+2\beta}(aK)^{MN^2+\beta N}}) = O(\epsilon^{\frac{MN^2+N\beta}{\beta}}).
\end{eqnarray}
 
Utilizing $\delta_1$ and $\delta_2$, we determine the bounds for the structural parameters of the PIE $\psi(\varphi(Y))$, referencing Lemma \ref{lem: relu}. Notably, $\psi(\varphi(Y)) = \sum_{t=1}^T \psi_t(\varphi_t(Y))$ signifies that the outlier network $\psi$ is a composition of $T$ sub-networks $\psi_t$, each being a weighted average of $N^{MN+\beta}$ distinct $\tilde{\psi}_{\s}$ functions. Each $\tilde{\psi}_{\s}$ is recursively constructed using at most $MN+\beta$ instances of $\tilde{x}_2$, as indicated in (\ref{eq: approx_decomp}), where $\tilde{x}_2(\cdot)$ possesses a depth and weight count of $O(\ln(1/\delta_2))$. Thus, each $\tilde{\psi}_{\s}$ exhibits a depth and parameter count of at most $(MN+\beta)O(\ln(1/\delta_2))$. Consequently, $\psi_t$ represents a weighted aggregation of up to $N^{MN+\beta}$ distinct $\tilde{\psi}_{\s}$ functions. Therefore, we establish that 
\begin{align*}
    &L_{\psi} = (MN+\beta)O(\ln(1/\delta_2)) = O(\ln(\epsilon^{-\frac{MN^2+N\beta}{\beta}})),  \\
    &S_{\psi} = TN^{MN+\beta}(MN+\beta)O(\ln(1/\delta_2)) = O(\epsilon^{-\frac{MN}{\beta}}).
\end{align*}
Likewise, the inner network $\varphi$ comprises $T$ sub-networks $\varphi_1,\ldots, \varphi_T$, each defined iteratively using $N$ instances of $\tilde{x}_1$ and $T$ distinct learnable weights $w_t$, as outlined in \eqref{eq: varphi_construction}. Hence, we ascertain
\begin{align*}
    &L_{\varphi} = NO(\ln(1/\delta_1)) = O(\ln(\epsilon^{-\frac{MN^2+N\beta}{\beta}})),\\
    &S_{\varphi} = KT + TN O(\ln(1/\delta_1)) = O(\epsilon^{-\frac{MN}{\beta}}).
\end{align*}

\subsection{Estimation error}
\label{sec: est}
Here, we evaluate the term
\begin{equation*}
    \frac{2}{n} \sum_{i=1}^n \xi_i(\hat{f}(Y_{(i)}) - f(Y_{(i)})).
\end{equation*}
As defined in (\ref{eq: Y}), here $Y_{(i)} \in \R^{K\times N}$ is the transformation of the $i$-th sample $X_{(i)} \in \R^{M\times N}$ in the dataset.
To employ concentration inequalities, we examine the expectation  of the upper bound of this term

\begin{equation*}
    \E_{\xi}\left[\sup_{f\in \Fcal}\left|\frac{2}{n}\sum_{i=1}^n \xi_i(\hat{f}(Y_{(i)}) - f(Y_{(i)}))\right|\right].
\end{equation*} 

Initially, we define a subset $\Fcal_{\delta} \subset \Fcal$ as
\begin{equation*}
    \Fcal_{\delta} = \{f-\hat{f} : \|f-\hat{f}\|_n \leq \delta, f\in \Fcal\}.
\end{equation*}
where $\delta$ is finite, given that functions in $\Fcal_{\delta}$ are bounded as
\begin{equation*}
    \|f-\hat{f}\|_n \leq \|f-\hat{f}\|_{\infty} \leq \|f\|_{\infty} + \|\hat{f}\|_{\infty} \leq 2J, \quad f\in \Fcal.
\end{equation*}

Thus, by applying the chaining technique from Lemma \ref{lem:F-Chaining}, we deduce
\begin{equation}
    E_{\xi} \left[ \sup_{f'\in \Fcal_{\delta}} \left| \frac{1}{n} \sum_{i=1}^n\xi_if'(Y_{(i)})\right| \right] \leq \frac{4\sqrt{2}\sigma}{n^{1/2}}\int_0^{\delta/2}\sqrt{2\log\Ncal(\epsilon', \Fcal_{\delta}, \| \cdot \|_{\infty})} d\epsilon',
\label{expectation}
\end{equation}
and subsequently, we establish a bound for $\log\Ncal(\epsilon', \Fcal_{\delta}, \|\cdot\|_{\infty})$, analogous to Theorem 2 in \citep{cover}.
\begin{lemma}
    \label{lem: cover}
    Let $f:[-M,M]^d \mapsto \R$ be fully connected ReLU neural network. Suppose $f$ has $L$ layers $W_1,\ldots, W_L$ with weights bounded by $B$ and let $\sigma$ denotes the ReLU activation function, we can represent $f$ as
    \begin{equation*}
        f(X) = W_L\sigma(W_{L-1} \cdots\sigma(W_1 X)\cdots).
    \end{equation*}
    Let $p_l$ denotes the width of layer $W_l$, then $f(X)$ is upper bounded by $\prod_{l=1}^L(p_lB)M = MB^{L}\prod_{l=1}^Lp_l$ and is Lipschitz continues with Lipschitz constant $\prod_{l=1}^L(p_lB)=B^{L}\prod_{l=1}^L p_l$. Suppose $f^*$ is another neural network with same structure as $f$ but with all the parameters $\epsilon$ away from $f$, then
    \begin{equation*}
        \|f-f^*\|_{\infty}\leq \epsilon \prod_{l=1}^L(p_lB)L M = \epsilon LM B^{L}\prod_{l=1}^Lp_l.
    \end{equation*}
\end{lemma}
Lemma \ref{lem: cover} is an application of Lemma 8 in \citep{cover}, we skip the detailed proof.

Consider $\psi(\varphi), \psi^*(\varphi^*) \in \Fcal$, which are two PIEs with identical structures. However, each corresponding parameter in these PIEs differs by no more than $\zeta$. We aim to establish a bound for $\|\psi(\varphi) - \psi^*(\varphi^*)\|_{\infty}$

\begin{equation*}
    |\psi(\varphi(Y))-\psi^*(\varphi^*(Y))| \leq \sum_{t=1}^T|\psi_t(\varphi_t(Y))-\psi^*_t(\varphi^*_t(Y)) |.
\end{equation*}

To further analyze this, we decompose the last term as follows
\begin{equation}
    |\psi_t(\varphi_t(Y))-\psi^*_t(\varphi^*_t(Y)) |\leq 
    |\psi_t(\varphi_t(Y))-\psi_t(\varphi^*_t(Y))|+|\psi_t(\varphi^*_t(Y))-\psi^*_t(\varphi^*_t(Y)) |.
    \label{eq: cover-decomp}
\end{equation}

\textbf{Bound of the first term in \eqref{eq: cover-decomp}}

We first bound $\|\varphi_t(Y)-\varphi_t^*(Y)\|_{\infty}$. Recall that $\varphi_t$ is iteratively defined with ReLU network $\tilde{x}_1$,which is a special case of $f$ of Lemma \ref{lem: cover}. Here we use $L_{\tilde{x}_1}$ to denote the number of layers in $\tilde{x}_1$. From the construction of $\varphi$, we know $L_{\varphi} = N L_{\tilde{x}_1}$. From Lemma\ref{lem: relu}, the width of $\tilde{x}_1$ is bounded by $C_w$. We apply Lemma \ref{lem: cover} with $L$ being $L_{\tilde{x}_1}$ and $p_l = C_w$, so for any $t,n$
\begin{align*}
    \varphi_{t,1,n}(Y_n) &= w_t^{\top}Y_n \leq BK, \\
    \varphi_{t,2,n}(Y_n) &= \tilde{x}_1(\varphi_{t,1,n}(Y_n), \varphi_{t,1,n}(Y_n)) \leq (C_wB)^{L_{\tilde{x}_1}}BK, \\
    \varphi_{t,k+1,n}(Y_n) &= \tilde{x}_1(\varphi_{t,1,n}(Y_n), \varphi_{t,k,n}(Y_n)) \leq (C_wB)^{L_{\tilde{x}_1}} |\varphi_{t,k,n}(Y_n)| \leq (C_wB)^{i\times L_{\tilde{x}_1}}BK.
\end{align*}
Thus, $\varphi_{t,N,n}(Y_n) \leq (C_wB)^{L_{\varphi}} BK$ and $\varphi_t(Y) \leq N (C_wB)^{L_{\varphi}} BK$. By defining $\Delta_{t,k,n} = |\varphi_{t,k,n}(Y_n) - \varphi_{t,k,n}^*(Y_n)|$, we iteratively bound the differences

\begin{align*}
    \Delta_{t,1,n} &= |w_t^{\top}Y_n-w_t^{*\top}Y_n| \leq K \zeta ,\\
    \Delta_{t,k,n} &\leq |\tilde{x}_1(\varphi_{t,1,n}(Y_n),\varphi_{t,k-1,n}(Y_n)) - \tilde{x}_1(\varphi_{t,1,n}(Y_n) \varphi_{t,k-1,n}^*(Y_n))| \\
    &+ |\tilde{x}_1(\varphi_{t,1,n}(Y_n),\varphi_{t,k-1,n}^*(Y_n)) - \tilde{x}_1^*(\varphi_{t,1,n}(Y_n), \varphi_{t,k-1,n}^*(Y_n))|\\
    &\leq (C_wB)^{L_{\tilde{x}_1}}\Delta_{t,k-1,n} + \zeta (C_wB)^{L_{\tilde{x}_1}}L_{\tilde{x}_1}\max(\varphi_{t,1,n}(Y_n),\varphi_{t,k-1,n}^*(Y_n)) \\
    &\leq (C_wB)^{L_{\tilde{x}_1}}\Delta_{t,k,n} + \zeta (C_wB)^{L_{\tilde{x}_1}}L_{\tilde{x}_1}(C_wB)^{i-1\times L_{\tilde{x}_1}}BK.
\end{align*}

Through recursive application, $\Delta_{t,N,n}$ is bounded by $\zeta (C_wB)^{N \times L_{\tilde{x}_1}} L_{\varphi} BK = \zeta (C_wB)^{L_{\varphi}} L_{\varphi} BK$, allowing us to establish a bound for $\|\varphi_t(Y) - \varphi_t^*(Y)\|_{\infty}$
\begin{align}
    &\|\varphi_t(Y)-\varphi_t^*(Y)\|_{\infty} \nonumber\\
    &= \left\|\left[\sum_{n=1}^N\Delta_{t,1,n}(Y_n),\ldots, \sum_{n=1}^N\Delta_{t,N,n}(Y_n)\right]\right\|_{\infty} \nonumber\\
    &\leq \sum_{n=1}^N\Delta_{t,N,n}(Y_n) \nonumber\\
    &\leq \zeta (C_wB)^{L_{\varphi}}L_{\varphi}BKN.
    \label{eq: inner_epsilon}
\end{align}

To finalize the bound on the first term in \eqref{eq: cover-decomp}, it is necessary to establish the Lipschitz continuity of $\psi_t$. Assume $x, y \in \R^N$, then
\begin{equation*}
    |\psi_t(x)-\psi_t(y)| = | \sum_{i = 0}^{MN+\beta} \sum_{\s \in \{ 1,..,N \}^i} \gamma_{t,\s} (\tilde{\psi}_{\s}(x)-\tilde{\psi}_{\s}(y))|.
\end{equation*}
According to Lemma \ref{lem: relu}, $\tilde{x}_2$ is Lipschitz continuous with a Lipschitz constant of $(C_wB)^{L_{\tilde{x}_2}}$. Let $\Gamma_i = |\tilde{\psi}_{\s,i}(x) - \tilde{\psi}_{\s,i}(y)|$, we have 
\begin{align*}
    \Gamma_1 &= |x_{\s_1}-y_{\s_1}| \leq \|x-y\|_{\infty}, \\
    \Gamma_{i+1} &=\tilde{x}_2(\tilde{\psi}_{\s,i}(x), x_{\m_{i+1}}) - \tilde{x}_2(\tilde{\psi}_{\s,i}(y), y_{\m_{i+1}}) \\
    &\leq (C_wB)^{L_{\tilde{x}_2}} \max\{\Gamma_i, |x_{\s_{i+1}}-y_{\s_{i+1}}|\} \\
    &\leq (C_wB)^{L_{\tilde{x}_2}} \Gamma_i.
\end{align*}
After some computations, we find $|\tilde{\psi}_{\s}(x) - \tilde{\psi}_{\s}(y)| \leq (C_wB)^{L_{\psi}}\|x - y\|_{\infty}$. Consequently, $\psi_t$ exhibits Lipschitz continuity with
\begin{equation*}
    |\psi_t(x)-\psi_t(y)| \leq N^{MN+\beta} (C_wB)^{L_{\psi}}\|x-y\|_{\infty}.
\end{equation*}
Therefore, the first term in \eqref{eq: cover-decomp} can be bounded as follows
\begin{align}
    \|\psi_t(\varphi_t(Y))-\psi_t(\varphi^*_t(Y))\|_{\infty} &\leq N^{MN+\beta} (C_wB)^{L_{\psi}}\|\varphi_t(Y)-\varphi_t^*(Y)\|_{\infty} \nonumber \\
    &\leq \zeta N^{MN+ \beta}(C_wB)^{L_{\psi}+L_{\varphi}}L_{\varphi}K,
    \label{eq: est-first}
\end{align}
where the last inequality is derived by incorporating \eqref{eq: inner_epsilon} and disregarding the term $BN$.

\textbf{Bound of the second term in \eqref{eq: cover-decomp}}

With the similar argument as in \eqref{eq: inner_epsilon}, we derive the following bound
\begin{align*}
    \|\psi_t(\varphi^*_t(Y))-\psi^*_t(\varphi^*_t(Y)) \|_{\infty} \leq |\sum_{i = 0}^{MN+\beta} \sum_{\s \in \{ 1,..,N \}^i} (\gamma_{t, \s}\tilde{\psi}_{\s}-\gamma_{t, \s}\tilde{\psi}^*_{\s})(\varphi^*_t(Y))|.
\end{align*}
Additionally, we establish that
\begin{align*}
    |(\gamma_{t, \s}\tilde{\psi}_{\s}-\gamma_{t, \s}\tilde{\psi}^*_{\s})(\varphi^*_t(Y))| &\leq |(\gamma_{t, \s}\tilde{\psi}_{\s}-\gamma_{t, \s}\tilde{\psi}^*_{\s}+\gamma_{t, \s}\tilde{\psi}^*_{\s}-\gamma_{t, \s}\tilde{\psi}^*_{\s})(\varphi^*_t(Y))| \\
    & \leq |\gamma_{t,\s} \zeta(C_wB)^{L_{\psi}}L_{\psi}\|\varphi^*_t(Y)\|_{\infty}| + |(\gamma_{t,\s}-\gamma_{t,\s}^*)(C_wB)^{L_{\psi}}\|\varphi^*_t(Y)\|_{\infty}| \\
    & \leq [\zeta B(C_wB)^{L_{\psi}}L_{\psi} + \zeta (C_wB)^{L_{\psi}}]\|\varphi^*_t(Y)\|_{\infty} \\
    & \leq \zeta (C_wB)^{L_{\psi}}L_{\psi}(C_wB)^{L_{\varphi}}K \\
    & = \zeta (C_wB)^{L_{\psi}+L_{\varphi}}L_{\psi}K.
\end{align*}
In a similar fashion to our earlier derivation for $\varphi$, we apply bounds to $\tilde{\psi}_{\s}$ and evaluate the discrepancy between $\tilde{\psi}_{\s}$ and $\tilde{\psi}^*_{\s}$. For simplicity, we omit the lower order term $B$ in the fourth inequality. Consequently, we obtain
\begin{equation}
    \|\psi_t(\varphi^*_t(Y))-\psi^*_t(\varphi^*_t(Y)) \|_{\infty} \leq \zeta (C_wB)^{L_{\psi}+L_{\varphi}}L_{\psi}N^{MN+\beta}K .
    \label{eq: est-second}
\end{equation}
To summarize, we bring (\ref{eq: est-first}) and (\ref{eq: est-second}) together
\begin{equation}
\label{est_error}
    \|\psi(\varphi(Y))-\psi^*(\varphi^*(Y))\|_{\infty}\leq \zeta T(C_wB)^{L_{\psi}+L_{\varphi}}(L_{\psi}+L_{\varphi})N^{MN+\beta}K .
\end{equation}

Given that the total number of parameters is constrained by $S_{\psi} + S_{\varphi}$, with each parameter bounded by $B$, we discretize the range of each parameter using a grid of size
\begin{equation*}
    \Delta = \epsilon'/T(C_wB)^{L_{\psi}+L_{\varphi}}(L_{\psi}+L_{\varphi})N^{MN+\beta}K ,
\end{equation*}
which leads to an upper bound on the covering number as follows
\begin{equation*}
    \log\Ncal(\epsilon', \Fcal_{\delta}, \| \cdot \|_{\infty}) \leq \log((\frac{2B}{\Delta})^{S_{\psi}+S_{\varphi}})= (S_{\psi}+S_{\varphi})\log(\frac{BT(C_wB)^{L_{\psi}+L_{\varphi}}(L_{\psi}+L_{\varphi})N^{MN+\beta}K }{\epsilon'}).
\end{equation*}

Returning to \eqref{expectation} and integrating, we derive
\begin{equation*}  
    E_{\xi} \left[ \sup_{f'\in \Fcal_{\delta}} \left| \frac{1}{n} \sum_{i \in [n]}\xi_if'(Y_{(i)})\right| \right] \leq 2\sqrt{2} \frac{\sigma \sqrt{S_{\psi}+S_{\varphi}}\delta }{n^{1/2}} \left( \log  \frac{(T(C_wB)^{L_{\psi}+L_{\varphi}}(L_{\psi}+L_{\varphi})N^{MN+\beta}K}{ \delta} +1 \right).
\end{equation*}
Utilizing this bound for the expectation , we apply the Gaussian concentration inequality, as presented in \citep{foundation}, by considering $\frac{1}{n}\sum_{i=1}^n\xi_i f'(Y_{(i)})$ as the Gaussian process $X(t)$. The variance of $\frac{1}{n}\sum_{i=1}^n\xi_i f'(Y_{(i)})$ is upper bounded by $\frac{\sigma^2 \delta^2}{n}$, thus
\begin{align}
    1 &-exp(-nu^2/2\sigma^2 \delta^2) \nonumber \\
    &\leq P_{\xi}\left(2 \sup_{f'\in \Fcal_{\delta}}|\frac{1}{n}\sum_{i=1}^n\xi_i f'(Y_{(i)})|\leq 2E_{\xi} [ \sup_{f'\in \Fcal_{\delta}} | \frac{1}{n} \sum_{i \in [n]}\xi_if'(Y_{(i)})| ]+u\right) \nonumber \\
    & \leq P_{\xi}\left(2 \sup_{f'\in \Fcal_{\delta}}|\frac{1}{n}\sum_{i=1}^n\xi_i f'(Y_{(i)})|\leq V_n \delta  (\log  \frac{(T(C_wB)^L LN^{MN+\beta}K}{\delta} +1)+u\right), \label{eq: est-concentration}
\end{align}
where $L=L_{\psi}+L_{\varphi}$, $S=S_{\psi}+S_{\varphi}$ and $V_n = 4\sqrt{2} \frac{\sigma \sqrt{S}}{n^{1/2}}$.

Let $\delta = \max\{\|\hat{f}-f\|_n, V_n\}$, then by the definition of $\Fcal_{\delta}$, we have
\begin{equation*}
    \sup_{f\in \Fcal}|\frac{1}{n}\sum_{i=1}^n\xi_i (\hat{f}(Y_{(i)}-f(Y_{(i)}))| \leq \sup_{f'\in \Fcal_{\delta}}|\frac{1}{n}\sum_{i=1}^n\xi_i f'(Y_{(i)})|.
\end{equation*}
Thus, for any $f \in \Fcal$, it follows that
\begin{align}
    &| \frac{2}{n}\sum_{i=1}^n\xi_i(\hat{f}(Y_{(i)})-f(Y_{(i)}))| \nonumber\\
    &\leq \max\{\|\hat{f}-f\|_n, V_n\}\left\{V_n (\log  \frac{(T(C_wB)^L LN^{MN+\beta}K}{V_n} +1)\right\} +u \nonumber\\
    &\leq \frac{1}{4}(\max\{\|\hat{f}-f\|_n, V_n\})^2+2\left\{V_n(\log  \frac{(T(C_wB)^L LN^{MN+\beta}K}{V_n} +1) \right\}^2 +u\label{eq: after_xy},
\end{align}
with probability at least $1 -exp(-nu^2/2\sigma^2 \delta^2$). The last inequality holds since $xy\leq \frac{1}{4}x^2+2y^2$. Utilizing the inequality given in \eqref{eq: basic_inequality}, we have
\begin{equation*}
    -\frac{2}{n}\sum_{i=1}^n\xi_i(\hat{f}(Y_{(i)})-f(Y_{(i)})) + \|f^*-\hat{f}\|_n^2 
    \leq \|f^*-f\|_n^2 .
\end{equation*}
Applying the inequality $\frac{1}{2}\|\hat{f} - f\|_n^2 \leq \|f - f^*\|_n^2 + \|f^* - \hat{f}\|_n^2$, we obtain
\begin{equation}
    -\frac{2}{n}\sum_{i=1}^n\xi_i(\hat{f}(Y_{(i)})-f(Y_{(i)})) + \frac{1}{2}\|\hat{f}-f\|^2_n \leq 2\|f^*-f\|_n^2. \label{eq: modify_basic}
\end{equation}
Combining \eqref{eq: modify_basic} and \eqref{eq: after_xy}, we derive
\begin{equation*}
    -\frac{1}{4}(\max\{\|\hat{f}-f\|_n, V_n\})^2-2\left\{V_n(\log  \frac{(T(C_wB)^L LN^{MN+\beta}K}{V_n} +1) \right\}^2 -u + \frac{1}{2}\|\hat{f}-f\|^2_n \leq 2\|f^*-f\|_n^2.
\end{equation*}
It can be verified that whether $\|\hat{f} - f\|_n \geq V_n$ or $\|\hat{f} - f\|_n \leq V_n$, the following holds
\begin{equation}
    \|\hat{f}-f\|_n \leq 4\left\{V_n(\log  \frac{(T(C_wB)^L LN^{MN+\beta}K}{V_n} +1) \right\}^2 +2u +4\|f^*-f\|_n^2.
    \label{eq: bound-hatf-f}
\end{equation}

Apply (\ref{eq: bound-hatf-f}) to the inequality $\frac{1}{2}\|\hat{f}-f^*\|_n^2\leq \|f^*-f\|^2_{n}+\|\hat{f}-f\|_n^2$, we obtain
\begin{equation}
    \|\hat{f}-f^*\|_n^2 \leq 10\|f^*-f\|_n^2+ 8\left\{V_n(\log  \frac{(T(C_wB)^L LN^{MN+\beta}K}{V_n} +1) \right\}^2 +4u,
    \label{eq: final-decomp-results}
\end{equation}
 with probability at least $1-\exp(-nu^2/2\sigma^2 \delta^2)$ for all $u>0$.

\subsection{Overall order}

Recall that \eqref{eq: final-decomp-results} is valid for any $f \in \Fcal$, allowing us to select $f$ as $f_n$ from Section \ref{sec: app}. So that $\delta^2 = \max\{\|\hat{f}-f\|^2_n, V_n^2\}$. Then we set $u = \frac{1}{16} \delta^2$
\begin{align}
    \|\hat{f}-f^*\|_n^2 &\leq 10\|f^*-f_n\|_n^2+ 8\left\{V_n(\log  \frac{(T(C_wB)^L LN^{MN+\beta}K}{V_n} +1) \right\}^2 + \frac{1}{4}(\|\hat{f}-f_n\|_n^2 + V_n^2) \nonumber \\ 
    &\leq 21\|f^*-f_n\|_n^2+ 17\left\{V_n(\log  \frac{(T(C_wB)^L LN^{MN+\beta}K}{V_n} +1) \right\}^2 , 
    \label{eq: u1}
\end{align}
with probability at least $1-\exp(-n \delta^2/2\sigma^2)$. The latter inequality follows since $\frac{1}{2}\|\hat{f}-f_n\|_n^2\leq \|f^*-f_n\|^2_{n}+\|\hat{f}-f^*\|_n^2$. Let $\epsilon=n^{-\frac{\beta}{MN+2\beta}}$ and substitute the order of of $S_{\psi}, S_{\varphi}, L_{\psi}, L_{\varphi}$ from Section \ref{sec: app}. We obtain that $V_n = O(n^{-\frac{2\beta}{MN+2\beta}})$, therefore \eqref{eq: u1} holds with probability $1-\exp(-n \delta^2/2\sigma^2) \geq 1-\exp(-n V_n^2/2\sigma^2)$ converging to one. Then we take the expectation of both sides of the inequality 
with respect to $P_X$

\begin{align}
    \|&\hat{f}-f^*\|_{L_2(P_X)}^2 \nonumber \\
    & \leq 21\|f_n-f^*\|^2_{L_2(P_X)} + 17\left\{V_n(\log  \frac{(T(C_wB)^L LN^{MN+\beta}K}{V_n} +1) \right\}^2  \nonumber\\
    & \leq 21\epsilon^2 + 17\left\{V_n(\log  \frac{(T(C_wB)^L LN^{MN+\beta}K}{V_n} +1) \right\}^2 , \label{eq: final-step}
\end{align}
where the latter inequality follows since $\|f_n - f^*\|_{L_2(P_X)} \leq \|f_n - f^*\|_{\infty} \leq \epsilon^2$. Disregarding constants and lower-order terms, we obtain
\begin{align*}
    \|\hat{f}-f^*\|_{L_2(P_X)}^2 &\leq 21\epsilon^2 + 17\left\{V_n(\log\frac{\epsilon^{-\frac{MN}{\beta}}(C_wB)^{\ln(\epsilon^{-\frac{MN^2+N\beta}{\beta}}+1)}}{V_n})\right\}^2 \\
    & = O(n^{-\frac{2\beta}{MN+2\beta}}),
\end{align*}  
with probability converging to one. So we conclude
\begin{equation}
    \| \hat{f} - f^*\|^2 _{L^2(P_X)} = O_p(n^{-\frac{2\beta}{MN+2\beta}}).
    \label{eq: final-rate}
\end{equation}

\section{Proof of Theorem \ref{theo: minimax}}
This section is devoted to deriving a lower bound for the $L_2$ minimax risk associated with the class of permutation invariant functions. Additionally, we aim to demonstrate that the PIE proposed in this study constitutes an optimal estimator in the minimax framework.

\begin{definition}[$L_2$ minimax risk]
    Given a random variable $X$ following a probability measure $P_X$ on $\mathbb{R}^d$, the $L_2$ minimax risk of estimation associated with any function space $I \in L_2(P_X)$ is defined as
    \begin{equation*}
        r^2_n(I,P_X,\sigma) = \inf_{\hat{f}\in \Acal_n}\sup_{f\in I}\|\hat{f}-f\|^2_{L_2(P_X)},
    \end{equation*}
    where $\Acal_n$ is the space of all measurable functions of data in $L_2(P_X)$ and $\sigma$ is the variance of Gaussian noise in the data generating process.
\end{definition}
Given the definition of permutation invariance as specified in Assumption \ref{def: PI}, it follows that $P^{\beta, 1\times MN}$ is a subset of $P^{\beta, M\times N}$. Consequently, $r^2_n(P^{\beta, M\times N},P_X,\sigma) \geq r^2_n(P^{\beta, 1 \times MN},P_X,\sigma)$. This implies that any lower bound on the $L_2$ minimax risk of $P^{\beta, 1 \times MN}$ also serves as a valid lower bound for $P^{\beta, M\times N}$. For ease of notation, let us denote $d = MN$ and use $P^{\beta, d}$ to refer to $P^{\beta, 1\times MN}$. 

We try to get the $L_2$ minimax lower bound using the relationship between minimax risk and packing number. So we first bound the packing number $\Tcal(\epsilon, P^{\beta, d}, \|\cdot \|_{L_2})$ using the lemma that follows

\begin{lemma}
    The packing number of $P^{\beta, d}$ is lower bounded with 
    \begin{equation*}
        \Tcal(\epsilon, P^{\beta, d}, \|\cdot \|_{L_2}) \geq\exp\{M_0 (1/\epsilon)^{d/\beta}\}. 
    \end{equation*} 
In addition,  we have 
    \begin{equation*}
        \Tcal(\epsilon, P^{\beta, d}, \|\cdot \|_{L_2(P_X)}) \geq\exp\{M_0 (\underline{P_X}/\epsilon)^{d/\beta}\}, 
    \end{equation*}
    where $M_0$ depends on only $\beta$ and $d$. 
\end{lemma}

\begin{proof}
    In alignment with the approach outlined in Lemma 6.1 of \citep{minimax}, we define
    \begin{equation*}
        \Kcal(X_1, \ldots, X_d) = \prod_{j=1}^d \Kcal_0(X_j), \quad \Kcal_0(t) = te^{-1/(1-t^2)}\I(|t|\leq 1),t\in \mathbb{R}. 
    \end{equation*}  
    For arbitrary radius $h \in (0,1/2)$, we evenly split $[0,1]$ into $m = \left\lceil \frac{1}{2h} \right\rceil $ intervals and obtain $D = m^d$ rectangular grid points $\left(\frac{j_1}{m}-h, \ldots, \frac{j_d}{m}-h\right)$  with $\left(j_1,\ldots, j_d\right) \in \left\{ 1,\ldots, m\right\}^d$. We use $\left\{x_k: k=1\ldots, D\right\}$ to represent these grids and assume $h$ is small enough to ensure $D\geq 8$. 
    
    For simplicity, let $\Wcal = \Wcal_J^{\beta,\infty}([0, 1]^{MN})$ and let $\|\cdot\|_{\Wcal}$ denote the norm of the Soblev space. Correspondingly, we construct functions $\phi_k$ for each $1 \leq k \leq D$
    \begin{equation*}
        \phi_k(x) = \frac{1}{\| \Kcal \|_{\Wcal} }h^{\beta}\Kcal\left(\frac{x-x_k}{h}\right), \quad x\in [0,1]^d,
    \end{equation*}
    with each $\phi_k$ having its support confined to $[x_k - h, x_k + h]^d$. It's noteworthy that $\Kcal_0(t)$ is an odd function, ensuring that $\int \psi_k(x) dx = 0$. Moreover, the supports of $\phi_k$ and $\phi_j$ for $k \neq j$ do not overlap, thus allowing for
\begin{equation}
    \int \phi_k \phi_j dx =0,  \quad 
            \| \phi_1 + \dots + \phi_k \|_{L_2}^2 = \sum_{i=1}^k \| \phi_i \|_{L_2}^2 = kh^{2\beta+d}\frac{\| \Kcal \|_{L_2}}{\| \Kcal \|_{\Wcal} }. 
            \label{eq: sum-ind}
\end{equation}      
    Furthermore, we define permutation invariant functions $\phi^*_k$ based on $\phi_k$ as 
    \begin{equation}
        \phi^*_k(x) = \frac{1}{|S_d|^{1/2}} \sum_{\varPi \in S_d} \frac{1}{\| \Kcal \|_{\Wcal} }h^{\beta}\Kcal\left(\frac{x-\varPi_{x_k}}{h}\right), \quad x\in [0,1]^d,
        \label{eq: phi*}
    \end{equation}
    so $\phi_k^*(x)$ is a permutation invariant function for any $k$ and
    \begin{align*}
        \| \phi^*_k \|^2_{L_2} &= \frac{1}{|S_d|} \left\| \sum_{\varPi \in S_d}  \frac{1}{\| \Kcal \|_{\Wcal} }h^{\beta}\Kcal\left(\frac{x-\varPi_{x_k}}{h}\right) \right\|_{L_2}^{2} \\
        &\geq \frac{1}{|S_d|}  \sum_{\varPi \in S_d}  \left\| \frac{1}{\| \Kcal \|_{\Wcal} }h^{\beta}\Kcal\left(\frac{x-\varPi_{x_k}}{h}\right) \right\|_{L_2}^{2} \\
        &=h^{2\beta+d}\frac{\| \Kcal \|_{L_2}}{\| \Kcal \|_{\Wcal} },
    \end{align*}
    
    where the second inequality becomes equal when for all permutations $\varPi$ in $S_d$, $\varPi_{x_k} \neq x_k$. Consider the total number of distinct functions $\phi^*_k$, denoted as $D^*$. In (\ref{eq: phi*}), we construct a set of $D$ functions  $\phi_k^*$ derived from $\phi_k$. Specifically, each $\phi^*_k$ is identical to $\phi^*_j$ for distinct indices $k$ and $j$ if and only if there exists a permutation $\varPi$ such that $\varPi_{x_k}$ equals $x_j$. Considering that there are at most $|S_d| \leq d^d$ distinct permutations for the $d$-dimensional vector $x_k$, we can infer that $D^*$ is bounded from below by $D/d^d$.


    Let $\Omega=\{0,1\}^{D^*}$ and for each $\omega \in \Omega$ define $f_{\omega} = \sum_{k=1}^{D^*}\omega_k \phi^*_k$. It is evident that each $f_{\omega}$ belongs to $P^{\beta, d}$
    \begin{equation*}
        \| f_{\omega} - f_{\omega'}\|_{L_2} = \left\{\sum_{k=1}^{D^*}(\omega_k -  \omega_k') ^2\int {\phi_k^*}^2 dx \right\}^{1/2} \geq  h^{\beta+d/2}\frac{\| \Kcal \|}{\| \Kcal \|_{\Wcal}} \psi^{1/2} (\omega, \omega'),
    \end{equation*}
    where $\psi(\omega, \omega') = \sum_{k=1}^{D^*}\I (\omega_k \neq  \omega_k')$ represents the hamming distance. Applying the Varshamov-Gilbert bound from coding theory, it can be shown that there exist $U \geq 2^{D^*/8}$ binary strings $\omega^{(1)}, \ldots, \omega^{(U)} \in \Omega$ such that $\psi(\omega^{(k)}, \omega^{(k')}) \geq D^*/8$ for $0\leq k \leq k' \leq A$. Then

    \begin{equation*}
        \|f_{\omega^{(k)}}-f_{\omega^{(k')}}\|_{L_2}\geq h^{\beta+d/2}\frac{\|\Kcal\|_{L_2}}{\|\Kcal\|_{\Wcal}}\sqrt{\frac{D^*}{8}}\geq M_1h^{\beta},
    \end{equation*}
    where $M_1=\|\Kcal\|_{L_2}/\{2^{(d+3)/2}d^{d/2}\|\Kcal\|_{\Wcal}\}$ is a constant that depends solely on $\beta$ and $d$. By setting $\epsilon = M_1h^{\beta}$, we are able to identify $U$ distinct functions within the function class $P_{\beta, d}$. These functions are separated from each other by at least $\epsilon$ with respect to the $L_2$ norm distance, with
    \begin{align*}
        U &\geq\exp\{D^*\log(2)/8\} \nonumber \\
        & \geq\exp\{m^d \log(2)/d^d8\} \nonumber \\
        & \geq\exp\{(1/2h)^d \log(2)/d^d8\} \nonumber \\
        & \geq\exp\{(M_1/\epsilon)^{d/\beta} \log(2)/d^d2^{d+3}\} \nonumber \\
        & =\exp\{M_0(1/\epsilon)^{d/\beta}\},
    \end{align*} 
    where $M_0 = (M_1^{d/\beta} \log(2))/d^d2^{d+3}$ depends on only $d$ and $\beta$. Consequently, in accordance with the definition of the packing number, we have
    \begin{equation*}
        \Tcal(\epsilon, P_{\beta, d}, \|\cdot \|_{L_2}) \geq\exp\{M_0 (1/\epsilon)^{d/\beta}\}.
    \end{equation*}
For any two functions $f, f' \in P^{\beta,d}$ that satisfy $\|f - f'\|_{L_2} \geq \epsilon$, it follows that $\|f - f'\|_{L_2(P_X)} \geq \epsilon \underline{P_X}$. Therefore
    \begin{equation*}
        \Tcal(\epsilon \underline{P_X}, P_{\beta, d}, \|\cdot \|_{L_2(P_X)}) \geq\exp\{M_0 (1/\epsilon)^{d/\beta}\},
    \end{equation*}
    and upon setting $\epsilon' = \underline{P_X}\epsilon$, we deduce
    \begin{equation*}
        \Tcal(\epsilon', P_{\beta, d}, \|\cdot \|_{L_2(P_X)}) \geq\exp\{M_0 (\underline{P_X}/\epsilon)^{d/\beta}\}.
    \end{equation*}
\end{proof}

By Theorem 6 of \citep{yuhongy}, the minimax risk $r_n$ is the solution to $\log(\Tcal(r_n,P^{\beta, d}, \|\cdot\|_{L_2(P_X)}))=\frac{nr_n^2}{\sigma}$. From the definition of packing number, $\Tcal(r_n,P^{\beta, d}, \|\cdot\|_{L_2(P_X)})$ decreases as $r_n$ increases, while $r^2_n$ increases with $r_n$. So any $r$ satisfying $\Tcal(r,P^{\beta, d}, \|\cdot\|_{L_2(P_X)})\geq \frac{nr^2}{\sigma}$ is a lower bound of $r_n$. Choosing $r = (M_0\sigma)^{\frac{\beta}{2\beta+d}}(P_X)^{\frac{d}{2\beta+d}}n^{-\frac{\beta}{2\beta+d}}$, it can be verified that $\Tcal(r, P^{\beta, d}, \|\cdot\|_{L_2(P_X)}) \geq \frac{nr^2}{\sigma}$. Consequently, $r^2$ forms a lower bound for the minimax risk $r_n^2$.

Returning to the notation where $d = MN$, we deduce that
\begin{equation*}
    r^2_n(P^{\beta, M\times N},P_X,\sigma) \geq r^2_n(P^{\beta, 1\times MN},P_X,\sigma) \geq C(\beta, M,N)n^{-\frac{2\beta}{2\beta+MN}},
\end{equation*}

which aligns with the convergence rate of PIE, as established in (\ref{eq: final-rate}), up to a constant factor. Thus, the minimax optimality of PIE in approximating fully permutation invariant functions is established.

\section{Proof of Corollary \ref{co: PPIE rate}}
\label{appendix:frate} 
In this section, our objective is to substantiate Corollary \ref{co: PPIE rate} through a broader perspective. To achieve this, we  expand the scope beyond the conventional permutation invariance assumptions and the associated Permutation Invariant Estimator PIE. We introduce and focus on the concept of partially permutation invariant functions and the corresponding Partially Permutation Invariant Estimator PPIE. Once we ascertain the convergence rate of PPIE, it will provide the necessary foundation to validate Corollary \ref{co: PPIE rate}.

\begin{definition}[Partially Permutation Invariant Function]
\label{def:ppie}
A function $f$ is partially permutation invariant if its input can be separate into $P$ input matrix $\{X_1,\cdots, X_P\}$ , where each $X_p \in \mathbb{R}^{M \times N_p}$. The function remains unchanged under column permutations within each individual input matrix. Formally, this property is  expressed  as
\begin{eqnarray}
    f(X_1,X_2,\cdots,X_P) = f(\varPi^{(1)}_{X_1},\varPi^{(2)}_{X_2},\cdots,\varPi^{(P)}_{X_P}),
\end{eqnarray}
    for any permutation $\varPi^{(k)}\in \Scal_{N_p}$.
\end{definition}

The corresponding estimator, known as the Partially Permutation Invariant Estimator (PPIE), is defined as follows:

\begin{definition}[Partially Permutation Invariant Estimator(PPIE) Structure]
    \label{def:ppie network}
    Given a partially permutation invariant function $f$ with $P$ input matrix $X = \{X_1,\cdots,X_P\}$,where $X_p \in \mathbb{R}^{M \times N_p}$ as described in Definition \ref{def:ppie}, the PPIE neural network is structured as
    \begin{equation*}
        \psi\left(\frac{1}{N_1}\sum_{n=1}^{N_1}\varphi^{(1)}(X_{1,i}),\frac{1}{N_2}\sum_{n=1}^{N_2}\varphi^{(2)}(X_{2,i}),\cdots,\frac{1}{N_P}\sum_{n=1}^{N_P}\varphi^{(K)}(X_{P,i})\right),   
    \end{equation*}
    where $X_{p,i}$ denotes the $i$-th column of $X_p$, and $\varphi$ and $\psi_p$ are neural networks.
\end{definition}

Our goal is to establish the bound for
\begin{align*}
    \|f^*-\hat{f}\|_n^2 \leq \|f^*-f\|_{\infty}^2 +\frac{2}{n}\xi_i(\hat{f}(x_i)-f(x_i)) .
\end{align*}
It can be shown that the fully permutation invariant target function, as stipulated in Assumption \ref{ass:PI}, is a specific case of the partially permutation invariant function. This is achieved by setting $P = 2$, with $X_1 \in \mathbb{R}^{(M+1) \times |\Ncal(i)|}$ and $X_2 \in \mathbb{R}^{(M+1) \times 1}$. Therefore, our focus shifts to examining the convergence rate of the PPIE function class.

\subsection{Approximation Error}
For ease of discussion, let us denote $N = \sum_{p=1}^P N_p$. Following a similar approach as in Section \ref{sec: app}, we define an approximation of the function $f^*$ by
\begin{equation*}
    f^*_1(X) = \sum_{\m\in\{0,1,\cdots,N^{\prime}\}^{K\times M \times N}} \phi_{\m}(X)  \sum_{\p: |\p|\leq\beta}\frac{D^{\p} f}{\p!}\bigg|_{X = \frac{\m}{N'}}(X-\frac{\m}{N'})^{\p},
\end{equation*}
where $\phi_{\m}$ is defined identically to its counterpart in Section \ref{sec: app}.

Considering $X = \{X_1, \cdots, X_P\}$ and defining $|\widehat{\Scal}| = \prod_{p=1}^P |\Scal_{N_p}|$, we construct the partially permutation invariant (PI) approximation using $f_1^*$
\begin{equation*}
    f_1(X) = f_1(X_1,\cdots,X_P) = \frac{1}{|\widehat{\Scal}|}\sum_{\substack{\varPi_1 \in \Scal_{N_1} \cdots \\ \varPi_K \in \Scal_{N_P}}} f_1^{*}\left((\varPi_1)_{X_1}, \cdots, (\varPi_K)_{X_P}\right).
\end{equation*}

Using similar arguments, it can be shown that
\begin{equation*}
    \|f^*(X)-f_1(X)\|_{\infty} \leq \frac{2^d d^{\beta}J}{\beta!}\left(\frac{1}{N'}\right)^{\beta}.
\end{equation*}

Let $K'_p = (N' + 1) \times N_p$ and define $K' = \sum_{p=1}^P K'_p$. We then construct $Y(X, N') = (Y_1(X_1, N'), \ldots, Y_P(X_P, N'))$ where $Y_p(X_p, N') \in \mathbb{R}^{K'_p \times M}$ is defined analogously to (\ref{eq: Y}). Consequently, $f_1(X)$ can be transformed into $g(Y)$ by 
\begin{align*}
    f_1^*(X)  &= g^*(Y(X,N')), \\
    f_1(X) &= \frac{1}{|\widehat{\Scal}|} \sum_{\substack{\varPi_1 \in \Scal_{N_1} \cdots \\ \varPi_K \in \Scal_{N_P}}} g^*\left(Y_1((\varPi_1)_{ X_1},N'), \cdots, Y_p((\varPi_K)_{X_P},N')\right)= g(Y). 
\end{align*}

From the arguments in Section \ref{sec: app}, we know $g(Y)$ is a partially permutation invariant polynomial. To further decompose $g(Y)$, we introduce the following lemma which is similar to Lemma \ref{lem: weyl}.
\begin{lemma}[Partially PI polynomial decomposition]
\label{lem: PPIE decompo}
    Given $Y = \{Y_1,\cdots,Y_P\}$ where  $Y_p\in\mathbb{R}^{M\times N_p}$, any partially permutation invariant polynomial $g(Y)$ can be  expressed  in the following form
    \begin{equation*}
g(Y)=\sum_{q=1}^Q h_{1, q}\left(Y_1\right) h_{2, q}\left(Y_2\right) \cdots h_{P, q}\left(Y_P\right),
\end{equation*}
where $Q$ is an integer, and $h_{p,q}(Y_p)$ are permutation invariant polynomials.
\end{lemma}
With Lemma \ref{lem: PPIE decompo}, we can show $g(Y) = \sum_{q=1}^Q h_{1, q}\left(Y_1\right) h_{2, q}\left(Y_2\right) \cdots h_{P, q}\left(Y_P\right)$, where each function $h_{p, q}$ is a fully permutation invariant polynomial. Therefore, we can further decompose each $h_{p,q}$ into a power sum basis as described in Lemma \ref{lem: Hilbert}
\begin{equation*}
    h_{p,q}(Y_p) = \sum_{t=0}^{T_{p,q}}s_{p,q,t}(\sum_{n=1}^{N_p}(a_{p,q,t}^{\top}Y_{p,n})^1,\cdots,\sum_{n=1}^{N_p}(a_{p,q,t}^{\top}Y_{p,n})^{N_p}).
\end{equation*}

With this decomposition of $g(Y)$, we can construct a PPIE, $\widetilde{g}(Y)$, as follows
\begin{equation*}
    \widetilde{g}(Y) = \tau(\widetilde{h}_{1,q}(Y_1), \ldots, \widetilde{h}_{p,q}(Y_p), \ldots, \widetilde{h}_{P,Q}(Y_P)),
\end{equation*}
where $\widetilde{h}_{p,q}(Y_p)$ approximates $h_{p,q}(Y_p)$, and $\tau$ approximates the product of the $h_{p,q}$ functions. Then we propose the structure of $\tau$ and $\widetilde{h}_{p,q}$ in detail. $\tau$ is iterativly defined with $\tilde{x}$ such that
\begin{equation}
    \tau(\widetilde{h}_{1,1}, \ldots, \widetilde{h}_{1,Q}, \widetilde{h}_{2,1}, \ldots, \widetilde{h}_{2,Q},\ldots, \widetilde{h}_{P,Q}) = \sum_{q=1}^{Q}\tilde{x}(\widetilde{h}_{1,q},\tilde{x}(\widetilde{h}_{2,q}\ldots,\widetilde{h}_{P,q})),
\end{equation}
where $\tilde{x}$ is a ReLU network from Lemma \ref{lem: relu} with an error of $\epsilon = \delta_3$. Following the methodology in Section \ref{sec: app}, each $h_{p,q}$ can approximated by a PIE, denoted as $\widetilde{h}_{p,q}$.
\begin{equation*}
    \widetilde{h}_{p,q}(Y_p) = \sum_{t=0}^{T_{p,q}}\psi_{p,q,t}(\varphi_{p,q,t,1}(Y_p),\cdots,\varphi_{p,q,t,N_p}(Y_p)),
\end{equation*}
where $\varphi_{p,q,t,n}$ are iteratively defined using $\tilde{x}_1$, akin to $\varphi_{t,n}$ in \eqref{eq: varphi_t}. Similarly, $\psi_{p,q,t}$ is defined iteratively with $\tilde{x}_2$, as is $\psi_t$ in \eqref{eq: psi_t}. The approximation errors for $\tilde{x}_1$ and $\tilde{x}_2$ are denoted as $\delta_1$ and $\delta_2$, respectively. We establish that $T_{p,q} = O((N')^{MN_p})$ and accordingly define $T = O((N')^{MN})$.

With some calculation, we can show function $\tilde{g}$ is a PPIE as in Definition \ref{def:ppie network}.  The function class of PPIE, denoted as $\Fcal = \Fcal(S_{\tau}, L_{\tau}, S_{\psi}, L_{\psi}, S_{\varphi}, L_{\varphi}, B)$, encapsulates the structure of $\tilde{g}$. Here, $S_{\tau}, S_{\psi}, S_{\varphi}$ and $L_{\tau}, L_{\psi}, L_{\varphi}$ represent the total number of parameters and the depth of the networks for $\tau, \psi, \varphi$, respectively. All the paramters in these networks are bounded by $B$.

Considering the definition of $f_n$, we have
\begin{equation*}
    \|f_1-f_n\|_{\infty} \leq \|f_1-\tilde{f}\|_{\infty} \leq |g(Y) - \tilde{g}(Y)|, \quad \forall Y .
\end{equation*}

To establish the bound for $\|f_1 - f_n\|_{\infty}$, it suffices to bound $|g(Y) - \tilde{g}(Y)|$. This decomposition and bounding process will be carried out following the approach described in Section \ref{sec: app}.

\begin{align*}
    |\tilde{g}(Y)-g(Y)| \nonumber 
    &\leq\sum_{q=1}^{Q}|\tilde{x}(\widetilde{h}_{1,q}(Y_1),\tilde{x}(\widetilde{h}_{2,q}(Y_2),\tilde{x}(\widetilde{h}_{3,q}(Y_3),\cdots)))-h_{1, q}\left(Y_1\right) h_{2, q}\left(Y_2\right) \cdots h_{P, q}\left(Y_P\right)|\\
    &\leq\sum_{q=1}^{Q}\left[\underbrace{|\tilde{x}(\widetilde{h}_{1,q}(Y_1),\tilde{x}(\widetilde{h}_{2,q}(Y_2),\tilde{x}(\widetilde{h}_{3,q}(Y_3),\cdots)))-\widetilde{h}_{1, q}\left(Y_1\right) \widetilde{h}_{2, q}\left(Y_2\right) \cdots \widetilde{h}_{P, q}\left(Y_P\right)|}_{\eta_1}\right . \\
    &\left . +\underbrace{|\widetilde{h}_{1, q}\left(Y_1\right) \widetilde{h}_{2, q}\left(Y_2\right) \cdots \widetilde{h}_{P, q}\left(Y_P\right)-h_{1, q}\left(Y_1\right) h_{2, q}\left(Y_2\right) \cdots h_{P, q}\left(Y_P\right)|}_{\eta_2}\right].
    \label{PPIE decomp}
\end{align*}

From results in Section \ref{sec: app}, it is established that for each pair of indices $p$ and $q$, the bound $\|h_{p,q}\|_{\infty} \leq T_{p,q}(MN_p+\beta)(aK_k^{\prime})^{MN_p+\beta}$ holds true. If we define $G$ as $T(MN+\beta)(aK^{\prime})^{MN+\beta}$, then it follows that $\|h_{p,q}\|_{\infty} \leq G$ for all $p, q$. According to \eqref{eq: in-out diff}, we note that each $\|\widetilde{h}_{p,q} - h_{p,q}\|_{\infty}$ can be upper bounded by $T_{p,q}(N(aK'_k)^N)^{MN_p+\beta}(\delta_1+ \delta_2)$, provided we disregard lower order terms. Setting $\delta$ to be $T(N(aK')^N)^{MN+\beta}(\delta_1+ \delta_2)$, we find that $\|\widetilde{h}_{p,q} - h_{p,q}\|_{\infty} \leq \delta$ for every $p, q$, upon neglecting lower-order components. Consequently, $\|\widetilde{h}_{p,q}\|_{\infty}$ is bounded by $G+\delta$.

Accordingly, the bound for $\eta_1$ can be similarly derived, utilizing the results in \eqref{eq: reduction}
\begin{equation*}
    \eta_1\leq (G+\delta)^{P+1}\delta_3.
\end{equation*}

To establish an upper bound for $\eta_2$, we consider the following inequalities
\begin{align*}
    \eta_2&\leq |\widetilde{h}_{1} \widetilde{h}_{2} \cdots \widetilde{h}_{P} - h_{1}h_{2}\cdots h_P|\\
    &\leq|\widetilde{h}_{1} \widetilde{h}_{2} \cdots \widetilde{h}_{P}-\widetilde{h}_{1} \widetilde{h}_{2} \cdots \widetilde{h}_{P-1}h_P|\\
    &+|\widetilde{h}_{1} \widetilde{h}_{2} \cdots \widetilde{h}_{P-1}h_P- \widetilde{h}_{1} \widetilde{h}_{2} \cdots h_{P-1}h_P|\\
    &+\cdots\\
    &+|\widetilde{h}_{1} h_{2}\cdots h_P- h_{1}h_{2}\cdots h_P|\\
    &\leq P (G+\delta)^{P-1}\delta.
\end{align*}

Given that $\delta$ is a relatively small error term compared to $G$, and considering that $P$ is independent of $\delta$, we can neglect the lower order terms. Consequently, we arrive at the bounds $\eta_1 \leq G^P \delta_3$ and $\eta_2 \leq PG^P\delta$. To summarize
\begin{equation*}
    \|f_n-f_1\|_{\infty} \leq |\tilde{g}(Y)-g(Y)| \leq Q(\eta_1+\eta_2) \leq QPG^P\delta + QG^P\delta_3.
\end{equation*}

Remembering that $\|f_n - f^*\|_{\infty}$ can be bounded by $\|f_1 - f^*\|_{\infty} + \|f_n - f_1\|_{\infty}$, we aim to ensure $\|f_n - f^*\|_{\infty} \leq \epsilon$. To achieve this, we set $\|f_1 - f^*\|_{\infty} = \frac{\epsilon}{2}$ and constrain $\|f_n - f_1\|_{\infty}$ to be at most $\frac{\epsilon}{2}$.

Let $\|f_1-f^*\| =\frac{\epsilon}{2}$, we can get $N' =O(\epsilon^{-1/\beta})$.

To ensure $\|f_n-f_1\| \leq \frac{\epsilon}{2}$ , we set $QPG^P\delta = QG^P\delta_3 = \frac{\epsilon}{4}$. By setting $QG^P\delta_3 = \frac{\epsilon}{4}$
\begin{eqnarray*}
    \delta_3 = O(\frac{\epsilon}{QT^P(MN+\beta)^P(aK^{\prime})^{P(MN+\beta)}})=O(\epsilon^{\frac{(MN+\beta)P}{\beta}}).
\end{eqnarray*}

Given that $\delta = (N(aK')^N)^{MN+\beta}(\delta_1 + \delta_2)$, and aiming for $QPG^P\delta = \frac{\epsilon}{8}$, we set
\begin{align*}
    &\delta_1 = \delta_2 = O(\frac{\epsilon}{QN^{MN+\beta}T^P(MN+\beta)^K(aK^{\prime})^{(N+K)(MN+\beta)}})=O(\epsilon^{\frac{(MN^2+\beta N)P}{\beta}}).\\
\end{align*}
With $\delta_3$ established, we can determine the network parameters for $\tau$. Recalling that $\tau$ sums up $Q$ distinct sub-networks, each constructed iteratively with $P$ instances of $\tilde{x}$, the depth and total number of parameters of $\tau$ can be approximated as
\begin{align*}
    L_{\tau} = P O (\ln(1/\delta_3))=O(\ln(\epsilon^{-\frac{(MN+\beta)P}{\beta}})),\\
    S_{\tau} = QK O(\ln(1/\delta_3)) = O(\ln(\epsilon^{-\frac{(MN+\beta)P}{\beta}})).
\end{align*}

The network parameters for $\psi$ and $\varphi$ are derived in a similar fashion to those in Section \ref{sec: app}, using $\delta_1$ and $\delta_2$. The details are omitted for brevity.
\begin{align}
    &L_{\psi} = (MN+\beta)\times O(\ln(1/\delta_2)) = O(\ln(\epsilon^{-\frac{(MN^2+N\beta)P}{\beta}})),\nonumber\\
    &S_{\psi} = TN^{MN+\beta}O(\ln(1/\delta_2)) = O(\epsilon^{-\frac{MN}{\beta}}),\nonumber\\
    &L_{\varphi}=N O(\ln(1/\delta_1)) = O(\ln(\epsilon^{-\frac{(MN^2+N\beta)P}{\beta}})),\nonumber\\
    &S_{\varphi} = PT+TNO(ln(1/\delta_1)) =O(\epsilon^{-\frac{MN}{\beta}}) \nonumber.
\end{align}

\subsection{Estimation Error}
\label{ppie est}
In this section, our focus is on bounding the term 
\begin{eqnarray*}
    \frac{2}{n}\sum_{i=1}^{n}\xi_i(\hat{f}(x_{i})-f(x_i)).
\end{eqnarray*}
The primary objective here is to establish an upper limit for the covering number associated with the PPIE function class.

Adopting the approach used in Section \ref{sec: est}, we consider two functions from the class $\Fcal$, denoted as $f$ and $f^*$, with the all corresponding parameters of these functions differ by at most $\zeta$. With this setup, we proceed to the following decomposition
\begin{align*}
    \|f-f^*\|_{\infty} &= \| \sum_{q=1}^{Q}\tilde{x}(\widetilde{h}_{1,q},\tilde{x}(\widetilde{h}_{2,q}\ldots,\widetilde{h}_{P,q}))
    -\sum_{q=1}^{Q}\tilde{x}^*(\widetilde{h}^*_{1,q},\tilde{x}^*(\widetilde{h}^*_{2,q}\ldots,\widetilde{h}^*_{P,q}))\|_{\infty}\nonumber\\
    &\leq \underbrace{\sum_{q=1}^{Q} \| \tilde{x}(\widetilde{h}_{1,q},\tilde{x}(\widetilde{h}_{2,q}\ldots,\widetilde{h}_{P,q}))
    -\tilde{x}(\widetilde{h}^*_{1,q},\tilde{x}(\widetilde{h}^*_{2,q}\ldots,\widetilde{h}^*_{P,q}))\|_{\infty}}_{\gamma_1}\nonumber\\
    &+\underbrace{\sum_{q=1}^{Q}\| \tilde{x}(\widetilde{h}^*_{1,q},\tilde{x}(\widetilde{h}^*_{2,q}\ldots,\widetilde{h}^*_{P,q}))
    -\tilde{x}^*(\widetilde{h}^*_{1,q},\tilde{x}^*(\widetilde{h}^*_{2,q}\ldots,\widetilde{h}^*_{P,q}))\|_{\infty}}_{\gamma_2}.\nonumber\\
\end{align*}

Building upon the insights from Section \ref{sec: est}, we have established that
\begin{align*}
    \|\widetilde{h}_{p,q} - \widetilde{h}^*_{p,q}\|_{\infty} &\leq \zeta T_{p,q}(C_wB)^{L_{\psi}+L_{\varphi}}(L_{\psi}+L_{\varphi})N_p^{MN_p+\beta}K'_p \\
    &\leq \zeta T(C_wB)^{L_{\psi}+L_{\varphi}}(L_{\psi}+L_{\varphi})N^{MN+\beta}K'.
\end{align*}

Applying a similar recursive technique as used for bounding the first term in \eqref{eq: cover-decomp}, it can be shown that
\begin{equation*}
    \gamma_1 \leq \sum_{q=1}^Q (C_wB)^{L_{\tau}}L_{\tau} \max_{1\leq  p\leq P}\|\widetilde{h}_{p,q} - \widetilde{h}^*_{p,q}\|_{\infty} \leq \zeta Q  T(C_wB)^{L_{\tau}+L_{\psi}+L_{\varphi}}(L_{\psi}+L_{\varphi})N^{MN+\beta}K'.
\end{equation*}

Similarly, for $\gamma_2$, following the approach for bounding the second term in \eqref{eq: cover-decomp}, we have
\begin{equation*}
    \gamma_2 \leq \sum_{q=1}^Q (C_wB)^{L_{\tau}} L_{\tau} \max_{1\leq  p\leq P}\|\widetilde{h}^*_{p,q}\|_{\infty} \leq  \zeta Q T(C_wB)^{L_{\tau}+L_{\psi}+L_{\varphi}}L_{\tau}(L_{\psi}+L_{\varphi})N^{MN+\beta}K',
\end{equation*}
where the final inequality is derived using results from Section \ref{sec: est}.

In summary, a perturbation of the parameters by $\zeta$ results in a bounded change in the output
\begin{equation*}
    \|f-f^*\|_{\infty}\leq \sum_{q=1}^Q (C_wB)^{L_{\tau}} L_{\tau} \max_{1\leq  p\leq P}\|\widetilde{h}^*_{p,q}\|_{\infty} \leq  \zeta Q T(C_wB)^{L_{\tau}+L_{\psi}+L_{\varphi}}(L_{\tau}+1)(L_{\psi}+L_{\varphi})N^{MN+\beta}K'.
\end{equation*}

Following this, we can replicate the arguments from Section \ref{sec: est} to determine the upper bound for both the covering number and the empirical Rademacher Complexity. Let $L = L_{\psi} + L_{\varphi} + L_{\tau}$ and $S = S_{\psi} + S_{\varphi} + S_{\tau}$.
\begin{align}
    \log\Ncal(\epsilon, \Fcal_{\delta}, \| \cdot \|_{\infty}) &\leq S\log(\frac{QBT(C_wB)^L(L_{\psi}+L_{\varphi})(L_{\tau}+1)N^{MN+\beta}K' }{\epsilon}), \nonumber\\
    E_{\xi} \left[ \sup_{f'\in \Fcal_{\delta}} | \frac{1}{n} \sum_{i=1}^n\xi_if'(Y_{(i)})| \right] &\leq 2\sqrt{2} \frac{\sigma \sqrt{S}\delta}{n^{1/2} } \log  (\frac{QT(C_wB)^L(L_{\psi}+L_{\varphi})(L_{\tau}+1)N^{MN+\beta} K'}{\delta} +1 ).    
\end{align}

Building upon the arguments from \eqref{eq: est-concentration} to \eqref{eq: final-decomp-results}, let's denote $V_n = 4\sqrt{2} \frac{\sigma \sqrt{S}}{n^{1/2}}$. Using this notation, we can  express the inequality as follows
\begin{eqnarray*}
        \|\hat{f}-f^*\|_n^2 \leq 10\|f^*-f\|_n^2+ 8\left\{V_n\log (\frac{QT(C_wB)^L(L_{\psi}+L_{\varphi})(L_{\tau}+1)N^{MN+\beta} K'}{V_n} +1) \right\}^2 +4u,
\end{eqnarray*}
with probability at least $1-\text{exp}(-nu^2/2\sigma^2\delta^2)$. 

Upon taking the expectation  on both sides, we derive a bound for $\|\hat{f} - f^*\|_{L_2(P_X)}$ similar to \eqref{eq: final-step}
\begin{align}
        \|&\hat{f}-f^*\|_{L_2(P_X)}^2 \nonumber \\
    & \leq 21\epsilon^2 +17\left\{V_n\log (\frac{QT(C_wB)^L(L_{\psi}+L_{\varphi})(L_{\tau}+1)N^{MN+\beta} K'}{V_n} +1) \right\}^2 ,\label{eq:ppie error}
\end{align}
with probability converging to one.
Let $\epsilon=n^{-\frac{\beta}{MN+2\beta}}$. By substituting the expressions of $V_n$, $S_{\psi}$, $S_{\varphi}$, $S_{\tau}$, $L_{\psi}$, $L_{\varphi}$, and $L_{\tau}$ in terms of $\epsilon$ into equation (\ref{eq:ppie error}) and disregarding smaller terms, we arrive at
\begin{equation*}
    \| \hat{f} - f^*\|^2 _{L^2(P_X)} = O_p(n^{-\frac{2\beta}{MN+2\beta}}).
\end{equation*}
  
In conclusion, it's noteworthy that the Permutation Invariant Estimator (PIE) is a specific instance of the partially permutation invariant function, achievable by setting $P = 2$, $N_1 = 1$, and $N_2 = N$, while adjusting the column dimension to $M + 1$. This modification seamlessly aligns with the desired result.

\section{Proof of Value-based Estimator}
\label{sec: EMP}
\subsection{The cross-fitting scheme}
 Consider our dataset $\{(X_{i,j}, A_{i,j}, Y_{i,j}): 1 \le i \le R, 1 \le j \le S\}$. Specifically, the $S$ samples are partitioned into $m$ equally sized batches, denoted as $[S] \in \bigcup_{z \in [m]} B_z$. Here, $m$ is chosen such that $m \geq 2$ and typically $m \asymp 1$. Each batch, $B_z$, contains approximately $|B_z| = S_z \asymp S/m \asymp S$ samples. For each sample indexed by $j \in [S]$, let $z_j \in [m]$ be the index of the batch containing that sample, such that $j \in B_{z_j}$. For the $j$-th sample, the optimization process described in (\ref{eq: MSE-non-d}) of the main text is applied to the out-of-batch samples $B_{-z_j} = [S] \backslash B_{z_j}$. This procedure yields estimators $\widehat{f}_j$ and $\widehat{m}_j$ which for offline policy evaluation.

\subsection{Proof of Corollary \ref{coro: nondynamic}}
According to the definitions of $\widehat{J}_{\textrm{VB}}$ and $J(\pi)$, the following holds true
\begin{equation*}
    |\widehat{J}_{\textrm{VB}}(\pi) -J(\pi)| = |\frac{1}{S}\sum_{j=1}^S\sum_{i=1}^R \widehat{f}_{i,j}(X_{i,j},\pi(X_{i,j}), \widehat{m}_i(X_{\mathcal{N}(i),t}, \pi(X_{\mathcal{N}(i),t}))) - \sum_{i=1}^R \mathbb{E}[\mathbb{E}(Y_i| \{A_j=\pi(X_j)\}_j, \{X_j\}_j)]|.
\end{equation*}
To simplify, we introduce the notation $x_{i,j}$ to denote $\left(X_{i,j}, \pi(X_{i,j}), \widehat{m}_i(X_{\mathcal{N}(i),t}, \pi(X_{\mathcal{N}(i),t}))\right)$ and $x_i$ for $(X_i, A_i, m_i(X_{\mathcal{N}(i)}, A_{\mathcal{N}(i)}))$. With this notation, the conditional expectation  $\mathbb{E}(Y_i | \{A_j = \pi(X_j)\}_j, \{X_j\}_j)$ can be rewritten as $f_i(X_i, A_i, m_i(X_{\mathcal{N}(i)}, A_{\mathcal{N}(i)}) = f_i(x_i)$. Thus, we have
\begin{align}
    |\widehat{J}_{\textrm{VB}}(\pi) -J(\pi)| &= |\frac{1}{S}\sum_{j=1}^S\sum_{i=1}^R \widehat{f}_{i,j}(x_{i,j}) - \sum_{i=1}^R \mathbb{E}[f_i(x_{i})]|\nonumber\\
    & \leq \sum_{i=1}^R\left\{|\frac{1}{S}\sum_{j=1}^S \widehat{f}_{i,j}(x_{i,j}) - \frac{1}{S}\sum_{j=1}^S\mathbb{E}[\widehat{f}_{i,j}(x_{i,j})]| + |\frac{1}{S}\sum_{j=1}^S\mathbb{E}[\widehat{f}_{i,j}(x_{i,j})]- \mathbb{E}[f_i(x_{i})]|\right\}\nonumber\\
    & \leq \sum_{i=1}^R\left\{|\frac{1}{S}\sum_{z=1}^m S_z\frac{1}{S_z}\sum_{j\in B_z} \widehat{f}_{i,j}(x_{i,j}) - \frac{1}{S}\sum_{z=1}^m S_z\frac{1}{S_z}\sum_{j\in B_z}\mathbb{E}[\widehat{f}_{i,j}(x_{i,j})]| \right . \nonumber\\
    & \left . + |\frac{1}{S}\sum_{z=1}^m S_z\frac{1}{S_z}\sum_{j\in B_z}\mathbb{E}[\widehat{f}_{i,j}(x_{i,j})]- \mathbb{E}[f_i(x_{i})]|\right\}\nonumber\\
     & \leq \sum_{i=1}^R\left\{\frac{1}{S}\sum_{z=1}^m S_z|\frac{1}{S_z}\sum_{j\in B_z} \widehat{f}_{i,j}(x_{i,j}) - \frac{1}{S_z}\sum_{j\in B_z}\mathbb{E}[\widehat{f}_{i,j}(x_{i,j})]| \right .\nonumber\\
     & \left . + \frac{1}{S}\sum_{z=1}^m S_z|\frac{1}{S_z}\sum_{j\in B_z}\mathbb{E}[\widehat{f}_{i,j}(x_{i,j})]- \mathbb{E}[f_i(x_{i})]|\right\}\nonumber\\
     &\leq \sum_{i=1}^R\left\{\frac{1}{S}\sum_{z=1}^m S_z|\frac{1}{S_z}\sum_{j\in B_z} \widehat{f}_{i,j}(x_{i,j}) - \mathbb{E}[\widehat{f}_{i,B^{(1)}_z}(x_{i,B^{(1)}_z})]| \right .\nonumber\\
     & \left . + \frac{1}{S}\sum_{z=1}^m S_z|\mathbb{E}[\widehat{f}_{i,B^{(1)}_z}(x_{i,B^{(1)}_z})]- \mathbb{E}[f_i(x_{i})]|\right\}\label{notation}.
\end{align}
The derivation of (\ref{notation}) is based on the observation that $\widehat{f}_{i,j}$ are equal for all $j\in B_z, \forall z\in[m]$. Additionally, the expected value $\mathbb{E}[\widehat{f}_{i,j}(x_{i,j})]$ remains constant for $t\in B_j$. Therefore, we designate $B^{(1)}_z$ as the initial element in $B_j$, leading to the conclusion that $\mathbb{E}[\widehat{f}_{i,B^{(1)}_z}(x_{i,B^{(1)}_z})] = \mathbb{E}[\widehat{f}_{i,j}(x_{i,j})]$ for all $t\in B_j$. Concerning the term $|\frac{1}{S_z}\sum_{j\in B_z} \widehat{f}_{i,j}(x_{i,j}) - \mathbb{E}[\widehat{f}_{i,B^{(1)}_z}(x_{i,B^{(1)}_z})]|$, Theorem 4.10 in \citep{wainwright2019high} is applied, yielding
\begin{equation*}
    \bigg\vert \frac{1}{S_z}\sum_{j\in B_z} \widehat{f}_{i,j}(x_{i,j}) - \mathbb{E}\widehat{f}_{i,B^{(1)}_z}(x_{i,B^{(1)}_z})\bigg\vert \leq \sup_{f\in \Gcal} \bigg\vert \frac{1}{S_z}\sum_{j\in B_z} f(x_{i,j}) - \mathbb{E}f(x_{i,B^{(1)}_z})\bigg\vert \leq \E_{X}\Rcal_{S_z} \Gcal + e,
\end{equation*}
with probability at least $1-\exp(-\frac{S_ze^2}{8J^2})$. Subsequently, Lemma \ref{lem:F-Chaining} is utilized to establish an upper bound for $\Rcal_{S_z}(\Gcal)$
\begin{equation*}
    \Rcal_{S_z}\Gcal \leq \frac{4\sqrt{2}}{S_z}\int_0^{J}\sqrt{\log2\Ncal(\epsilon,\Gcal, \|\cdot\|_{\infty})}d\epsilon.
\end{equation*}

Building upon the arguments presented in \ref{ppie est}, we establish a bound for the covering number of $\Gcal$. For simplicity, we omit these details. Considering $x_{i,j} \in \R^{(M+1)\times (N+1)}$ and recognizing that $S-S_z$ samples are utilized to construct $\hat{f}_{i,j}$, the Rademacher complexity $\Rcal_{S_z}(\Gcal)$ can be upper bounded as $O((S-S_z)^{\frac{(M+1)(N+1)}{2((M+1)(N+1)+2\beta)}}S_z^{-1/2})$.

Next, combining this with the bound for $\|\widehat{f}_i-f^*_i\|_{L^2(P_X)}$, we derive
\begin{align*}
    |\widehat{J}_{\textrm{VB}}(\pi) -J(\pi)| 
    &\leq \sum_{i=1}^R\left\{\frac{1}{S}\sum_{z=1}^m S_zC_j(S-S_z)^{\frac{(M+1)(N+1)}{2((M+1)(N+1)+2\beta)}}S_z^{-1/2}  + \frac{1}{S}\sum_{z=1}^m S_zC^{\prime}_j(S-S_z)^{-\frac{\beta}{(M+1)(N+1)+2\beta}}\right\}+ e \nonumber\\
    &\leq \tilde{C}RS^{-\frac{\beta}{(M+1)(N+1)+2\beta}} +e = O(RS^{-\frac{\beta}{(M+1)(N+1)+2\beta}})+e,
\end{align*}
with probability at least $1-\exp(-\frac{Se^2}{8J^2})$. The second inequality is predicated on the assumption that $S_z \asymp S/m \asymp S$.

\subsection{Proof of Corollary \ref{coro: dynamic}}
\label{appendix:dynamic}

In this section, our objective is to establish the convergence rate for the value-based estimator in dynamic setting. We commence by broadening the scope of the transition kernel $\Pcal^{\pi}$, as defined in Section \ref{subsection:dy}.
\begin{equation*}
    \Pcal^{\pi}r(X,A) = \E\left[r(X',A')|X'\sim P(\cdot|X,A), A'\sim \pi(X')\right].
\end{equation*}

As deduced from \eqref{eq: VB}, the value-based estimator $\widehat{J}_{\gamma}^{VB}(\pi)$ is intrinsically linked to $\widehat{Q}^{(i,1)}$. Our initial endeavor is to establish an upper bound for $\|Q^{(i,1)}_{i}(X_{\Ncal^*(i), 1}, A_{\Ncal^*(i), 1}) - \widehat{Q}_{\pi}^{(i,1)}(X_{\Ncal^*(i), 1}, A_{\Ncal^*(i), 1})\|_{L^2(p_i)}$, with $p_i$ denoting the stationary distribution of the confounder-action pair $(X_{\Ncal^*(i), t}, A_{\Ncal^*(i), t})$ 
under the given behavioural policy. For the sake of brevity, we will hereafter refer to these as $Q^{(i,1)}_{i}$ and $\widehat{Q}_{\pi}^{(i,1)}$, omitting the  explicit  mention of $X$ and $A$.

Based on the established definition of $Q_{\pi_i}^{(1)}$ in (\ref{eq: def-Q}), it follows that 
\begin{align*}
    \widehat{Q}_{\pi}^{(i,1)} - Q^{(i,1)}_{i} &= \widehat{Q}_{\pi}^{(i,1)} - \left[r_i + \gamma \Pcal^{\pi} Q_{\pi_i}^{(1)}\right] + \gamma\Pcal^{\pi} \widehat{Q}_{\pi}^{(i,1)} - \gamma \Pcal^{\pi} \widehat{Q}_{\pi}^{(i,1)}\\
    & = \widehat{Q}_{\pi}^{(i,1)} - r_i - \gamma \Pcal^{\pi} \widehat{Q}_{\pi}^{(i,1)} + \gamma\Pcal^{\pi}[\widehat{Q}_{\pi}^{(i,1)} - Q^{(i,1)}_{i}]. 
\end{align*}

Defining the square error of one-step approximation as $\zeta^{(t)}_i= \left(\widehat{Q}_{\pi}^{(i,T)} - r_i - \gamma \Pcal^{\pi} \widehat{Q}_{\pi}^{(i,T)}\right)^2$, we obtain 
\begin{align*}
    \left(\widehat{Q}_{\pi}^{(i,1)} - Q^{(i,1)}_{i}\right)^2 &= \zeta^{(1)}_i + \gamma^2 \left[\Pcal^{\pi}(\widehat{Q}_{\pi}^{(i,1)} - Q^{(i,1)}_{i})\right]^2 \\
    &\leq \zeta^{(1)}_i + \gamma^2 \Pcal^{\pi}(\widehat{Q}_{\pi}^{(i,1)} - Q^{(i,1)}_{i})^2. 
\end{align*}
where the second inequality holds since the expectation  of a squared random variable exceeds the square of its expectation . In a similar vein, we deduce that  
\begin{equation*}
    \Pcal^{\pi}\left(\widehat{Q}_{\pi}^{(i,1)} - Q^{(i,1)}_{i}\right)^2 \leq \Pcal^{\pi}\zeta^{(2)}_i + \gamma^2 \Pcal^{\pi} \Pcal^{\pi} (\widehat{Q}_{\pi}^{(i,1)} - Q^{(i,1)}_{i})^2,  
\end{equation*}
leading to the conclusion that 
\begin{equation*}
    \left(\widehat{Q}_{\pi}^{(i,1)} - Q^{(i,1)}_{i}\right)^2 \leq \zeta^{(1)}_i + \gamma^2 \Pcal^{\pi}\zeta^{(2)}_i + \gamma^4 \Pcal^{\pi} \Pcal^{\pi} (\widehat{Q}_{\pi}^{(i,1)} - Q^{(i,1)}_{i})^2. 
\end{equation*}
Utilizing the established recurrence relation, we derive that
\begin{equation*}
    \left(\widehat{Q}_{\pi}^{(i,1)} - Q^{(i,1)}_{i}\right)^2 \leq \sum_{t=1}^T \gamma^{2t-2} (\Pcal^{\pi})^{t-1} \zeta^{(t)}_i + \gamma^{2T} (\Pcal^{\pi})^{\top} (\widehat{Q}_{\pi}^{(i,1)} - Q^{(i,1)}_{i})^2. 
\end{equation*}
Taking expectation s on both sides with respect to $p_i$, we obtain
\begin{equation}
    \|Q^{(i,1)}_{i} - \widehat{Q}_{\pi}^{(i,1)}\|_{L^2(p_i)} \leq \sum_{t=1}^T \E_{p_i}\left[\gamma^{2t-2} (\Pcal^{\pi})^{t-1} \zeta^{(t)}_i \right] + \gamma^{2T}\E_{p_i}\left[ (\Pcal^{\pi})^{\top} (\widehat{Q}_{\pi}^{(i,1)} - Q^{(i,1)}_{i})^2 \right]. 
    \label{eq: propagation}
\end{equation}
Considering any measurable function $f$ over time step $t$, we have
\begin{equation*}
    \E_{p_{i}}\left[(\Pcal^{\pi})^tf\right] = \int (\Pcal^{\pi})^tf dp_{i,1} = \int f d p_{\pi, t}= \E_{p_{\pi,t}}\left[f\right]. 
\end{equation*} 
Applying the Cauchy-Schwarz inequality leads to
\begin{equation*}
    \E_{p_{\pi, t}}(f) \leq \sqrt{\int\Big\vert\frac{d p_{\pi, t}}{d p_{i}} \Big\vert d p_i} \cdot \sqrt{\int f d p_{i}}. 
\end{equation*}
From Assumption \ref{ass: concentration}, it follows that
\begin{equation*}
    \E_{p_{\pi, t}}(f) \leq \kappa(t) \cdot \E_{p_{i}}(f) \leq 
 \kappa\E_{p_{i}}(f).
\end{equation*}
Incorporating this into \eqref{eq: propagation}, we deduce
\begin{equation}
\label{eq: Q-inequality}
    \|Q^{(1)}_{\pi} - \widehat{Q}_{\pi}^{(i,1)}\|_{L^2(p_{i})} \leq \kappa \sum_{t=1}^T \gamma^{2t-2} \E_{p_{i}}\left[ \zeta^{(t)}_i \right] + \kappa \gamma^{2T}\E_{p_{i}}(\widehat{Q}_{\pi}^{(i,T)} - Y_i)^2. 
\end{equation}

From Assumption \ref{ass: Tass}, for all $t$ both $r_i + \gamma \Pcal^{\pi}\widehat{Q}_{\pi}^{(i,t)}$ and $r_i$ are permutation invariant and belongs to $\Wcal^{\beta, \infty}([0,J]^{(N+1)(M+1)})$. Using the results in Corollary \eqref{co: PPIE rate}, it is evident that for each $t$, $\E_{p_i}\left[ \zeta^{(t)}_i\right]$ and $\E_{p_i}(\widehat{Q}_{\pi}^{(i,T)} - Y_i)^2$ are bounded above by $O(S^{-\frac{2\beta}{(N+1)(M+1)+2\beta}})$ with probability converging to one. Bring back to \eqref{eq: Q-inequality}, we have
\begin{equation*}
    \|Q^{(1)}_{p_{i,1}} - \widehat{Q}_{\pi}^{(i,1)}\|_{L^2(p_i)} \leq \kappa T\gamma^{2T} O(S^{-\frac{\beta}{(N+1)(M+1)+2\beta}}) \leq O(\kappa T S^{-\frac{\beta}{(N+1)(M+1)+2\beta}}). 
\end{equation*}
Employing empirical process techniques akin to those in Section \ref{sec: EMP}, we can establish that
\begin{align*}
    |J_{\gamma}(\pi) - \widehat{J}_{\gamma}^{\text{VB}}(\pi)| \leq O(RS^{-\frac{\beta}{(N+1)(M+1)+2\beta}}) + O(\kappa TRS^{-\frac{\beta}{(N+1)(M+1)+2\beta}}) + e = O(\kappa TRS^{-\frac{\beta}{((N+1)(M+1)+2\beta)}})+e,
\end{align*} 
with probability at least $1-\exp(-\frac{Se^2}{8J^2})$. 
\bibliography{reference}